\documentclass{article}

\usepackage{arxiv}

\usepackage[utf8]{inputenc} % allow utf-8 input
\usepackage[T1]{fontenc}    % use 8-bit T1 fonts
\usepackage{hyperref}       % hyperlinks
\usepackage{url}            % simple URL typesetting
\usepackage{booktabs}       % professional-quality tables
\usepackage{amsfonts}       % blackboard math symbols
\usepackage{nicefrac}       % compact symbols for 1/2, etc.
\usepackage{microtype}      % microtypography
\usepackage{lipsum}		% Can be removed after putting your text content
\usepackage{graphicx}
\usepackage{natbib}
\usepackage{doi}
\usepackage{multirow}
\usepackage{wrapfig}

\usepackage{appendix, subcaption}
\usepackage{amsmath}
\usepackage{amsthm}
\usepackage{amssymb}

\usepackage{algorithm}
\usepackage{algorithmic}

\newtheorem{theorem}{Theorem}
\newtheorem{lemma}{Lemma}
\newtheorem{assumption}{Assumption}

\title{Faster Stochastic Variance Reduction Methods for Compositional MiniMax Optimization}

\author{{Jin Liu} \\
	Central South University\\
	\texttt{liujin06@csu.edu.cn} \\
	\And
	{Xiaokang Pan} \\
	Central South University\\
	\texttt{224712176@csu.edu.cn} \\
        \And
	{Junwen Duan} \\
	Central South University\\
	\texttt{jwduan@csu.edu.cn} \\
        \And
	{Hong-Dong Li} \\
	Central South University\\
	\texttt{hongdong@csu.edu.cn} \\
        \And
	{Youqi Li} \\
	Beijing Institute of Technology\\
	\texttt {liyouqi@bit.edu.cn} \\
        \And
	{Zhe Qu}\\
	Central South University\\
	\texttt{zhe\_qu@csu.edu.cn} \\
}

\date{}

\renewcommand{\headeright}{}
\renewcommand{\undertitle}{}
\renewcommand{\shorttitle}{Faster Stochastic Variance Reduction Methods for Compositional MiniMax Optimization}

\hypersetup{
pdftitle={Faster Stochastic Variance Reduction Methods for Compositional MiniMax Optimization}
}

\begin{document}
\maketitle

\begin{abstract}
This paper delves into the realm of stochastic optimization for compositional minimax optimization—a pivotal challenge across various machine learning domains, including deep AUC and reinforcement learning policy evaluation. Despite its significance, the problem of compositional minimax optimization is still under-explored. Adding to the complexity, current methods of compositional minimax optimization are plagued by sub-optimal complexities or heavy reliance on sizable batch sizes. To respond to these constraints, this paper introduces a novel method, called Nested STOchastic Recursive Momentum (NSTORM), which can achieve the optimal sample complexity of $O(\kappa^3 /\epsilon^3 )$ to obtain the $\epsilon$-accuracy solution. We also demonstrate that NSTORM can achieve the same sample complexity under the Polyak-\L ojasiewicz (PL)-condition—an insightful extension of its capabilities. Yet, NSTORM encounters an issue with its requirement for low learning rates, potentially constraining its real-world applicability in machine learning. To overcome this hurdle, we present ADAptive NSTORM (ADA-NSTORM) with adaptive learning rates. We demonstrate that ADA-NSTORM can achieve the same sample complexity but the experimental results show its more effectiveness. All the proposed complexities indicate that our proposed methods can match lower bounds to existing minimax optimizations, without requiring a large batch size in each iteration. Extensive experiments support the efficiency of our proposed methods.
\end{abstract}

\section{Introduction}
In recent years, minimax optimization theory has been considered more attractive due to the broad range of machine learning applications, including generative adversarial networks \cite{goodfellow2014generative, arjovsky2017wasserstein, gulrajani2017improved}, adversarial training of deep neural networks \cite{madry2018towards, wang2021adversarial, qu2023prevent}, robust optimization \cite{chen2017robust, mohri2019agnostic, qu2022generalized}, and policy evaluation on reinforcement learning \cite{sutton2018reinforcement, hu2019efficient, zhang2021taming}. At the same time, many machine learning problems can be formulated as compositional optimizations, for example, model agnostic meta-learning \cite{finn2017model, gao2022convergence} and risk-averse portfolio optimization \cite{zhang2020optimal, shapiro2021lectures}. Due to the important growth of these two problems in machine learning fields, the compositional minimax problem should be also clearly discussed, which can be formulated as follows:
\begin{equation}\label{Eq:MinimaxObjective}
    \min_{x\in\mathcal{X}}\max_{y \in \mathcal{Y}} f(g(x), y) \triangleq \mathbb{E}_{\zeta}f( \mathbb{E}_{\xi} [g(x; \xi)], y ;\zeta),
\end{equation}
where $g(\cdot): \mathcal{X} \to \mathbb{R}^d , f(\cdot, \cdot): (\mathbb{R}^d ,\mathcal{Y}) \to \mathbb{R}$, $\xi \in \Xi$, $\zeta \in \Omega$, $\mathcal{X}$ and $\mathcal{Y}$ are convex and compact sets. Suppose that $f(g(x), y)$ is a strongly concave objective function with respect to $y$ for all $x \in \mathcal{X}$. 

Numerous research studies have been dedicated to investigating the convergence analysis of minimax optimization problems \cite{nemirovski2009robust, palaniappan2016stochastic, lin2020gradient, yang2020global, chen2020efficient, rafique2022weakly} across diverse scenarios. Various methodologies have been devised to address these challenges. Approaches such as Stochastic Gradient Descent Ascent (SGDA) have been proposed \cite{lin2020gradient}, accompanied by innovations like variance-reduced SGDA \cite{luo2020stochastic, xu2020enhanced}  that aim to expedite convergence rates. Moreover, the application of Riemannian manifold-based optimization has been explored \cite{huang2020gradient} across different minimax scenarios, showcasing the breadth of methodologies available. However, all of these methods are only designed for the non-compositional problem. It indicates that the stochastic gradient can be assumed as an unbiased estimation of the full gradient of both the two sub-problems. Because it is too difficult to get an unbiased estimation in compositional optimization, these methods cannot be directly used to optimize the compositional minimax problem.

Recent efforts have yielded just two studies on the nonconvex compositional minimax optimization problem \eqref{Eq:MinimaxObjective}, named Stochastic Compositional Gradient Descent Ascent (SCGDA) \cite{gao2021convergence} and Primal-Dual Stochastic Compositional Adaptive (PDSCA) \cite{yuan2022compositional}. However, they only can obtain the sample complexity $O(\kappa^4 /\epsilon^4 )$ for achieving the $\epsilon$-accuracy solution, which limits the applicability in many machine learning scenarios. Consequently, there is a pressing need to devise a more streamlined approach capable of tackling this challenge. In addition, some compositional minimization optimizations have been proposed, such as SCGD \cite{wang2017stochastic}, STORM \cite{cutkosky2019momentum}, and RECOVER \cite{qi2021online}. They may not be directly utilized for the minimax problem \eqref{Eq:MinimaxObjective}, because the minimization of objective $f(g(x))$ depends on the maximization of objective $f(g(x),\cdot)$ for any $x \in \mathcal{X}$. Furthermore, the combined errors from Jacobian and gradient estimators worsen challenges in both sub-problems. We aim to develop an approach effectively tackling the compositional minimax problem \eqref{Eq:MinimaxObjective}, optimizing sample complexities efficiently, without requiring a large batch size.

In this paper, to address the aforementioned challenges, we first develop a novel Nested STOchastic Recursive Momentum (NSTORM) method for the problem \eqref{Eq:MinimaxObjective}. The NSTORM method leverages the variance reduction technique \cite{cutkosky2019momentum} to estimate the inner/outer functions and their gradients. The theoretical result shows that our proposed NSTORM method can achieve the optimal sample complexity of $O(\kappa^3 /\epsilon^3 )$. To the best of our knowledge, NSTORM is the first method to match the best sample complexity in existing minimax optimization studies \cite{huang2023adagda, luo2020stochastic} without requiring a large batch size. We also demonstrate that NSTORM can achieve the same sample complexity under the Polyak-\L ojasiewicz (PL)-condition, which indicates an insightful extension of NSTORM. In particular, the central idea of our proposed NSTORM method and analysis has two aspects: 1) the variance reduction is applied to both function and gradient values, which is different from \cite{gao2021convergence, yuan2022compositional} and 2) the estimator of the inner gradient $\nabla g(x)$ is updated with a projection to ensure that the error can be bounded regardless of the minimization sub-problem. Furthermore, because NSTORM requires a small learning rate to obtain the optimal sample complexity, it may be difficult to set in real-world scenarios. To address this issue, we take advantage of adaptive learning rates in NSTORM and design an adaptive version, called ADAptive NSTORM (ADA-NSTORM). We also demonstrate that ADA-NSTORM can also obtain the same sample complexity as NSTORM, i.e., $O(\kappa^3 /\epsilon^3)$ and performs better in practice without tuning the learning rate manually. 

\section{Related Work}
{\bf Compositional minimization problem:} The compositional minimization optimization problem is common in many machine learning scenarios, e.g., meta-learning \cite{finn2017model} and risk-averse portfolio optimization \cite{zhang2020optimal}, which can be defined as follows:
\begin{equation}
    \min_{x} f(g(x)) \triangleq \mathbb{E}_{\zeta} f(\mathbb{E}_\xi [g(x; \xi)]; \zeta).
\end{equation}
A typical challenge to optimizing the compositional minimization problem is that we cannot obtain an unbiased estimation of the full gradient by SGD, i.e., $\mathbb{E}_{\xi, \zeta} [\nabla g(x;\xi)^\top \nabla_g f(g(x;\xi), \zeta)] \neq \nabla g(x)^\top \nabla_g f(g(x))$. To address this issue, some methods have been developed in the past few years. For example, \cite{wang2017stochastic} uses stochastic gradient for the inner function value when computing the stochastic gradient. However, the convergence rate only can achieve $O(1/\epsilon^8 )$ for the nonconvex objective, which has an obvious convergence gap to the regular SGD method. To improve the convergence speed, some advanced variance reduction techniques have been leveraged into Stochastic Compositional Gradient Descent (SCGD) \cite{wang2017stochastic}. For example, SAGA \cite{zhang2019composite}, SPIDER \cite{fang2018spider}, and STORM \cite{cutkosky2019momentum} were leveraged into SCGD and achieved a better convergence result, i.e., $O(1/\epsilon^3 )$. Recently, some studies \cite{yuan2019stochastic, zhang2021multilevel, jiang2022optimal, pmlr-v162-tarzanagh22a} bridged the gap between stochastic bilevel or multi-level optimization problems and stochastic compositional problems, and developed efficient methods. However, all of these methods only investigated the convergence result for minimization problems, ignoring the maximization sub-problem.

\noindent{\bf Minimax optimization problem:} The minimax optimization problem is an important type of model and leads to many machine learning applications, e.g., adversarial training and policy optimization. Typically, the minimax optimization problem can be defined as follows:
\begin{equation}\label{Eq:minimax}
    \min_{x\in\mathbb{R}^{d_1}}\max_{y \in \mathcal{Y}} f(x, y) \triangleq \mathbb{E}_{\xi}f(x, y ;\xi).
\end{equation}
Note that both $x$ and $y$ in \eqref{Eq:minimax} are trained from the same dataset. Currently, the prevailing approach for solving minimax optimization problems involves alternating between optimizing the minimization and maximization sub-problems. Stochastic Gradient Descent Ascent (SGDA) methods \cite{lin2020gradient, yan2020optimal, yuan2020stochastic} have been proposed as initial solutions to address this problem. Subsequently, accelerated gradient descent ascent methods \cite{luo2020stochastic, xu2020enhanced} emerged, leveraging variance reduction techniques to tackle stochastic minimax problems based on the variance reduction techniques. Additionally, research efforts have been made to explore non-smooth nonconvex-strongly-concave minimax optimization \cite{huang2020gradient, chen2020efficient}. Moreover, \cite{huang2020gradient} proposed the Riemannian stochastic gradient descent ascent method and some variants for the Riemannian minimax optimization problem. \citep{qiu2020single} reformulated nonlinear temporal-difference learning as a minimax optimization problem and proposed the single-timescale SGDA method. However, all of these methods fail to address the compositional structure inherent in the compositional minimax optimization problem presented in \eqref{Eq:MinimaxObjective}.

\section{The Proposed Method}
\subsection{Design Challenge}
Compared to the conventional minimax optimization problem, the main challenge in compositional minimax optimization is that we cannot obtain the unbiased gradient of the objective function $f$. Although we can access the unbiased estimation of each function and its gradient, i.e., $\mathbb{E}_\xi [g(x; \xi)] = g(x)$, $\mathbb{E}_\zeta [f(y; \zeta)] = f(y)$ and $\mathbb{E}_{\zeta}[\nabla f(y; \zeta)] = \nabla f(y)$, it is still difficult to obtain an unbiased estimation of the gradient $ \nabla f(g(x),y)$. This is due to the fact that the expectation over $\xi$ cannot be moved into the gradient $\nabla f$, i.e., $\mathbb{E}_{\xi, \zeta} [\nabla f(g(x; \xi), y;\zeta )] \neq \nabla f(g(x), y)$. Similarly, we cannot get the unbiased estimation of the function value $f$ such that $\mathbb{E}_{\xi, \zeta} [f(g(x ; \xi), y;\zeta)] \neq f(g(x), y)$.

Motivated by the aforementioned challenge, one potential approach to improve the evaluation of both function values and Jacobians is to utilize variance-reduced estimators. These estimators can effectively reduce estimation errors. However, applying variance-reduced estimators directly to minimax optimization in compositional minimization \cite{zhang2020optimal, qi2021online} is not straightforward. This is because if the estimators for Jacobians are not bounded, the estimation error may increase for the maximization sub-problem. In order to address this issue, \cite{gao2021convergence} and \cite{yang2022stochastic} have developed SCGDA and PDSCA methods to approach the compositional minimax optimization, respectively. However, they only obtain the sample complexity as $O(\kappa^4 /\epsilon^4 )$ to achieve $\epsilon$-accuracy solution, which is much slower than existing compositional minimization or minimax optimization methods. To obtain the optimal sample complexity without requiring large batch sizes, our proposed method modifies the STORM \citep{cutkosky2019momentum, jiang2022optimal} estimator and incorporates gradient projection techniques. This modification ensures that the Jacobians can be bounded for the minimization sub-problem and the gradients are projected onto a convex set for the maximization problem, thereby reducing gradient estimation errors.

\subsection{Nested STOchastic Recursive Momentum (NSTORM)}
In this subsection, we will present our proposed method, named the Nested STOchastic Recursive Momentum (NSTORM), to solve the compositional minimax problem in \eqref{Eq:MinimaxObjective}. We aim to find an $\epsilon$-accuracy to achieve low sample complexity without using large batch sizes.

Our proposed NSTORM method is illustrated in Algorithm~\ref{ALG:SCGDA}. Inspired by STROM \citep{cutkosky2019momentum}, the NSTORM method leverages similar variance-reduced estimators for both the two sub-problems in \eqref{Eq:MinimaxObjective}. Note that our goal is to find an $\epsilon$-stationary point with low sample complexity. As we mentioned before because we cannot obtain the unbiased estimation of $\nabla_x f(g(x_t), y_t)$, we use estimators $u_t$ and $v'_t$ to estimate the inner function $g(x_t)$ and its gradient $\nabla g(x_t)$, respectively. In each iteration $t$, the two estimators $u_t$ and $v'_t$ can be computed by:
\begin{equation}\label{Eq:uestimate}
\begin{split}
    u_{t} = (1-\beta_t) & u_{t-1} + \beta_t g(x_{t}; \xi_{t}) + (1-\beta_t)(g(x_{t}; \xi_{t})- g(x_{t-1}; \xi_{t})),
\end{split}
\end{equation}
\begin{equation}\label{Eq:vestimate}
\begin{split}
    v'_t = \Pi_{C_{g}}&[ (1-\beta_t )v'_{t-1} + \beta_t \nabla g(x_{t} ;\xi_{t})  + (1-\beta_t )(\nabla g(x_{t} ;\xi_{t} ) - \nabla g(x_{t-1};\xi_{t}))],
\end{split}
\end{equation}
where $0< \beta_t <1$. Note that the projection operation $\Pi_{C_g}(x) = \mathop{\arg\min}_{\|w\| \leq C_g} \|w- x\|^2$ aims to bound the error of the stochastic gradient estimator, which also facilitates the outer level estimator. More specifically, we need to reduce the variance of the estimator (because true gradients are in the projected domain, projection does not degrade the analysis); on the other side, we must avoid the variance of the estimator accumulating after the outer level, i.e., maximization sub-problem.

\begin{algorithm}[t!]
	\caption{Illustration of NSTORM method.}
	\label{ALG:SCGDA}
 \renewcommand{\algorithmicrequire}{\textbf{Initialization:}}
    \renewcommand{\algorithmicensure}{\textbf{Output:}}
	\begin{algorithmic}[1]
		\REQUIRE $x_{1}$, $y_{1}= y^{*}(x_{1})$, $\gamma$, $\beta_t$, $\alpha_t$, $\eta_t$
        \FOR{$t=1$ to $T$ }
            \STATE Draw a sample $\xi_{t}$;
            \IF {$t = 1$}
            \STATE $u_{t} = g(x_{t}; \xi_{t})$, $v'_{t} = \nabla_{x} g(x_{t}; \xi_{t})$, $v''_{t} = \nabla_{g} f(u_{t} ,y_{t}; \xi_{t})$, and $w_{t} = \nabla_{y} f(u_{t} ,y_{t};\xi_{t})$;
            \ELSE
                \STATE Compute estimators $u_t$ and $v'_t$ by \eqref{Eq:uestimate} and \eqref{Eq:vestimate};
                \STATE Draw another sample $\zeta_{t}$;
                \STATE Compute the estimator $v''_t$ by \eqref{Eq:OuterEstimation};
                \STATE $v_{t} = v'_{t} v''_{t}$;
                \STATE Compute the estimator $w_t$ by \eqref{Eq:westimate};
            \ENDIF
            \STATE Update $x_{t+1}$ and $y_{t+1}$ by \eqref{Eq:Updatexy};
        \ENDFOR
	\end{algorithmic}
\end{algorithm}

For the outer level function, if we use the same strategy to compute the gradient as SCGDA \cite{gao2021convergence}, i.e., $v_t = (v'_t )^{\top}\nabla_g f(u_t,y_t; \zeta_t)$, we have to use large batches and the variance produced by $v_t$ cannot be bounded. Therefore, we also estimate the outer function by the NSTORM method, which results in a tighter bound for $\mathbb{E}[\|v_t - \nabla_{x} f(g(x_t), y_t)\|^2]$. As such, we estimate the gradient $\nabla_g f(u_t ,y_t)$ by $v''_t$, which can be computed by:
\begin{equation}\label{Eq:OuterEstimation}
\begin{split}
    &v''_t = (1-\beta_t )v''_{t-1} + \beta_t \nabla_g f(u_t,y_t;\zeta_t ) + (1-\beta_t )(\nabla_g f(u_t, y_t;\zeta_t ) - \nabla_g f(u_{t-1}, y_{t-1};\zeta_t )).
\end{split}
\end{equation}
Based on the chain rule, the estimated compositional gradient is equal to $v'_t v''_t$, i.e., $v_t = v'_t v''_t$. To avoid using large batches, we estimate the outer function $\nabla_{y} f(g(x_t), y_t)$ by $w_t$ based on the NSTORM estimator, which can be computed by: 
\begin{equation}\label{Eq:westimate}
\begin{split}
    & w_t = (1-\alpha_t )w_{t-1} + \alpha_t \nabla_{y} f(u_t, y_t; \zeta_t)  + (1-\alpha_t )(\nabla_{y} f(u_t ,y_t;\zeta_t) - \nabla_{y} f(u_{t-1}, y_{t-1};\zeta_t)),
\end{split}
\end{equation}
where $0<\alpha_t <1$. After obtaining the estimators $v_t$ and $w_t$, we can use the following strategy to update the parameters $x$ and $y$ in the compositional minimax problem:
\begin{align}\label{Eq:Updatexy}
     x_{t+1} = x_t-\gamma \eta_t v_t , ~~~ y_{t+1} = y_t+\eta_t w_t ,
\end{align}
\begin{wrapfigure}{h}{0.4\textwidth}
  \vspace{-0.45cm}
  \centering
  \includegraphics[width=0.4\textwidth]{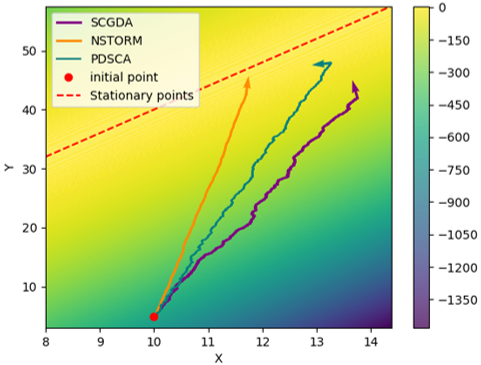}
  \vspace{-0.6cm}
  \caption{Trajectories of different methods for the compositional minimax optimization.}\label{toy_example}
\end{wrapfigure}
where $\gamma$ is the step size, and $\eta_{t}$ is the learning rate. Note that in the first iteration, we evaluate all estimators $u_1, v'_1, v''_1 , w_1$ by directly computing inner level function and gradients, i.e., line 4 in Algorithm~\ref{ALG:SCGDA}. The reason we choose two level estimators is to avoid using large batches, and only need to draw two samples, i.e., $\xi_t$ and $\zeta_t$, to calculate estimators for updating $x_t$ and $y_t$, respectively. The common idea to achieve the optimal solution of minimax \cite{lin2020gradient} is that the step size of $x$ should be smaller than $y$. In addition, the compositional minimization sub-problem will generate larger errors, which incurs more challenges for NSTORM. Particularly, in \eqref{Eq:Updatexy}, if we simply set the same step sizes of $x$ and $y$, our proposed NSTORM method may fail to converge, which is confirmed in the subsequent proof. Therefore, we set $\gamma < 1$ to ensure that the step size of $x$ is less than $y$.

To clearly explain the advantages of NSTORM, a toy example is illustrated in Figure~\ref{toy_example}. Consider the following concrete example of a nonconvex-strongly-concave function.: $f(g(x),y) = -2g(x)^2 + 2g(x)y - \frac{1}{2}y^2$, where $g(x) = 2x$. In Figure~\ref{toy_example}, we simulate these stochastic oracles by adding noise when obtaining function gradients and function values. This function obtains the biased estimation in the minimization sub-problem affording the problem in \eqref{Eq:MinimaxObjective}. It can be observed that NSTORM performs more robustness on noisy and biased estimation, which brings up an opportunity to obtain the optimal solution with shorter and smoother paths compared to other benchmarks.

\subsection{Convergence Analysis of NSTORM}
In what follows, we will prove the convergence rate of our proposed NSTROM method in Algorithm~\ref{ALG:SCGDA}. We first state some commonly-used assumptions for compositional and minimax optimizations \cite{gao2021convergence, wang2017stochastic, xian2021faster, yuan2019stochastic, zhang2020optimal} to facilitate our convergence analysis. In order to simplify the notations and make the paper coherence, we denote $\nabla f(a,b) = (\nabla_{a} f(a, b), \nabla_{b} f(a,b))$ for $(a,b) \in \mathcal{A}\times \mathcal{B}$ in the following assumptions, where $\mathcal{A} = \{g(x)|x \in \mathcal{X}\}$ and $\mathcal{B} = \mathcal{Y}$.

\begin{assumption}\label{Ass:Smooth}
(Smoothness) There exists a constant $L >0$, such that 
\begin{equation}
\|\nabla f(a_1 ,b_1 ) -\nabla f(a_2, b_2 )\| \leq L \|(a_1 ,b_1 ) - (a_2 ,b_2 )\|,\nonumber
\end{equation}
where $\forall (a_1 , b_1 ), (a_2 , b_2 )\in \mathcal{A} \times \mathcal{B}$. In addition, we assume that there exists a constant $L_g >0$, $L_f > 0$ such that
\begin{equation}
\begin{split}
\|\nabla g(x_1; \xi ) - \nabla g(x_2; \xi)\| \leq L_g \|x_1 -x_2 \|,~~\| g(x_1; \xi ) -  g(x_2 ; \xi)\|  \leq L_f \|x_1 -x_2 \|,   \nonumber
\end{split}
\end{equation}
where $\forall x_1 ,x_2 \in \mathcal{X}$.
\end{assumption}

\begin{assumption}\label{Ass:BoundedGrad}
(Bounded Gradient) There exist two constants $C_g >1$ and $C_f >0$, where the two gradients can be bounded by $\mathbb{E}[\|\nabla g(x)\|^2 ] \leq C_g^2 , ~~\forall x \in \mathcal{X}$ and $\mathbb{E}[\|\nabla f(a, b)\|^2 ] \leq C_f^2 , ~~\forall (a, b) \in \mathcal{A} \times \mathcal{B}$.
\end{assumption}

\begin{assumption}\label{Ass:BoundedVariance}
(Bounded Variance) There exist three constants $\sigma_f >0$, $\sigma_g >0$, and $\sigma_{g'} >0$, where the three kinds of variance can be bounded by:
\begin{equation}
\begin{split}
    \mathbb{E}[\|\nabla f(a,b;\zeta) - \nabla f(a, b)&\|^2 ] \leq \sigma_f^2 , ~~ \forall (a, b) \in \mathcal{A}\times \mathcal{B},\\\nonumber
    \mathbb{E}[\|\nabla g(x;\xi) - \nabla g(x)\|^2 &]\leq \sigma_{g'}^2 , ~~x \in \mathcal{X},\\
    \mathbb{E}[\|g(x;\xi) - g(x)\|^2 ]\leq &\sigma_{g}^2 , ~~x \in \mathcal{X}.\nonumber
\end{split}
\end{equation}
\end{assumption}

\begin{assumption}\label{Ass:StrongConcave}
(Strongly Concave) There exists a constant $\mu >0$, such that 
\begin{equation}
f(a, b_1 ) \leq f(a, b_2 ) + \langle \nabla_{b} f(a, b_2 ) , a_1 - b_2 \rangle - \frac{\mu}{2}\|b_1 - b_2 \|^2,\nonumber
\end{equation}
where $\forall a \in \mathcal{A}$ and $\forall b_1 ,b_2 \in \mathcal{B}$.
\end{assumption}

Similar to existing minimax studies \cite{lin2020gradient, xian2021faster}, we also use $\epsilon$-point of $\nabla \Phi(x)$, i.e., $\|\nabla \Phi(x)\| \leq \epsilon$ as the convergence criterion in our focused compositional minimax problem, where $\Phi (x) = \max_{y \in \mathcal{Y}}f(g(x), y)$ and $y^* (x) = \arg\max_{y \in \mathcal{Y}} f(g(x), y)$. We demonstrate that $\Phi(x)$ is differentiable and $(C_g^2 L\kappa+C_f L_g )$-smooth, where $\kappa = L/\mu$, and $y^* (x)$ is $\kappa$-Lipschitz, which has some differences compared to the minimax optimization \cite{lin2020gradient}. We defer detailed proof in the supplementary. 
% Therefore, we propose the following two lemmas to show the property: 
% \begin{lemma}\label{lemma:y^*smooth}
% Under the Assumptions~\ref{Ass:Smooth}, \ref{Ass:BoundedGrad} and \ref{Ass:StrongConcave}, $\forall x_1, x_2 \in \mathcal{X}$ and $\forall y \in \mathcal{Y}$, we can obtain the following:
% \begin{equation}
% \|y^*(x_1)- y^*(x_2)\| \leq C_g \kappa \|x_1- x_2\|.\nonumber
% \end{equation}
% \end{lemma}

% \begin{lemma}\label{lemma:phismooth}
% Under the Assumptions~\ref{Ass:Smooth}-\ref{Ass:StrongConcave} and Lemma~\ref{lemma:y^*smooth}, $\forall x_1, x_2 \in \mathcal{X}$, we can obtain the following:
% \begin{equation}
% \|\nabla \Phi(x_1) - \nabla \Phi(x_2)\| \leq (C_g^2 L\kappa+C_f L_g)\|x_1 - x_2\|.\nonumber
% \end{equation}
% \end{lemma}

% \noindent{\bf Remark 1.} From Lemmas~\ref{lemma:y^*smooth} and \ref{lemma:phismooth}, we can see that the smooth bound of $y^* (x)$ and $\Phi (x)$ in the compositional minimax problem are looser than the conventional minimax problem. Due to the biased gradient estimator, we need to use the Assumption~\ref{Ass:BoundedGrad}, and we can see that the smooth bounds include $C_g$ and $C_f$. Note that we define $L_{\Phi} = C_g^2 L\kappa+C_f L_g$ in short.

Now, we can obtain the following convergence result of our proposed NSTORM method in Algorithm~\ref{ALG:SCGDA} to solve the compositional minimax problem in \eqref{Eq:MinimaxObjective}:

\begin{theorem}\label{The:ConvergenceNonConvex}
    Under the Assumptions~\ref{Ass:Smooth}-\ref{Ass:StrongConcave}, for Algorithm~\ref{ALG:SCGDA}, by setting $\eta_{t} = \frac{1}{(m+t)^{1/3}}$, $m> \operatorname{max}\{125L^3, 8\gamma^3L_\Phi^3, (12L^2c_1^2+4L^2c_2^2)^3, c_1^3, c_2^3\}$, $c_1 \geq 2+4\gamma(C_f^2+C_g^2)+2C_g^2L^2\gamma$, $c_2\geq \frac{2}{3}+180L^2+\frac{36\gamma C_g^2L^2}{\mu^2}$, $\beta_t = c_1\eta_{t-1}^2$, $\alpha_t = c_2\eta_{t-1}^2$, $ 0 < \gamma \leq \frac{1}{\sqrt{B^2+20\kappa^4C_g^2}}$, where $B=\frac{100C_g^2L^4}{\mu^2}+2L_f^2+2L_g^2+12L^2L^2_f+4L^2C_g^2$, we can obtain the following:
    \begin{equation}  
    \frac{1}{T}\sum_{t=1}^{T}\mathbb{E}\|\nabla\Phi(x_t)\| \leq  \frac{m^{1/6}\sqrt{M}}{\sqrt{\gamma T}} +\frac{\sqrt{M}}{\sqrt{\gamma}T^{1/3}},\nonumber
    \end{equation}
where $M = \Phi(x_1)-\Phi_*+ \sigma_g^2+\sigma_{g^{'}}^2+\sigma_f^2 +L^2\sigma_g^2+(c_1^2(2\sigma_g^2+2\sigma^{2}_{g^{'}} + 2\sigma_f^2 + 4L^2\sigma_g^2)+4\sigma_f^2c_2^2)\ln(T+m)$ and $\Phi_*$ represents the minimum value of $\Phi(x)$. 
\end{theorem}

\noindent{\bf Remark 1.} As discussed in the previous section, our proposed NSTORM method in Algorithm~\ref{ALG:SCGDA} results in tighter bounds for all variances, i.e., $\mathbb{E}[\|u_t - g(x_t )\|^2 ]$, $\mathbb{E}[\|v'_t -\nabla g(x_t )\|^2 ]$, $\mathbb{E}[\|v''_{t} -\nabla_g f(u_t,y_t)\|^2]$ and $\mathbb{E}[\|w_{t} - \nabla_{y} f(g(x_{t}) , y_t )\|^2 ]$, which makes NSTROM method converge faster comparing with existing studies. Therefore, it is very important to show the upper bounds of these variances. We will show the detailed analysis in the supplementary. 

\noindent{\bf Remark 2.} Without loss of generality, let $m = O(1)$, we have $M = O(\ln (m+T)) = O(1)$. Therefore, our proposed NSTORM method has a convergence rate of $O\left(1/T^{1/3}\right)$. Let $\frac{1}{T}\sum_{t=1}^{T}\mathbb{E}[\nabla \Phi(x_t )] = O\left(1/T^{1/3}\right) \leq \epsilon$, we have $T = O\left(\kappa^3/ \epsilon^3 \right)$. Because we only need two samples, i.e., $O(1)$, to estimate the stochastic to compute the gradient in each iteration, and need $T$ iterations. Therefore, our NSTORM method requires sample complexity of $O\left(\kappa^3 /\epsilon^3 \right)$ for finding an $\epsilon$-accuracy point of the compositional minimax problem in \eqref{Eq:MinimaxObjective}. Because the SCGDA method \cite{gao2021convergence} only achieves $O\left(\kappa^4 /\epsilon^4 \right)$ with requiring a large batch size as $O(T)$, it is observed that our proposed NSTROM method improves the convergence rate significantly.

\noindent{\bf Remark 3.} It is worth noting that if we moderate the assumption of $f(g(x), y)$ with respect to $y$ to follow the PL-condition instead of strongly-concave in Assumption~\ref{Ass:StrongConcave}, NSTORM can also obtain the sample complexity, i.e., $O(\kappa ^3/\epsilon^3)$. To the best of our knowledge, this is the first study to design the method for compositional minimax optimization, which highlights the extensibility and applicability of NSTORM. The detailed description will be shown in the supplementary.
% if we replace Assumption \ref{Ass:StrongConcave} with the less stringent Assumption \ref{Ass:mu-pl-condition}, the function $f(g(x),y)$ is not necessarily concave with respect to y. However, the convergence rate of NSTORM remains at $O(1/T^{1/3})$ by adjusting the following parameters, $m > \max\{1, c_1^3, c_2^3, (24L^2c_1^2+8L^2c_2^2)^3\}$,  $c_1 \geq 2+360L^2\lambda\gamma(C_f^2+C_g^2)$,  $c_2 \geq\frac{2}{3}+40L^2\lambda^2$, $0<\lambda \geq \frac{m^{1/3}}{2L}$ and $0 < \gamma < \operatorname{min}\{\frac{\lambda\mu^2}{8C_g^2L^2},\frac{B}{10C_g^2L^2},\frac{\lambda\mu^2}{9C_g^2L^2}, \frac{4}{BA_1}\}$, where $A_1 = L_f^2+L_g^2+6L^2L_f^2+2L^2C_g^2$. }

\section{ADAptive-NSTORM (ADA-NSTORM)}
\subsection{Learning Procedure of ADA-NSTORM}
According to the analysis of variance in the NSTORM method, due to the large variance of the two-level estimator, we must select a smaller learning step  to update parameters $x$ and $y$. As a result, this degrades the applicability of NSTORM. Adaptive learning rates \cite{huang2020gradient, huang2023adagda} have been developed to accelerate many optimization methods including (stochastic) gradient-based methods based on momentum technology. Therefore, we leverage adaptive learning rates in NSTORM and propose ADAptive NSTORM (ADA-NSTORM) method, which is illustrated in Algorithm~\ref{alg:AOA}. 

\begin{algorithm}[t!]
    \caption{Illustration of ADA-NSTORM method.}
    \label{alg:AOA}
    \renewcommand{\algorithmicrequire}{\textbf{Initialization:}}
    \renewcommand{\algorithmicensure}{\textbf{Output:}}
    \begin{algorithmic}[1]
        \REQUIRE $x_{1}$, $y_{1}=y^{*}(x_{1})$, $\gamma$, $\lambda$, $\beta_t$, $\alpha_t$, $\eta_t$,  $\tau$, $a_t$, $b_t$.
        \FOR{$t=1$ to $T$ }
            \STATE Draw a sample $\xi_{t}$;
            \IF {$t = 1$}
                \STATE $u_{t} = g(x_{t}; \xi_{t})$, $v'_{t} = \nabla_{x} g(x_{t}; \xi_{t})$, $v''_{t} = \nabla_{g} f(u_{t} ,y_{t}; \xi_{t})$, and $w_{t} = \nabla_{y} f(u_{t} ,y_{t};\xi_{t})$;
            \ELSE
                \STATE Compute the estimator $u_t$ by \eqref{Eq:uestimate};
                \STATE Compute the estimator $v'_t$ by \eqref{Eq:vestimate};
                \STATE Draw another sample $\zeta_{t}$;
                \STATE Compute the estimator $v''_t$ by \eqref{Eq:OuterEstimation};
                \STATE $v_{t} = v'_{t} v''_{t}$;
                \STATE Compute the estimator $w_t$ by \eqref{Eq:westimate};
            \ENDIF
            \STATE Generate the adaptive matrices $A_t \in \mathbb{R}^{d \times d}$ and $B_t \in \mathbb{R}^{p \times p}$ by \eqref{Eq:AdamA} and \eqref{Eq:AdamB} (Adam);
            \STATE Compute the $\tilde{x}_{t+1}$ and $\tilde{y}_{t+1}$ by \eqref{Eq:Tx} and \eqref{Eq:Ty};
            \STATE Compute the $x_{t+1}$ and $y_{t+1}$ by \eqref{Eq:xy};
        \ENDFOR
    \end{algorithmic}
\end{algorithm}

In each iteration $t$, we first use the NSTORM method to update all estimators related to the inner/outer functions and their gradients. At Line 13 in Algorithm~\ref{alg:AOA}, we generate the adaptive matrices $A_t$ and $B_t$ for the two variables $x$ and $y$, respectively. In particular, the general adaptive matrix $A_t \succeq \rho I_{d}$ is updated for the variable $x$, and the global adaptive matrix $B_t $  is for $y$. It is worth noting that we can generate the two matrices $A_t$ and $B_t$ by a class of adaptive learning rates generators such as Adam \cite{kingma2014adam}, AdaBelief, \cite{zhuang2020adabelief},  AMSGrad \cite{j.2018on}, AdaBound \cite{luo2018adaptive}. Due to the space limitation, we only discuss Adam \cite{kingma2014adam} in the main paper, and other Ada-type generators will be deferred in the supplementary. In particular, the Adam generator can be computed by:
\begin{equation}\label{Eq:AdamA}
    a_{t}=\tau a_{t-1}+(1-\tau)v_t^2 ,~~ A_{t}=\operatorname{diag}\left(\sqrt{a_{t}}+\rho\right),
\end{equation}
\begin{equation} \label{Eq:AdamB}
    b_{t}=\tau b_{t-1}+(1-\tau) w_t^2 ,~~ B_{t}=\operatorname{diag}(\sqrt{b_t}+\rho),
\end{equation}
where $t \geq 1$, $\tau \in (0,1)$ and $\rho >0$. Due to the biased full gradient in the compositional minimax problem, we leverage the gradient estimator $v_t$ and $w_t$ to update adaptive matrices instead of simply using the gradient, i.e., $\nabla_x f(x_t, y_t; \xi_t )$ and $\nabla_y f(x_t, y_t; \zeta_t )$ \citep{huang2020gradient, huang2023adagda}. After obtaining adaptive learning matrices $A_t$ and $B_t$, we use adaptive stochastic gradient descent to update the parameters $x$ and $y$ as follows:
\begin{equation*}\label{Eq:Tx}
\begin{aligned}
    &\tilde{x}_{t+1} =x_{t}-\gamma A_{t}^{-1} v_{t} = \operatorname{argmin}_{x \in \mathbb{R}^{d}}\left\{\left\langle x, v_{t}\right\rangle+\frac{1}{2 \gamma}\left(x-x_{t}\right)^{T} A_{t}\left(x-x_{t}\right)\right\},
\end{aligned}    
\end{equation*}

\begin{equation*}\label{Eq:Ty}
\begin{aligned}
    &\tilde{y}_{t+1} =y_{t}-\lambda B_{t}^{-1} w_{t} =\operatorname{argmin}_{y \in \mathbb{P}^{d}}\left\{\left\langle y, w_{t}\right\rangle+\frac{1}{2 \lambda}\left(y-y_{t}\right)^{T} B_{t}\left(y-y_{t}\right)\right\},
\end{aligned}    
\end{equation*}

where $\gamma$ and $\lambda$ are step sizes for updating $\tilde{x}$ and $\tilde{y}$, respectively. At Line 15 in Algorithm 2, we use the momentum iteration to further update the primal variable $x$ and the dual variable $y$ as follows: 
\begin{equation}\label{Eq:xy}
\begin{split}
    x_{t+1} &= x_{t} + \eta_t(\tilde{x}_{t+1}-x_t),\\
    y_{t+1} &=y_{t}+\eta_t (\tilde{y}_{t+1}-y_t).
\end{split}
\end{equation}

\subsection{Convergence analysis of ADA-STORM}
We will introduce one additional assumption to facilitate the convergence analysis of ADA-NSTORM.
\begin{assumption}\label{ass ab bound}
In Algorithm~\ref{alg:AOA}, the adaptive matrices $A_t$, $\forall t \geq 1$ for updating the variables $x$ satisfies $A^{T}_t =A_t $ and $\lambda_{min}(A_t) \geq  \rho \ > 0$, where $\rho$ is an appropriate positive number. We consider the adaptive matrics $B_t $, $\forall t \geq 1$ for updating the variables $y$ satisfies $\hat{b}I_p \geq B_t  \geq bI_p > 0$, where $I_p$ denotes a $d$-dimensional identity matrix.
\end{assumption}
\noindent{\bf Remark 4.} Assumption~\ref{ass ab bound} ensures that the adaptive matrices $A_t$, $\forall t \geq 1$ are positive definite, which is widely used in \cite{huang2020gradient, huang2023adagda, huang2023enhanced}. This Assumption also guarantees that the global adaptive matrices $B_t$, $\forall t \geq 1$ are positive definite and bounded, resulting in mild conditions. To support the mildness of this assumption, we will empirically show that the learning performance does not have obvious changes by varying the bound of $a_t$ and $b_t$. In particular, we also show the requirement of some popular Adam-type generators in the supplementary and provide the corresponding advice on compositional minimax optimization.
% were bounded by adding projections to $a_t$ and $b_t$ and constantly changing the projection boundary values, and finding that the projection boundary continued to expand, while the final result changed very little, indicating that $a_t$ and $b_t$ are bounded, which proved the moderation of the Assumption \ref{ass ab bound}.}

% For example, in the context of the problem $\max_{x}\mathbb{E}[f(x;\xi)]$, \cite{huang2023adagda} applies a global adaptive learning rate to the update form $x_t = x_{t-1} - \eta \frac{\nabla f(x_{t-1};\xi_{t-1})}{b_t}, b_t^2 = b_{t-1}^2 + \|\nabla f(x_{t-1}; \xi_{t-1})\|^2 , b_0 > 0, \eta>0$ for all $t \geq 1$. This update form can be equivalently expressed as $x_t = x_{t-1} - \eta B_t^{-1} \nabla f(x_{t-1}; \xi_{t-1})$ with $B_t = b_t I_p$ and $b_t \geq \cdots \geq b_0 >0$. Similarly, \cite{huang2023enhanced} utilizes a global adaptive learning rate in the update form $x_{t+1} = x_t -\eta g_t/b_t$, where $g_t$ is the stochastic gradient and $b_t = \left(w+ \sum_{i=1}^{t}\|\nabla f(x_i ;\xi_i )\|^2\right)^\alpha /k$, $k>0$, $w>0$ and $\alpha \in (0,1)$. This can be rewritten as $x_{t+1} = x_t - \eta B_t^{-1}g_t$ with $B_t = b_t I_p$ and $b_t \geq \cdots \geq b_0 = \frac{w^\alpha}{k} >0$. In addition, in the context of the problem, $\min_{x} f(x) = \mathbb{E}[f(x;\xi)]$ to approach the stationary point, i.e., $\nabla f(x) =0$ or even $\nabla f(x;\xi) =0$, $\forall \xi$, these global adaptive learning rates are generally bounded, i.e., $\hat{b} \geq b_t \geq b >0$, $\forall t \geq 1$.

\begin{figure*}
    \centering
    \includegraphics[width=\textwidth]{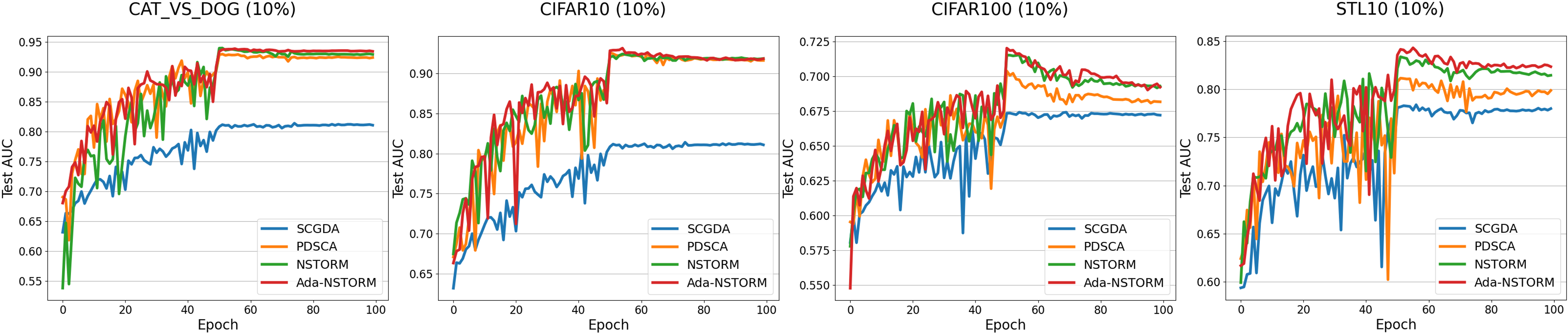}
    \caption{Convergence performance on four benchmark datasets with an imbalance ratio of 10\%}
    \label{fig:Convergence_10}
\end{figure*}

\begin{theorem}\label{theorem ago2}
    Given Assumptions~\ref{Ass:Smooth}-\ref{ass ab bound}, for Algorithm~\ref{alg:AOA}, by setting $\eta_t=\frac{1}{(m+t)^{1/3}}$,  $m > \max \{\frac{8L_{\Phi}^3\gamma^3}{\rho^3},(10L^2c_1^2+4L^2c_2^2)^3, c_1^3, c_2^3\}$,  $c_1 \geq 2+\frac{5\gamma(2C_f^2+2C_g^2+C_g^2L^2)}{\rho}$ , $c_2 \geq \frac{2}{3} + \frac{125\lambda L^2}{2\mu b}+\frac{125\gamma C_g^2\kappa^2\hat{b}}{3b}$, $\gamma  \leq \frac{\rho}{4\sqrt{B_1^2+\rho B_2}}$, where $B_1 = \frac{50C_g^2\kappa^4\hat{b}}{\lambda^2}$, $B_2 = \frac{70\kappa^3L}{\lambda}+2L_f^2+2L_g^2+12L^2L_f^2+4L^2C_g^2$, $\beta_{t+1}=c_1\eta_t^2\leq c_1\eta_t < 1$, $\alpha_{t+1}=c_2\eta_t^2\leq c_2\eta_t < 1$, $0 < \lambda \leq \frac{b}{6L}$, we can obtain the following:
    \begin{equation}
    \begin{split}
    &\frac{1}{T}\sum^T\limits_{t=1}\mathbb{E}[\|\nabla\Phi(x_t)\|] \\
    &\leq \sqrt{\frac{1}{T}\sum^T\limits_{t=1}\mathbb{E}[\|A_t\|^2]} \cdot \left(\frac{2\sqrt{5Mm^{1/3}}}{\sqrt{\gamma T}}+\frac{2\sqrt{5M}}{\sqrt{\gamma}T^{1/3}}\right),\nonumber
    \end{split}    
    \end{equation}
    where $M =  (\Phi(x_1)-\Phi_*+ \sigma_g^2+\sigma_{g^{'}}^2+\sigma_f^2 +L^2\sigma_g^2 )/\rho + ((2c_1^2(\sigma^2_g + \sigma^{2}_{g^{'}} + \sigma^2_f + 6L^2\sigma^2_g ) +4c_2^2\sigma_f^2)\ln(m+T))/\rho$ and $\Phi_*$ represents the minimum value of $\Phi(x)$. 
\end{theorem}

\noindent{\bf Remark 5.} Without loss of generality, let $b = O(1)$ and $\hat{b} = O(1)$. The proof of Theorem~\ref{theorem ago2} is deferred in supplementary. From Theorem~\ref{theorem ago2}, given $\gamma  \leq \frac{\rho}{4\sqrt{B_1^2+\rho B_2}}$, where $B_1 = \frac{50C_g^2\kappa^4\hat{b}}{\lambda^2}$, $B_2 = \frac{70\kappa^3L}{\lambda}+2L_f^2+2L_g^2+12L^2L_f^2+4L^2C_g^2$, $\lambda \leq \frac{b}{6L}$,  we can see that $\gamma = O(1/\kappa^2)$, $\lambda = O(1/L)$, $M = O(1)$. Then, we can get the convergence rate $O\left(1/ T^{1/3}\right)$. Therefore, to achieve $\epsilon$-accuracy solution, the total sample complexity is $ O(\kappa^3 / \epsilon^3 )$. It is worth noting that the term $\sqrt{\frac{1}{T}\sum^T\limits_{t=1}\mathbb{E}[\|A_t\|^2]}$ is bounded to
the existing adaptive learning rates in Adam algorithm \cite{kingma2014adam} and so on. \cite{yuan2022compositional} develops a PDSCA method with adaptive learning rates to approach the compositional minimax optimization problem and achieves $O(1/\epsilon^4)$ complexity. However, \cite{yuan2022compositional} sets $\eta = O(1/\sqrt{T})$, which is difficult to know in practice.

\section{Experiments}
In this section, we present the results of our experiments that assess the performance of two proposed methods: NSTORM and ADA-NSTORM in the deep AUC problem \cite{yuan2020robust, yuan2022compositional}. To establish a benchmark, we compare our proposed methods against existing compositional minimax methods, namely SCGDA \citep{gao2021convergence} and PDSCA \cite{yuan2022compositional}. To optimize the AUC loss, the outer function corresponds to an AUC loss and the inner function represents a gradient descent step for minimizing a cross-entropy loss, as \cite{yuan2022compositional}. The deep AUC problem can be formulated as follows:
\begin{equation}\label{AUC loss}
\begin{aligned}
     &\min_{x, a, b} \max_{y \in \Omega} \Theta\left(x-\alpha \nabla L_{\mathrm{AVG}}(x), a, b, y\right).
\end{aligned}
\end{equation}
The function $\Theta$ is to optimize the AUC score. The inner function $x-\alpha \nabla L_{\mathrm{AVG}}(x)$ aims to optimize the average cross-entropy loss $L_{\mathrm{AVG}}$. $\alpha$ is a hyper-parameter. 

Rather than the deep AUC problem, we also evaluate our proposed methods on the risk-averse portfolio optimization problem \cite{shapiro2021lectures, zhang2021taming} and the policy evaluation in reinforcement learning \cite{yuan2019stochastic, zhang2019stochastic}. All experimental setups and results will be shown in the supplementary.

\noindent{\bf Learning Model and Datasets.}
We employ four distinct image classification datasets in our study: CAT\_vs\_DOG, CIFAR10, CIFAR100 \cite{krizhevsky2009learning}, and STL10 \cite{coates2011analysis}. To create imbalanced binary variants prioritizing AUC optimization, we followed \cite{yuan2020robust} methodology. Similarly, as in \cite{yuan2022compositional}, ResNet20 \cite{he2016deep} was used. Weight decay was consistently set to 1e-4. Each method was trained with batch size 128, spanning 100 epochs. We varied parameter $m$ ({50, 500, 5000}) and set $\gamma$ ({1, 0.9, 0.5}). Learning rate $\eta_t$ reduced by 10 at 50\% and 75\% training. Also, $\beta$ is set to 0.9. For robustness, each experiment was conducted thrice with distinct seeds, computing mean and standard deviations. Notably, the ablation study focused on the CIFAR100 dataset, 10\% imbalanced ratio, as detailed in the main paper.

% \usepackage{graphicx}
\begin{table*}[t!]
\centering

%\resizebox{\columnwidth}{!}{%
\begin{small}
\begin{tabular}{c|cccc|cccc}
\hline
\multicolumn{1}{l|}{Datasets} & \multicolumn{4}{c|}{CAT\_vs\_DOG} & \multicolumn{4}{c}{CIFAR10}                                                \\ \hline
imratio           & SCGDA        & PDSCA                 & NSTORM         & ADA-NSTORM         & SCGDA        & PDSCA                 & NSTORM         & ADA-NSTORM      \\ \hline
\multirow{2}{*}{1\%}         & 0.750 & \textbf{0.792} & 0.786 & 0.786  & 0.679 & 0.699 & 0.689 & \textbf{0.703}          \\ 
   & $\pm$0.004 & $\pm$\textbf{0.009}  &  $\pm$0.001 & $\pm$0.011 & $\pm$0.011 & $\pm$0.008 &  $\pm$0.003 & $\pm$\textbf{0.008} \\\hline
\multirow{2}{*}{5\%}    & 0.826 & 0.890 & 0.895  & \textbf{0.901} & 0.782 & 0.878 & 0.882  & \textbf{0.894}  \\ 
& $\pm$0.006 & $\pm$0.006 & $\pm$0.006 & \textbf{$\pm$0.004} & $\pm$0.006 & $\pm$0.003 & $\pm$0.002 & \textbf{$\pm$0.005}\\\hline
\multirow{2}{*}{10\%}      & 0.857 & 0.932   & 0.932   & \textbf{0.933} & 0.818 & 0.926   & 0.926   & \textbf{0.931} \\
& $\pm$0.009 & $\pm$0.002 & $\pm$0.005 & $\pm$\textbf{0.002}
& $\pm$0.004 & $\pm$0.001 & $\pm$0.001 & $\pm$\textbf{0.001} \\\hline
\multirow{2}{*}{30\%}   & 0.897 & 0.969  & 0.970  &\textbf{0.972} & 0.882 & 0.953  & 0.953  & \textbf{0.955}  \\ 
& $\pm$0.008 & $\pm$0.002 & $\pm$0.001 & $\pm$\textbf{0.001} & $\pm$0.005 & $\pm$0.001 & $\pm$0.002 & $\pm$\textbf{0.001} \\\hline
\multicolumn{1}{l|}{Datasets} & \multicolumn{4}{c|}{CIFAR100} & \multicolumn{4}{c}{STL10}                               \\ \hline
imratio           & SCGDA        & PDSCA                 & NSTORM         & ADA-NSTORM         & SCGDA        & PDSCA                 & NSTORM         & ADA-NSTORM      \\ \hline
\multirow{2}{*}{1\%}    & 0.588 & 0.583 & 0.583 & \textbf{0.593 }   & 0.670 & \textbf{0.682} & 0.659 & 0.657     \\ 
   & $\pm$0.007 & $\pm$0.004  &  $\pm$0.007 & $\pm$\textbf{0.002} & $\pm$0.006 & $\pm$\textbf{0.016}  &  $\pm$0.013 & $\pm$0.003 \\\hline
\multirow{2}{*}{5\%}    & 0.641 & 0.651& 0.648  & \textbf{0.655} & 0.734 & 0.775& 0.779  & \textbf{0.781}  \\ 
 & $\pm$0.007 & $\pm$0.006 & $\pm$0.003 & \textbf{$\pm$0.007} & $\pm$0.007 & $\pm$0.003 & $\pm$0.005 & \textbf{$\pm$0.007}\\\hline
\multirow{2}{*}{10\%}     & 0.673 & 0.708   & 0.709   & \textbf{0.715}  & 0.779 & 0.824   & 0.827   & \textbf{0.833}  \\
& $\pm$0.005 & $\pm$0.006 & $\pm$0.002 & $\pm$\textbf{0.006} 
& $\pm$0.014 & $\pm$0.011 & $\pm$0.007 & $\pm$  \textbf{0.004}\\\hline
\multirow{2}{*}{30\%}    & 0.713 & 0.787  & 0.786  & \textbf{0.787} & 0.843 & \textbf{0.901}  & 0.893  & 0.895 \\ 
& $\pm$0.002 & $\pm$0.007 & $\pm$0.001 & $\pm$\textbf{0.001} & $\pm$0.010 & $\pm$\textbf{0.002} & $\pm$0.005 & $\pm$0.007  \\\hline
\end{tabular}%
\end{small}
\caption{Testing performance on the four datasets by varying imbalanced ratios.}
\label{tab:performance}
\end{table*}

\begin{figure*}[t!]
\centering

\begin{minipage}{0.24\columnwidth}
  \centering
\includegraphics[width=\textwidth]{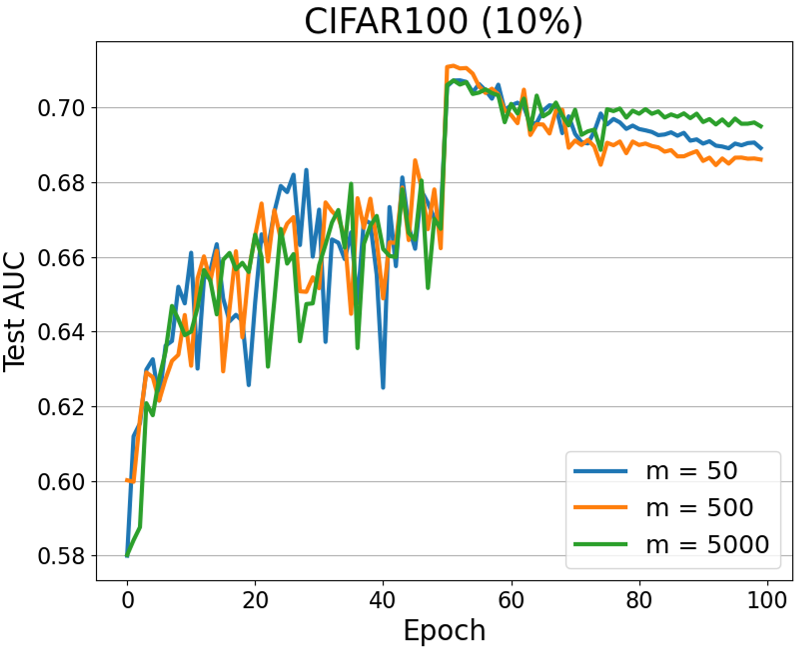}
\caption{Impact of $m$ for NSTORM.}
\label{fig:lr_m_cifar100}
\end{minipage}%
\hfill
\begin{minipage}{0.24\columnwidth}
  \centering
\includegraphics[width=\textwidth]{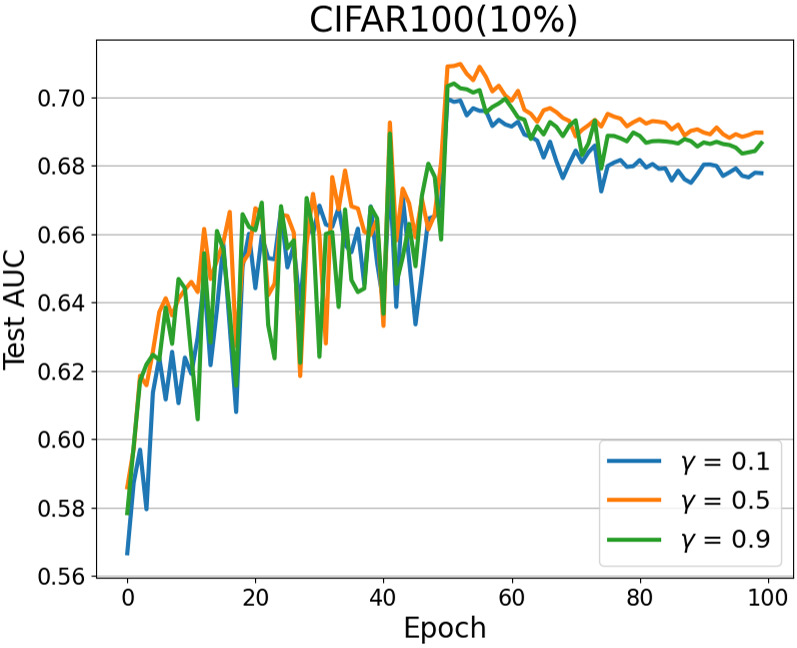}
\caption{Impact of $\gamma$ for NSTORM.}
\label{fig:lr_gamma_cifar100}
\end{minipage}%
\hfill
\begin{minipage}{0.24\columnwidth}
  \centering
\includegraphics[width=\textwidth]{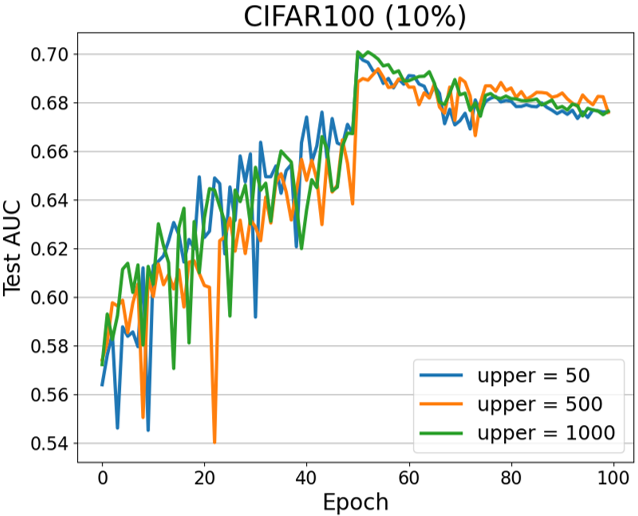}
\caption{Impact of upper bound on $a_t$ and $b_t$.}
\label{fig:upper_bound_cifar100}
\end{minipage}
\hfill
\begin{minipage}{0.24\columnwidth}
  \centering
\includegraphics[width=\textwidth]{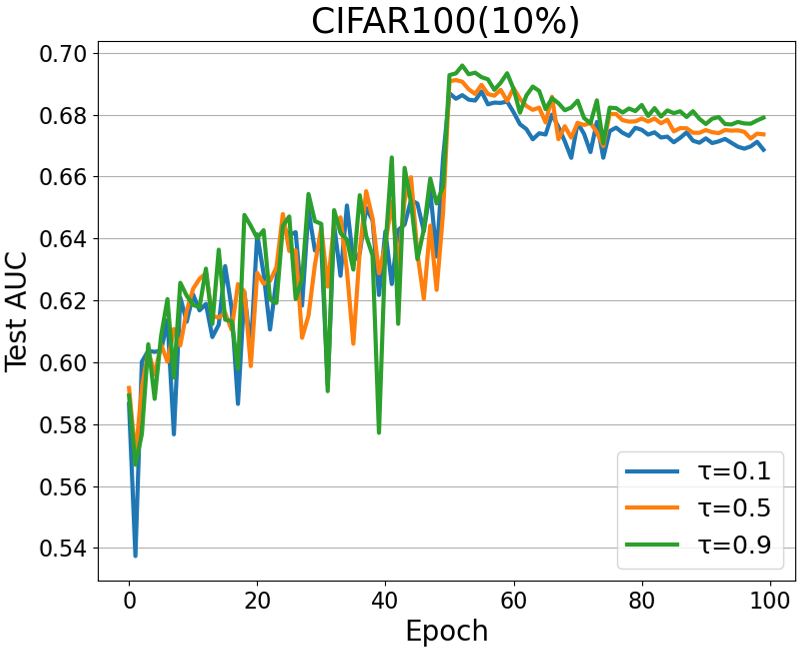}
\caption{Impact of $\tau$ for ADA-NSTORM.}
\label{fig:tau_cifar100}
\end{minipage}%
\end{figure*}

\subsection{Performance Evaluation}
The training progression of deep AUC is illustrated in Figure~\ref{fig:Convergence_10}. It shows the notable swiftness of convergence exhibited by our two proposed methods. Furthermore, across all four datasets, our methods consistently yield the most favorable test AUC outcomes. It is evident from the results depicted in Figure~\ref{fig:Convergence_10} that even NSTORM, which lacks an adaptive generator, surpasses the performance of SCGDA and PDSCA methods, thus reinforcing the validity of our theoretical analysis. Intriguingly, despite ADA-NSTORM sharing a theoretical foundation with NSTORM, it outperforms the testing AUC performance in the majority of scenarios.

The testing AUC outcomes are summarized in Table \ref{tab:performance}, with the optimal AUC values among 100 epochs. Combining these findings with Figure \ref{fig:Convergence_10}, a recurring pattern emerges: best testing AUC performance is typically attained around the 50th epoch, followed by overfitting. Both Table \ref{tab:performance} and Figure \ref{fig:Convergence_10} show that our proposed methods consistently outperform benchmarks. ADA-NSTORM achieves an impressive AUC of 0.833 on STL10 with a 10\% imbalanced ratio. Notable exceptions are the CAT\_vs\_DOG and CIFAR100 datasets with a 1\% imbalanced ratio, possibly due to their proximity to the training set's distribution.

\subsection{Ablation Study}
We conducted experiments to fine-tune parameters $m$ and $\gamma$ for NSTORM, as demonstrated in Figure~\ref{fig:lr_m_cifar100} and Figure~\ref{fig:lr_gamma_cifar100}. In Theorem~\ref{The:ConvergenceNonConvex}, we consider $m$ as the lower bound, controlling the learning rate $\eta_t$. Interestingly, adjusting $m$ yields minimal alterations. On the other hand, $\gamma$ determines the relative step sizes of $x$ and $y$. Figure~\ref{fig:lr_gamma_cifar100} reveals that an optimal $\gamma$ value of approximately 0.5 yields a testing AUC of 0.827.

To assess the influence of $a_t$ and $b_t$ in Assumption~\ref{ass ab bound}, we investigate the testing AUC under varying upper bounds for $a_t$ and $b_t$, as illustrated in Figure~\ref{fig:upper_bound_cifar100}. Notably, changing from an upper bound of 50 to 1000 yields a minimal change in the testing AUC, validating the mildness of Assumption~\ref{ass ab bound}. In addition, the parameter $\tau$ is related to the adaptive generator within ADA-NSTORM. Figure~\ref{fig:tau_cifar100} shows the impact of $\tau$ on ADA-NSTORM's performance within the deep AUC problem. Intriguingly, varying $\tau$ from 0.1 to 0.9 leads to a mere change of 0.127. These ablation studies effectively reinforce the robustness of our proposed methods.

\section{Conclusion}
In this paper, we first proposed a novel method named NSTORM for optimizing the compositional minimax problem. By leveraging variance-reduced techniques of both function and gradient values, we demonstrate that the proposed NSTORM method can achieve the sample complexity of $O(\kappa^3 /\epsilon^3)$ for finding an $\epsilon$-stationary point without using large batch sizes. NSTORM under the PL-condition is also demonstrated to achieve the same sample complexity, which indicates its extendability. To the best of our knowledge, all theoretical results match the best sample complexity in existing minimax optimization. Because NSTORM requires a small learning rate to achieve the optimal complexity, this limits its applicability in real-world machine learning scenarios. To take advantage of adaptive learning rates, we develop an adaptive version of NSTORM named ADA-STORM, which can achieve the same complexity with the learning rate changing adaptively. Extensive experimental results support the effectiveness of our proposed methods.

\bibliographystyle{unsrtnat}
\bibliography{template}  %%% Uncomment this line and comment out the ``thebibliography'' section below to use the external .bib file (using bibtex) .

\clearpage
\setcounter{page}{1}
\appendices
\section{Additional Toy Examples}
In Figure \ref{fig:all_toy_example}, the trajectories of the four benchmarks are shown beneath the variances of different levels. This visualization is based on the setting of the same number of iterations, identical $x$ and $y$ step sizes, as well as a common $\beta$ value. and we can find that the four benchmarks follow the same trajectory to reach the stationary point when the variance is equal to 0. As the variance increases, our method reaches the stationary point with a smoother and shorter path, which indicates that our method has better robustness to very noisy datasets.

\begin{figure*}[h]
    \centering
    \includegraphics[width=\textwidth]{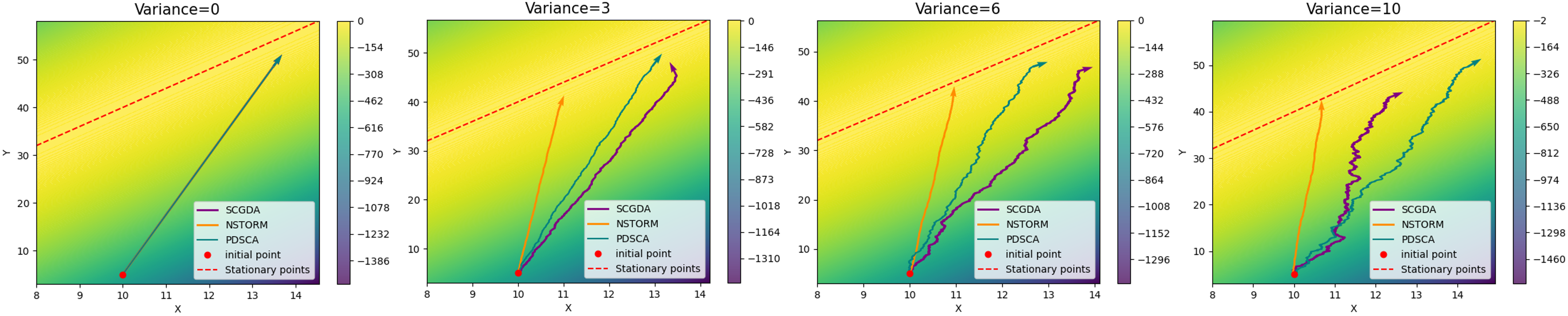}
    \caption{Trajectories of four benchmarks under varying variances. }
    \label{fig:all_toy_example}
\end{figure*}

\section{Useful Lemmas}
The following two lemmas aim to show the Lipschitz and smooth properties of the compositional minimization optimization, and they can facilitate the proof of all theorems. 
\begin{lemma}\label{lemma:y^*smooth}
Under the Assumptions~\ref{Ass:Smooth}, \ref{Ass:BoundedGrad} and \ref{Ass:StrongConcave}, $\forall x_1, x_2 \in \mathcal{X}$ and $\forall y \in \mathcal{Y}$, we can obtain the following:
\begin{equation}
\|y^*(x_1)- y^*(x_2)\| \leq C_g \kappa \|x_1- x_2\|.\nonumber
\end{equation}
\end{lemma}
\begin{proof}
 According to the definition of $y^*(x)$, we have $\nabla_yf(g(x_1),y^*(x_1)) = \nabla_yf(g(x_2),y^*(x_2)) = 0$, then we can get 
\begin{equation}
\begin{aligned}
    \|y^*(x_1)-y^*(x_2)\| &\leq \frac{1}{\mu}\| \nabla_y f(g(x_1),y^*(x_1)) - \nabla_y f(g(x_1),y^*(x_2))\| \\
    & = \frac{1}{\mu}\|\nabla_y f(g(x_2),y^*(x_2)) - \nabla_y f(g(x_1),y^*(x_2))\| \\
    & \leq \frac{L}{\mu}\| g(x_1)-g(x_2)\| \leq C_g \kappa  \|x_1-x_2\|,
\end{aligned}
\end{equation}
where the first inequality holds by $\mu$-strong concavity and the second inequality holds by $L$-Smoothness.
\end{proof}
\begin{lemma}\label{lemma:phismooth}
Under the Assumptions~\ref{Ass:Smooth}-\ref{Ass:StrongConcave} and Lemma~\ref{lemma:y^*smooth}, $\forall x_1, x_2 \in \mathcal{X}$, we can obtain the following:
\begin{equation}
\|\nabla \Phi(x_1) - \nabla \Phi(x_2)\| \leq (C_g^2 L\kappa+C_f L_g)\|x_1 - x_2\|.\nonumber
\end{equation}
\end{lemma}
\begin{proof}
\begin{equation}
\begin{aligned}
    \|\nabla \Phi(x_1)-\Phi(x_2)\|& = \|\nabla_xf(g(x_1),y^*(x_1))-\nabla_xf(g(x_2),y^*(x_2))\| \\
    & =  \|\nabla g(x_1)\nabla_g f(g(x_1),y^*(x_1))-\nabla g(x_2)\nabla_g f(g(x_2),y^*(x_2))\|\\
    & = \|\nabla g(x_1)\nabla_g f(g(x_1),y^*(x_1))- \nabla g(x_2)\nabla_g f(g(x_1),y^*(x_1)) \\
    &\quad+\nabla g(x_2)\nabla_g f(g(x_1),y^*(x_1)) - \nabla_xf(g(x_2),y^*(x_2))\|\\
    & \leq C_g \|\nabla_gf(g(x_1),y^*(x_1))-\nabla_gf(g(x_2),y^*(x_2))\|+C_f\|\nabla g(x_1)-\nabla g(x_2)\| \\
    & \leq C_fL_g\|x_1-x_2\|+C_gL\|g(x_1)-g(x_2)\|+C_gL\|y^*(x_1)-y^*(x_2)\|\\
    & \leq C_fL_g\|x_1-x_2\|+C_g^2L\|x_1-x_2\|+\frac{C_g^2L^2}{\mu}\|x_1-x_2\|\\
    & \leq (C_fL_g+\frac{2C_g^2L^2}{\mu})\|x_1-x_2\|,
\end{aligned}    
\end{equation}
where the last two inequality holds by Lemma \ref{lemma:y^*smooth} and the last inequality holds by $L / \mu >1$.
\end{proof}

\section{Proof of Theorem~\ref{The:ConvergenceNonConvex}}
We first provide some useful lemmas and then we show the sample complexity of the proposed NSTORM method in Theorem~\ref{The:ConvergenceNonConvex}.
\begin{lemma}\label{lemma_vt_variance}
Given Assumptions \ref{Ass:Smooth}-\ref{Ass:StrongConcave} for Algorithm \ref{ALG:SCGDA}, we can obtain the following:
\begin{equation}
\begin{aligned}
    \mathbb{E}[\|v_t-\nabla f(g(x_t),y_t)\|^2]\leq2C_{g}^{2}L^2\mathbb{E}[\|u_{t}-g(x_t)\|^2]+4C_{f}^{2}\mathbb{E}[\|v'_{t}-\nabla g(x_{t})\|^2]+4C_{g}^{2}\mathbb{E}[\|v''_{t}-\nabla_gf(u_{t},y_t)\|^2].\nonumber
\end{aligned}    
\end{equation}
\end{lemma}

\begin{proof}
\begin{equation}
\begin{aligned}
    &\mathbb{E}[\|v_{t}-\nabla f(g(x_{t}),y_{t})\|^2]\\
    &\leq 2\mathbb{E}[\|v'_{t}v''_{t}-\nabla g(x_{t})\nabla_gf(u_{t},y_t)\|^2]
    +2\mathbb{E}[\|\nabla g(x_{t})\nabla_gf(u_{t},y_t)-\nabla_xf(g(x_t),y_t)\|^2] \\
    &\leq 2\mathbb{E}[\|v'_{t}v''_{t}-\nabla g(x_{t})\nabla_gf(u_{t},y_t)\|^2]
    +2C_{g}^{2}L^2\mathbb{E}[\|u_{t}-g(x_t)\|^2] \\
    &\leq 2C_{g}^{2}L^2\mathbb{E}[\|u_{t}-g(x_t)\|^2]
    +4C_{f}^{2}\mathbb{E}[\|v'_{t}-\nabla g(x_{t})\|^2]
    +4C_{g}^{2}\mathbb{E}[\|v''_{t}-\nabla_gf(u_{t},y_t)\|^2],
\end{aligned}    
\end{equation}
where the second inequality follows from Assumptions ~\ref{Ass:Smooth} and \ref{Ass:BoundedGrad} and the last inequality follows from the following inequality:
\begin{equation}
\begin{aligned}
    \mathbb{E}[\|v'_{t}v''_{t}-\nabla g(x_{t})\nabla_gf(u_{t},y_t)\|^2] &=\mathbb{E}[\|\nabla g(x_{t})\nabla_gf(u_{t},y_t)-v'_{t}\nabla_gf(u_{t},y_t)
    +v'_{t}\nabla_gf(u_{t},y_t)-v'_{t}v''_{t}\|^2] \\
    &\leq2\mathbb{E}[\|\nabla g(x_{t})\nabla_gf(u_{t},y_t)
    -v'_{t}\nabla_gf(u_{t},y_t)\|^2]
    +2\mathbb{E}[\|v'_{t}\nabla_gf(u_{t},y_t)-v'_{t}v''_{t}\|^2]\\
    &\leq2C_{f}^{2}\mathbb{E}[\|v'_{t}-\nabla g(x_{t})\|^2]+2C_{g}^{2}\mathbb{E}[\|\nabla_g f(u_{t},y_t)-v''_{t}\|^2].
\end{aligned}    
\end{equation}
This completes the proof.
\end{proof}
It is worth noting that the estimated error of gradient in $x$ is determined by $\mathbb{E}[\|u_t-g(x_t)\|^2]$, $\mathbb{E}[\|v'_t-\nabla g(x_t)\|^2]$ and $\mathbb{E}[\|v''_t-\nabla_gf(u_t,y_t)\|^2]$. Therefore, as long as these three terms have smaller errors, the estimated error of gradient is smaller. The following lemmas aim to show how to control these errors.

\begin{lemma}\label{lemmau12}
Given Assumptions \ref{Ass:Smooth}-\ref{Ass:StrongConcave}, for Algorithm \ref{ALG:SCGDA}, we can obtain the following:
\begin{equation}
    \mathbb{E}[\|u_{t+1}-u_{t}\|^2]\leq2\beta_{t+1}^2\mathbb{E}[\|g(x_{t})-u_{t}\|^2]+2\beta_{t+1}^2\sigma_{g}^{2}+2L_{f}^{2}\gamma^2\eta_{t}^2\mathbb{E}[\|v_t\|^2].\nonumber
\end{equation}
\end{lemma}

\begin{proof}
\begin{equation}
\begin{aligned}
    \mathbb{E}[\|u_{t+1}-u_{t}\|^2]&=\mathbb{E}[\|(1-\beta_{t+1})(u_{t}-g(x_{t};\xi_{t+1}))+g(x_{t+1};\xi_{t+1})-u_{t}\|^2] \\
    &=\mathbb{E}[\|\beta_{t+1}(g(x_{t})-u_{t})+(g(x_{t+1};\xi_{t+1})-g(x_{t};\xi_{t+1}))+\beta_{t+1}(g(x_{t};\xi_{t+1})-g(x_{t}))\|^2]\\
    &\leq\mathbb{E}[\|\beta_{t+1}(g(x_{t})-u_{t})+\beta_{t+1}(g(x_{t};\xi_{t+1})-g(x_{t}))\|^2]+2L_{f}^{2}\mathbb{E}[\|x_{t+1}-x_{t}\|^2]\\
    &\leq2\beta_{t+1}^2\mathbb{E}[\|(g(x_{t})-u_{t})\|^2]+2\beta_{t+1}^2\sigma_{g}^{2}+2L_{f}^{2}\gamma^2\eta_{t}^{2}\mathbb{E}[\|v_t\|^2].
\end{aligned}      
\end{equation}
This completes the proof.
\end{proof}

\begin{lemma}\label{lemma:u_t_variance}
Given Assumptions \ref{Ass:Smooth}-\ref{Ass:StrongConcave}, for Algorithm \ref{ALG:SCGDA}, we can obtain the following:
\begin{equation}
\begin{aligned}
    \mathbb{E}[\|u_{t+1}-g(x_{t+1})\|^2]&\leq(1-\beta_{t+1})\mathbb{E}[\|u_t-g(x_t)\|^2]+2\beta_{t+1}^2\sigma_{g}^{2}+2L_{f}^{2}\gamma^2\eta_{t}^{2}\mathbb{E}[\|v_t\|^2].\nonumber
\end{aligned}    
\end{equation} 
\end{lemma}

\begin{proof}
\begin{equation}
\begin{aligned}
    \mathbb{E}[\|u_{t+1}-g(x_{t+1})\|^2]&=\mathbb{E}[\|(1-\beta_{t+1})(u_{t}-g(x_{t};\xi_{t+1}))+g(x_{t+1};\xi_{t+1})-g(x_{t+1})\|^2]\\
    &=\mathbb{E}[\|(1-\beta_{t+1})(u_{t}-g(x_{t}))+\beta_{t+1}(g(x_{t};\xi_{t+1})-g(x_{t}))\\
    &\quad+(g(x_{t})-g(x_{t+1})-(g(x_{t};\xi_{t+1})-g(x_{t+1};\xi_{t+1})))\|^2]\\
    &\leq(1-\beta_{t+1})\mathbb{E}[\|u_{t}-g(x_{t})\|^2]+2\beta_{t+1}^{2}\mathbb{E}[\|g(x_{t};\xi_{t+1}))-g(x_{t})\|^2]\\
    &\quad+2\mathbb{E}[\|g(x_{t+1};\xi_{t+1}))-g(x_{t};\xi_{t+1})\|^2]\\
    &\leq(1-\beta_{t+1})\mathbb{E}[\|u_{t}-g(x_{t})\|^2]+2\beta_{t+1}^2\sigma_{g}^{2}+2L_{f}^{2}\gamma^2\eta_{t}^{2}\mathbb{E}[\|v_t\|^2],
\end{aligned}     
\end{equation}
where the last two inequality is derived from $\beta_{t+1} < 1$.
\end{proof}
We have the flexibility to control the term $2\beta_{t+1}^2\sigma_g^2$ from the right side by selecting sufficiently small values for $\eta_t$, due to the definition of $\beta_{t+1}$. Similarly, by choosing a small enough step size $x$, i.e., $\gamma$, we can control the term $2L_f^2\gamma^2\eta_t^2$. It's worth noting that in \cite{gao2021convergence}, they adopt a strategy of using a large batch size to control the variance caused by the function value estimator and the gradient estimator of the inner function. However, this approach incurs significant computational costs.

\begin{lemma}\label{Lemma:v'_t_variance}
Given Assumptions \ref{Ass:Smooth}-\ref{Ass:StrongConcave}, for Algorithm \ref{ALG:SCGDA}, we can obtain the following:
\begin{equation}
    \mathbb{E}[\|v'_{t+1}-\nabla g(x_{t+1})\|^2] \leq (1-\beta_{t+1})\mathbb{E}[\|v'_t-\nabla g(x_t)\|^2]+2\beta_{t+1}^2\sigma_{g^{'}}^{2} + 2L_{g}^{2}\gamma^2\eta_{t}^{2}\mathbb{E}[\|v_t\|^2].\nonumber   
\end{equation}
\end{lemma}

\begin{proof}
\begin{equation}
\begin{aligned}
    \mathbb{E}[\|v'_{t+1}-\nabla g(x_{t+1})\|^2]&=\mathbb{E}[\|(1-\beta_{t+1})(u_{t}-\nabla g(x_{t};\xi_{t+1}))+\nabla g(x_{t+1};\xi_{t+1})-\nabla g(x_{t+1})\|^2]\\
    &=\mathbb{E}[\|(1-\beta_{t+1})(u_{t}-\nabla g(x_{t}))+\beta_{t+1}(\nabla g(x_{t};\xi_{t+1})-\nabla g(x_{t}))\\
    &\quad+(\nabla g(x_{t})-\nabla g(x_{t+1})-(\nabla g(x_{t};\xi_{t+1})-\nabla g(x_{t+1};\xi_{t+1})))\|^2]\\
    &\leq(1-\beta_{t+1})\mathbb{E}[\|u_{t}-\nabla g(x_{t})\|^2]+2\beta_{t+1}^{2}\mathbb{E}[\|\nabla g(x_{t};\xi_{t+1}))-\nabla g(x_{t})\|^2]\\
    &\quad+2\mathbb{E}[\|\nabla g(x_{t+1};\xi_{t+1}))-\nabla g(x_{t};\xi_{t+1})\|^2]\\
    &\leq(1-\beta_{t+1})\mathbb{E}[\|v_t-\nabla g(x_t)\|^2]+2\beta_{t+1}^2\sigma_{g^{'}}^{2}+2L_{g}^{2}\gamma^2\eta_{t}^{2}\mathbb{E}[\|v_t\|^2],
\end{aligned}
\end{equation}   
where the last inequality holds by Assumption \ref{Ass:BoundedVariance}.  
\end{proof}

\begin{lemma}\label{Lemma:v''_t-variance}
Given Assumptions \ref{Ass:Smooth}-\ref{Ass:StrongConcave}, for Algorithm \ref{ALG:SCGDA}, we can obtain the following:
\begin{equation}
\begin{aligned}
    \mathbb{E}[\|v''_{t+1}-\nabla_gf(u_{t+1},y_{t+1})\|^2]&\leq(1-\beta_{t+1})\mathbb{E}[\|v''_{t}-\nabla_gf(u_{t},y_t)\|^2]+2\beta_{t+1}^{2}\sigma_{f}^{2}+4L^2L_{f}^{2}\gamma^2\eta_{t}^{2}\mathbb{E}[\|v_t\|^2]\\
    &\quad+4L^2\beta_{t+1}^{2}\mathbb{E}[\|u_{t}-g(x_{t})\|^2]+4\beta_{t+1}^{2}L^2\sigma_g^2+2L^2\eta_{t}^{2}\mathbb{E}[\|w_t\|^2].\nonumber
\end{aligned}    
\end{equation} 
\end{lemma}

\begin{proof}
\begin{equation}
\begin{aligned}
    &\mathbb{E}[\|v''_{t+1}-\nabla_gf(u_{t+1},y_{t+1})\|^2]\\
    &=\mathbb{E}[\|(1-\beta_{t+1})(v''_{t}-\nabla_gf(u_{t},y_t;\zeta_{t+1}))
    +\nabla_gf(u_{t+1},y_{t+1};\zeta_{t+1})-\nabla_gf(u_{t+1},y_{t+1})\|^2]\\
    &=\mathbb{E}[\|(1-\beta_{t+1})(v''_{t}-\nabla_gf(u_{t},y_t))+\beta_{t+1}(\nabla_gf(u_{t},y_t;\zeta_{t+1})-\nabla_gf(u_{t},y_t))\\
    &\quad+\nabla_gf(u_{t},y_t)-\nabla_gf(u_{t+1},y_{t+1})-(\nabla_gf(u_{t},y_t;\zeta_{t+1})-\nabla_gf(u_{t+1},y_{t+1};\zeta_{t+1}))\|^2]\\
    &\leq(1-\beta_{t+1})\mathbb{E}[\|v''_{t}-\nabla_gf(u_{t},y_t)\|^2]+2L^2(\mathbb{E}[\|y_{t+1}-y_t\|^2]+\mathbb{E}[\|u_{t+1}-u_{t}\|^2])+2\beta_{t+1}^{2}\sigma_{f}^{2}\\
    &\leq(1-\beta_{t+1})\mathbb{E}[\|v''_{t}-\nabla_gf(u_{t},y_t)\|^2]+2\beta_{t+1}^{2}(\sigma_{f}^{2}+2L^2\sigma_g^2)+2L^2\eta_{t}^{2}\mathbb{E}[\|w_t\|^2]\\
    &\quad+4L^2L_{f}^{2}\gamma^2\eta_{t}^{2}\mathbb{E}[\|v_t\|^2]+4L^2\beta_{t+1}^{2}\mathbb{E}[\|u_{t}-g(x_{t})\|^2],
\end{aligned}    
\end{equation}
where the first inequality holds by the smoothness and the last inequality holds by Lemma~\ref{lemmau12}.
\end{proof}
The gap between the gradient estimator of the outer function and its true value is greater than the gap for the inner function estimator. This is due to two reasons. First, during iterations, updates to $x$ introduce errors when $y$ is subsequently updated. Second, multi-layer estimators for combined functions propagate more errors through the layers. However, the accumulation of error is bounded.
\begin{lemma}\label{lemma:w_t-variance}
Given Assumptions \ref{Ass:Smooth}-\ref{Ass:StrongConcave}, for Algorithm \ref{ALG:SCGDA}, based on Lemma \ref{lemmau12}, we can obtain the following:
\begin{equation}
\begin{aligned}
    &\mathbb{E}[\|w_{t+1}-\nabla_y f(g(x_{t+1}),y_{t+1})\|^2]\\
    &\leq(1-\alpha_{t+1})\mathbb{E}[\|w_{t}-\nabla_yf(g(x_{t}),y_{t})\|^2]+4\alpha_{t+1}^{2}\sigma_{f}^{2}+8L^{2}\eta_{t}^{2}\mathbb{E}[\|w_t\|^2]\\
    &\quad+(4L^2\alpha_{t+1}^2+8L^2\beta_{t+1}^2)\mathbb{E}[\|u_{t}-g(x_{t})\|^2]+8L^2\beta_{t+1}^{2}\sigma_{g}^{2}+(8L_{f}^{2}L^2\gamma^2\eta_{t}^{2}+4L^2C_g^2\gamma^2\eta_t^2)\mathbb{E}[\|v_t\|^2].\nonumber
\end{aligned}    
\end{equation}
\end{lemma}

\begin{proof}
\begin{equation}
\begin{aligned}
    &\mathbb{E}[\|w_{t+1}-\nabla_yf(g(x_{t+1}),y_{t+1})\|^2]\\
    & = \mathbb{E}[\|(1-\alpha_{t+1})(w_t-\nabla_yf(u_t,y_t;\zeta_{t+1}))+\nabla_yf(u_{t+1},y_{t+1};\zeta_{t+1})-\nabla_yf(g(x_{t+1}),y_{t+1})\|^2] \\
    & = \mathbb{E}[\|(1-\alpha_{t+1})(w_t-\nabla_yf(g(x_t),y_t))+\alpha_{t+1}(\nabla_yf(u_t,y_t;\zeta_{t+1})-\nabla_yf(g(x_t),y_t))\\
    &\quad+\nabla_yf(g(x_t),y_t)-\nabla_yf(g(x_{t+1}),y_{t+1})-(\nabla_yf(u_t,y_t;\zeta_{t+1})-\nabla_yf(u_{t+1},y_{t+1};\zeta_{t+1}))]\\
    & \leq (1-\alpha_{t+1})\mathbb{E}[\|w_t-\nabla_yf(g(x_t),y_t)\|^2]+2\alpha_{t+1}^2\underbrace{\mathbb{E}[\|\nabla_yf(u_t,y_t;\zeta_{t+1})-\nabla_yf(g(x_t),y_t)\|^2]}\limits_{Q_1}\\
    &\quad+4L^2C_g^2\gamma^2\eta_t^2\mathbb{E}[\|v_t\|^2]+8L^2\eta_t^2\mathbb{E}[\|w_t\|^2]+4L^2\mathbb{E}[\|u_{t+1}-u_t\|^2].
\end{aligned}
\end{equation}
Next, we bound the term $Q_1$:
\begin{equation}
\begin{aligned}
    &\mathbb{E}[\|\nabla_yf(u_t,y_t;\zeta_{t+1})-\nabla_yf(g(x_t),y_t)\|^2]\\
    & =\mathbb{E}[\|\nabla_yf(u_t,y_t;\zeta_{t+1})-\nabla_yf(u_t,y_t)+\nabla_yf(u_t,y_t)-\nabla_yf(g(x_t),y_t)\|^2] \leq 2\sigma_f^2+2L^2\mathbb{E}[\|u_t-g(x_t)\|^2].
\end{aligned}
\end{equation}
Then, according to Lemma \ref{lemmau12}, we can conclude that:
\begin{equation}
\begin{aligned}
    &\mathbb{E}[\|w_{t+1}-\nabla_yf(g(x_{t+1}),y_{t+1})\|^2]\\
    &\leq(1-\alpha_{t+1})\mathbb{E}[\|w_{t}-\nabla_yf(g(x_{t}),y_{t})\|^2]+4\alpha_{t+1}^{2}\sigma_{f}^{2}+8L^{2}\eta_{t}^{2}\mathbb{E}[\|w_t\|^2]\\
    &\quad+(4L^2\alpha_{t+1}^2+8L^2\beta_{t+1}^2)\mathbb{E}[\|u_{t}-g(x_{t})\|^2]+8L^2\beta_{t+1}^{2}\sigma_{g}^{2} +(8L_{f}^{2}L^2\gamma^2\eta_{t}^{2}+4L^2C_g^2\gamma^2\eta_t^2)\mathbb{E}[\|v_t\|^2].
\end{aligned}
\end{equation}
This completes the proof.
\end{proof}

According to the proof of Lemma \ref{lemma:w_t-variance}, we can find that the variance is produced by more terms, such as $\mathbb{E}[\|u_{t+1} - u_t\|^2]$. But based on Lemma \ref{lemmau12}, the error of growth is also very limited. So, taking the same strategy, i.e., small $\beta_{t+1}$ and $\gamma$, we can control the variance.
\begin{lemma}\label{Lemma:phi}
Given Assumptions \ref{Ass:Smooth}-\ref{Ass:StrongConcave} for Algorithm \ref{ALG:SCGDA}, based on Lemma \ref{lemma_vt_variance}, by setting $\eta_t \leq \frac{1}{2L_{\Phi}}$, we can obtain the following:
\begin{equation}
\begin{aligned}
\Phi(x_{t+1})&\leq\Phi(x_t)-\frac{\gamma\eta_t}{2}\|\Phi(x_t)\|^2-\frac{\gamma\eta_t}{4}\|v_t\|^2+\gamma\eta_tC_{g}^{2}L^2\|y^*(x_t)-y_t\|^2 \\
&\quad+2C_{g}^{2}L^2\gamma\eta_t\|u_{t}-g(x_t)\|^2+4C_{f}^{2}\gamma\eta_t\|u_{t}-\nabla g(x_{t})\|^2 + 4C_{g}^{2}\gamma\eta_t\|v''_{t}-\nabla_gf(u_{t},y_t)\|^2.\nonumber
\end{aligned}    
\end{equation}
\end{lemma}

\begin{proof}
Since $\Phi(x_t)$ is $L_{\Phi}$-smooth, by setting $\eta_t\leq\frac{1}{2L_{\Phi}}$ , we have:
\begin{equation}
    \begin{aligned}
        \Phi(x_{t+1})&\leq\Phi(x_{t})+\langle\nabla \Phi(x_{t}),x_{t+1}-x_t\rangle+\frac{L_{\Phi}}{2}\|x_{t+1}-x_t\|^2 \\
        &\leq\Phi(x_{t})-\gamma\eta_t\langle\nabla\Phi(x_{t}),v_t\rangle+\frac{L_{\Phi}\gamma^2\eta_{t}^2}{2}\|v_t\|^2\\
        &= \Phi(x_{t})+\frac{\gamma\eta_t}{2}\|\nabla\Phi(x_{t})-v_t\|^2-\frac{\gamma\eta_t}{2}\|\nabla\Phi(x_{t})\|^2+(\frac{\gamma^2\eta_tL_\Phi}{2}-\frac{\gamma\eta_t}{2})\|v_t\|^2 \\
        &\leq\Phi(x_{t})+\frac{\gamma\eta_t}{2}\|\nabla\Phi(x_{t})-v_t\|^2-\frac{\gamma\eta_t}{2}\|\nabla\Phi(x_{t})\|^2-\frac{\gamma\eta_t}{4}\|v_t\|^2 \\
        &\leq\Phi(x_{t})-\frac{\gamma\eta_t}{2}\|\nabla\Phi(x_{t})\|^2-\frac{\gamma\eta_t}{4}\|v_t\|^2+\gamma\eta_t\|\nabla\Phi(x_{t})-\nabla_xf(g(x_t),y_t)\|^2\\
        &\quad+\gamma\eta_t\|\nabla_xf(g(x_t),y_t)-v_t\|^2.
    \end{aligned}    
\end{equation}
Besides, according to the definition of $\Phi(x_t)$ and Assumptions~\ref{Ass:Smooth}-\ref{Ass:BoundedGrad} we also have:
\begin{equation}
\begin{aligned}
    &\|\nabla\Phi(x_{t})-\nabla_xf(g(x_t),y_t)\|^2\\
    &=\|\nabla_xf(g(x_t),y^*(x_t))-\nabla_xf(g(x_t),y(x_t))\|^2 \\
    &=\|\nabla g(x_t)\nabla_gf(g(x_t),y^*(x_t))-\nabla g(x_t)\nabla_gf(g(x_t),y(x_t))\|^2\leq C_{g}^{2}L^2\|y^{*}(x_t)-y_t\|^2.
\end{aligned}    
\end{equation}
Based on Lemma~\ref{lemma_vt_variance}, we can conclude that:
\begin{equation}
\begin{aligned}
    \Phi(x_{t+1})&\leq\Phi(x_t)-\frac{\gamma\eta_t}{2}\|\nabla\Phi(x_t)\|^2-\frac{\gamma\eta_t}{4}\|v_t\|^2+\gamma\eta_tC_{g}^{2}L^2\|y^*(x_t)-y_t\|^2 \\
    &\quad+2C_{g}^{2}L^2\gamma\eta_t\|u_{t}-g(x_t)\|^2+4C_{f}^{2}\gamma\eta_t\|u_{t}-\nabla g(x_{t})\|^2+4C_{g}^{2}\gamma\eta_t\|v''_{t}-\nabla_gf(u_{t},y_t)\|^2.
\end{aligned}    
\end{equation}
This completes the proof.
\end{proof}

\begin{lemma}
Given Assumptions \ref{Ass:Smooth}-\ref{Ass:BoundedVariance} for Algorithm \ref{ALG:SCGDA}, by setting $\eta_t \leq \frac{1}{5L}$,  we can obtain the following:
\begin{equation}
\begin{aligned}
      \mathbb{E}[\|y^*(x_{t+1})-y_{t+1}\|^2]&\leq(1-\frac{\mu\eta_t}{4})\mathbb{E}[\|y^*(x_{t})-y_{t}\|^2]+\frac{5\kappa^3C_g^2\gamma^2\eta_t}{L}\mathbb{E}[\|v_t\|^2]\\
      &\quad+\frac{9\eta_t}{\mu}\mathbb{E}[\|w_{t}-\nabla_yf(g(x_t),y_{t})\|^2]-\frac{3\eta_t}{5\mu}\mathbb{E}[\|w_{t}\|^2].\nonumber
\end{aligned}        
\end{equation}
\end{lemma}

\begin{proof}
Similar to \cite{xian2021faster}, define $z_{t} = y_{t} +\theta w_{t}$ for some constant $\theta$. Due to the strongly concave of $f$ in $y$, we have:
\begin{equation}\label{eq:strongly-concave-y}
\begin{aligned}
    f(g(x_t),y^*(x_t))&\leq f(g(x_t),y_t)+\langle\nabla_yf(g(x_t),y_t),y^*(x_t)-y_t\rangle-\frac{\mu}{2}\|y^*(x_t)-y_t\|^2\\
    &=f(g(x_t),y_t)+\langle w_t,y^*(x_t)-z_t\rangle+\langle \nabla_yf(g(x_t),y_t)-w_t,y^*(x_t)-z_t\rangle\\
    &\quad+\theta\langle \nabla_yf(g(x_t),y_t),w_t\rangle-\frac{\mu}{2}\|y^*(x_t)-y_t\|^2.
\end{aligned}
\end{equation}
Besides, as function $f$ is smooth in $y$, we have:
\begin{equation}\label{eq:smooth-y}
    -\frac{L\theta^2}{2}\|w_t\|^2\leq f(g(x_t),z_t)-f(g(x_t),y_t)-\theta\langle \nabla_yf(g(x_t),y_t),w_t\rangle.
\end{equation}
Due to the definition of $y^*(x_t)$, we have $f(g(x_t),y^*(x_t)) \geq f(g(x_t),y_t)$. combine \eqref{eq:strongly-concave-y} and \eqref{eq:smooth-y} we have:
\begin{equation}\label{eq combine concavity and smooth}
\begin{aligned}
    0&\leq\langle w_t,y^*(x_t)-z_t\rangle+\langle \nabla_yf(g(x_t),y_t)-w_t,y^*(x_t)-z_t\rangle-\frac{\mu}{2}\|y^*(x_t)-y_t\|^2-\frac{L\theta^2}{2}\|w_t\|^2\\
    &=\langle w_t,y^*(x_t)-y_t\rangle-\frac{\mu}{2}\|\|y^*(x_t)-y_t\|^2+\langle \nabla_yf(g(x_t),y_t)-w_t,y^*(x_t)-z_t\rangle-(\theta-\frac{L\theta^2}{2})\|w_t\|^2.
\end{aligned}
\end{equation}
By Cauchy-Schwartz inequality we have:
\begin{equation}\label{eq C-S ineq}
\begin{aligned}
\langle \nabla_yf(g(x_t),y_t)-w_t,y^*(x_t)-z_t\rangle&\leq\frac{4}{\mu}\|\nabla_yf(g(x_t),y_t)-w_t\|^2+\frac{\mu}{8}\|y^*(x_t)-y_t\|^2+\frac{\mu\theta^2}{8}\|w_t\|^2.    
\end{aligned}
\end{equation}
Adding \eqref{eq C-S ineq} and \eqref{eq combine concavity and smooth}, and setting $\theta = \frac{4}{5\mu}$, we obtain:
\begin{equation}\label{eq add c_s and combine concavity and smooth}
    0\leq\langle w_t,y^*(x_t)-y_t\rangle-\frac{\mu}{4}\|y^*(x_t)-y_t\|^2+\frac{4}{\mu}\|\nabla_yf(g(x_t),y_t)-w_t\|^2-\frac{2}{5\mu}\|w_t\|^2.
\end{equation}
As we have:
\begin{equation}\label{eq:w_t and max_y_t - y_t}
    2\eta_t\langle w_t,y^*(x_t)\rangle = \|y_t-y^*(x_t)\|^2+\|y_{t+1}-y_t\|^2-\|y_{t+1}-y^*(x_t)\|^2.
\end{equation}
Summing \eqref{eq add c_s and combine concavity and smooth} and \eqref{eq:w_t and max_y_t - y_t}, and rearranging terms, we obtain:
\begin{equation}\label{eq y_{t+1}-y*(x_{t+1})}
\begin{aligned}
    \|y_{t+1}-y^*(x_t)\|^2&\leq(1-\frac{\mu\eta_t}{2})\|y_t-y^*(x_t)\|^2+\|y_{t+1}-y_t\|^2+\frac{8\eta_t}{\mu}\|\nabla_yf(g(x_t),y_t)-w_t\|^2-\frac{4\eta_t}{5\mu}\|w_t\|^2.
\end{aligned}    
\end{equation}
According to Young's inequality:
\begin{equation}\label{}
\begin{aligned}
    &\|y_{t+1}-y^*(x_{t+1})\|^2\\
    &\leq(1+\frac{\mu\eta_t}{4})\|y_{t+1}-y^*(x_t)\|^2+(1+\frac{4}{\mu\eta_t})\|y^*(x_t)-y^*(x_{t+1})\|^2\\
    &\leq(1-\frac{\mu\eta_t}{4})\|y_{t}-y^*(x_t)\|^2+\frac{9\eta_t}{\mu}\|\nabla_yf(g(x_t),y_t)-w_t\|^2+\frac{5\kappa}{L\eta_t}\|y^*(x_{t+1})-y^*(x_{t})\|^2\|-\frac{3\eta_t}{5\mu}\|w_t\|^2\\
    &\leq(1-\frac{\mu\eta_t}{4})\|y_{t}-y^*(x_t)\|^2+\frac{9\eta_t}{\mu}\|\nabla_yf(g(x_t),y_t)-w_t\|^2+\frac{5\kappa^3C_g^2\gamma^2\eta_t}{L}\|v_t\|^2-\frac{3\eta_t}{5\mu}\|w_t\|^2,
\end{aligned}    
\end{equation}
where the second inequality holds by \eqref{eq y_{t+1}-y*(x_{t+1})} and $\eta_t \leq \frac{1}{5L}$.
\end{proof}

Based on the above lemmas, we can approach the proof of Theorem \ref{The:ConvergenceNonConvex}. Define the potential function, for any $t\geq 1$:
\begin{equation}
    P_t=\mathbb{E}[\Phi(x_t)]+A\mathcal{T}_t+\frac{1}{\eta_{t-1}}(\mathcal{J}_t+\mathcal{H}_t+\mathcal{X}_t+\mathcal{K}_t),
\end{equation}
where $A=20L^2\mu+\frac{4\gamma C_g^2L^2}{\mu}$. We denote that $\mathcal{J}_t: =\mathbb{E}[\|u_{t}-g(x_{t})\|^2]$, $\mathcal{H}_t: =\mathbb{E}[\|v'_{t}-\nabla g(x_{t})\|^2]$, $\mathcal{X}_t: =\mathbb{E}[\|w_{t}-\nabla_yf(g(x_{t}),y_{t})\|^2]$, $\mathcal{K}_t: =\mathbb{E}[\|v''_{t}-\nabla_gf(u_{t},y_{t})\|^2]$ and $\mathcal{T}_t : =\mathbb{E}[\|y^*(x_{t})-y_{t}\|^2]$. Then based on the above lemmas, we have:
\begin{equation}
\begin{aligned}
    P_{t+1}-P_t &= \mathbb{E}[\Phi(x_{t+1})]-\mathbb{E}[\Phi(x_t)]+A(\mathcal{T}_{t+1}-\mathcal{T}_t)\\
    &\quad+\frac{1}{\eta_{t}}(\mathcal{J}_{t+1}+\mathcal{H}_{t+1}+\mathcal{X}_{t+1}+\mathcal{K}_{t+1})-\frac{1}{\eta_{t-1}}(\mathcal{J}_t+\mathcal{H}_t+\mathcal{X}_t+\mathcal{K}_t)\\
    &\leq-\frac{\gamma\eta_t}{2}\mathbb{E}[\|\nabla\Phi(x_t)\|^2]-\frac{\gamma\eta_t}{4}\mathbb{E}[\|v_t\|^2]+C_g^2L^2\gamma\eta_t\mathcal{T}_t+2C_g^2L^2\gamma\eta_t\mathcal{J}_t\\
    &\quad+4C_f^2\gamma\eta_t\mathcal{H}_t+4C_g^2\gamma \eta_t \mathcal{K}_t + A (\mathcal{T}_{t+1} - \mathcal{T}_t)\\
    &\quad+\frac{1}{\eta_{t}}(\mathcal{J}_{t+1}+\mathcal{H}_{t+1}+\mathcal{X}_{t+1}+\mathcal{K}_{t+1})-\frac{1}{\eta_{t-1}}(\mathcal{J}_t+\mathcal{H}_t+\mathcal{X}_t+\mathcal{K}_t)\\
    &\leq-\frac{\gamma\eta_t}{2}\mathbb{E}[\|\nabla\Phi(x_{t})\|^2]+(2C_g^2L^2\gamma\eta_t+\frac{1-\beta_{t+1}}{\eta_t}+\frac{8L^2\beta^{2}_{t+1}}{\eta_t}+\frac{4L^2\alpha_{t+1^2}}{\eta_t}-\frac{1}{\eta_{t-1}})\mathcal{J}_t\\
    &\quad+(4C_f^2\gamma\eta_t+\frac{1-\beta_{t+1}}{\eta_t}-\frac{1}{\eta_{t-1}})\mathcal{H}_t+(4C_g^2\gamma\eta_t+\frac{1-\beta_{t+1}}{\eta_t}-\frac{1}{\eta_{t-1}})\mathcal{K}_t\\
    &\quad+(\frac{9A\eta_t}{\mu}+\frac{1-\alpha_{t+1}}{\eta_t}-\frac{1}{\eta_{t-1}})\mathcal{X}_t+\frac{\beta^2_{t+1}}{\eta_t}(2\sigma_g^2+2\sigma^{2}_{g^{'}} + 2\sigma_f^2 + 4L^2\sigma_g^2)+\frac{4\sigma_f^2\alpha_{t+1}^2}{\eta_t}\\
    &\quad+(\frac{5\kappa^3C_g^2\gamma^2A\eta_t}{L}+(2L_f^2 + 2L_g^2 + 12L^2L_f^2+4L^2C_g^2)\gamma^2\eta_t-\frac{\gamma\eta_t}{4})\mathbb{E}[\|v_t\|^2].
\end{aligned}    
\end{equation}
By setting  $\eta_t=\frac{1}{(t+m)^{1/3}}$, $m> \operatorname{max}\{125L^3, 8\gamma^3L_\Phi^3, (12L^2c_1^2+4L^2c_2^2)^3, c_1^3, c_2^3\}$,  $\beta_{t+1}=c_1\eta_t^2\leq c_1\eta_t < 1$, $\alpha_{t+1}=c_2\eta_t^2\leq c_2\eta_t < 1 $ and $\eta_t=\frac{1}{(t+m)^{1/3}}$ we have:
\begin{equation}\label{Eq:eta_t}
\begin{aligned}
    \frac{1}{\eta_t}-\frac{1}{\eta_{t-1}}=(m+t)^{\frac{1}{3}}-(m+t-1)^{\frac{1}{3}}&\leq\frac{1}{3(m+t-1)^{2/3}}=\frac{2^{2/3}}{3(2(m+t-1))^{2/3}}\\
    &\leq\frac{2^{2/3}}{3(m+t)^{2/3}}=\frac{2^{2/3}}{3}\eta_t^2\leq\frac{2}{3}\eta_t^2,
\end{aligned}    
\end{equation}
where the first inequality holds by $(x+y)^{1/3}-x^{1/3}\leq yx^{-2/3}$.
\\
Let $c_1\geq 2+4\gamma  (C_f^2+C_g^2)+2C_g^2L^2\gamma$, we have:
\begin{equation}
\begin{aligned}
    4C_f^2\gamma\eta_t+\frac{1-\beta_{t+1}}{\eta_t}-\frac{1}{\eta_{t-1}}&\leq 4C_f^2\gamma\eta_t + \frac{2}{3}\eta_t^2 -c_1\eta_t\leq 4C_f^2\gamma\ -4\gamma  (C_f^2+C_g^2)-2C_g^2L^2\gamma\leq 0,
\end{aligned}    
\end{equation}
and we also have:
\begin{equation}
\begin{aligned}
    4C_g^2\gamma\eta_t+\frac{1-\beta_{t+1}}{\eta_t}-\frac{1}{\eta_{t-1}}&\leq 4C_g^2\gamma\eta_t + \frac{2}{3}\eta_t^2 -c_1\eta_t\leq 4C_g^2\gamma\ -4\gamma  (C_f^2+C_g^2)-2C_g^2L^2\gamma\leq 0.
\end{aligned}    
\end{equation}
Let $c_2\geq \frac{2}{3}+180L^2+\frac{36\gamma C_g^2L^2}{\mu^2}$, we have:
\begin{equation}
\begin{aligned}
    \frac{9A\eta_t}{\mu}+\frac{1-\alpha_{t+1}}{\eta_t}-\frac{1}{\eta_{t-1}}&\leq\frac{9A}{\mu}+\frac{2}{3} -c_2\leq 0.
\end{aligned}    
\end{equation}
By setting $B=\frac{100C_g^2L^4}{\mu^2}+2L_f^2+2L_g^2+12L^2L^2_f+4L^2C_g^2$ and $0 < \gamma \leq\frac{1}{4\sqrt{B^2+20\kappa^4C_g^2}}$, we can get:
\begin{equation}
\begin{aligned}
     P_{t+1}-P_t& \leq-\frac{\gamma\eta_t}{2}\mathbb{E}[\|\nabla\Phi(x_t)\|^2]+(2C_g^2L^2\gamma\eta_t+\frac{2}{3}\eta_t-c_1\eta_t+12L^2c_1^2\eta_t^2+4L^2c_2^2\eta_t^2)\mathcal{J}_t\\
     &\quad+(\frac{5\kappa^3\gamma^2A\eta_t}{L}+(2L_f^2+2L_g^2+12L^2L_f^2+4L^2C_g^2)\gamma^2\eta_t-\frac{\gamma\eta_t}{4})\mathbb{E}[\|v_t\|^2]\\
     &\quad+c_1^2\eta_t^3(2\sigma_g^2+2\sigma^{2}_{g^{'}} + 2\sigma_f^2 + 4L^2\sigma_g^2)+4\sigma_f^2c_2^2\eta_t^3\\
     &\leq -\frac{\gamma\eta_t}{2}\mathbb{E}[\|\nabla\Phi(x_t)\|^2]+(2C_g^2L^2\gamma+\frac{2}{3}-c_1+\frac{12L^2c_1^2+4L^2c_2^2}{m^{1/3}})\eta_t\mathcal{J}_t\\
     &\quad+\eta_t(20\kappa^4\gamma^3C_g^2+50\kappa^3\gamma^2 L\mu+(2L_f^2+2L_g^2+8L^2L^2_f)\gamma^2-\frac{\gamma}{4})\mathbb{E}[\|v_t\|^2]\\
     &\quad+c_1^2\eta_t^3(2\sigma_g^2+2\sigma^{2}_{g^{'}} + 2\sigma_f^2 + 4L^2\sigma_g^2)+4\sigma_f^2c_2^2\eta_t^3\\
     &\leq -\frac{\gamma\eta_t}{2}\mathbb{E}[\|\nabla\Phi(x_t)\|^2]+c_1^2\eta_t^3(2\sigma_g^2+2\sigma^{2}_{g^{'}} + 2\sigma_f^2 + 4L^2\sigma_g^2)+4\sigma_f^2c_2^2\eta_t^3\\
     &\quad+\gamma\eta_t(\frac{20\kappa^4C_g^2}{16(B^2+20\kappa^4C_g^2)}+\frac{B}{4\sqrt{B^2+20\kappa^4C_g^2}}-\frac{1}{4})\mathbb{E}[\|v_t\|^2]\\
     & = -\frac{\gamma\eta_t}{2}\mathbb{E}[\|\nabla\Phi(x_t)\|^2]+c_1^2\eta_t^3(2\sigma_g^2+2\sigma^{2}_{g^{'}} + 2\sigma_f^2 + 4L^2\sigma_g^2)+4\sigma_f^2c_2^2\eta_t^3\\
     &\quad+\gamma\eta_t(\frac{20\kappa^4C_g^2/B^2+4\sqrt{1+20\kappa^4 C_g^2 / B^2}-4(1+20\kappa^4 C_g^2 /B^2)}{16(1+20\kappa^4C_g^2 / B^2)})\mathbb{E}[\|v_t\|^2]\\
     &\leq-\frac{\gamma\eta_t}{2}\mathbb{E}[\|\nabla\Phi(x_t)\|^2]+c_1^2\eta_t^3(2\sigma_g^2+2\sigma^{2}_{g^{'}} + 2\sigma_f^2 + 4L^2\sigma_g^2)+4\sigma_f^2c_2^2\eta_t^3,
\end{aligned}    
\end{equation}
where the last inequality holds by $4\sqrt{1+x}-4-3x < 0,\forall x >0 $.
\\
Setting $L_1 =c_1^2(2\sigma_g^2+2\sigma^{2}_{g^{'}} + 2\sigma_f^2 + 4L^2\sigma_g^2)+4\sigma_f^2c_2^2$, then by summing up and rearranging, we have:
\begin{equation}
\begin{aligned}
    \mathbb{E}\bigg[\sum^{T}\limits_{t=1}\frac{\gamma\eta_t}{2}\|\nabla\Phi(x_t)\|^2\bigg]&\leq(P_1-P_{T+1})+L_1\sum^{T}\limits_{t=1}\eta_t^3\leq(P_1-P_{T+1})+L_1 \ln(T+1).
\end{aligned}    
\end{equation}
From the initialization condition, it is easy to get:
\begin{equation}
\begin{aligned}
     \mathbb{E}\bigg[\sum^{T}\limits_{t=1}\frac{\gamma\eta_t}{2}\|\nabla\Phi(x_t)\|^2\bigg]&\leq(\Phi(x_1)-\Phi_*)+\sigma_g^2+\sigma_{g^{'}}^2+\sigma_f^2 +L^2\sigma_g^2+L_1 \ln(T+1),
\end{aligned}    
\end{equation}
where $\sigma_g^2+\sigma_{g^{'}}^2+\sigma_f^2 +L^2\sigma_g^2 $ is  is the variance produced by the first iteration.
Since $\eta_t$ is  decreasing, we have:
\begin{equation}
\begin{aligned}
     \mathbb{E}\bigg[\sum^{T}\limits_{t=1}\frac{\gamma\eta_T}{2}\|\nabla\Phi(x_t)\|^2\bigg]&\leq(\Phi(x_1)-\Phi_*)\ +\sigma_g^2+\sigma_{g^{'}}^2+\sigma_f^2 +L^2\sigma_g^2 +L_1 \ln(T+1).
\end{aligned}    
\end{equation}
Similar to the proof of Theorem 1 in STORM \cite{cutkosky2019momentum}, denoting that $M=(\Phi(x_1)-\Phi_*) +\sigma_g^2+\sigma_{g^{'}}^2+\sigma_f^2 +L^2\sigma_g^2 +L_1 \ln(T+1)$, we have:
\begin{equation}
\begin{aligned}
\mathbb{E}\Bigg[\sqrt{\sum^{T}\limits_{t=1}\|\nabla\Phi(x_t)\|^2}\Bigg]^2&\leq\mathbb{E}[\frac{1}{\gamma \eta_T}]\mathbb{E}\bigg[\sum^{T}\limits_{t=1}\frac{\gamma\eta_T}{2}\|\nabla\Phi(x_t)\|^2\bigg]\leq\mathbb{E}[\frac{M}{\gamma \eta_T}] =\mathbb{E}[\frac{M(m+3)^{1/3}}{\gamma}],
\end{aligned}    
\end{equation}
which indicates that:
\begin{equation}
\begin{aligned}
\mathbb{E}\Bigg[\sqrt{\sum^{T}\limits_{t=1}\|\nabla\Phi(x_t)\|^2}\Bigg]\leq\frac{\sqrt{M}(m+T)^{1/6}}{\sqrt{\gamma}},
\end{aligned}    
\end{equation}
using Cauchy-Schwarz inequality, we have:
\begin{equation}
\begin{aligned}
    \frac{\sum^{T}\limits_{t=1}\|\nabla\Phi(x_t)\|}{T}\leq\frac{\sqrt{\sum^{T}\limits_{t=1}(\mathbb{E}[\|\nabla\Phi(x_t)\|^2])}}{\sqrt{T}},
\end{aligned}    
\end{equation}
therefore we have:
\begin{equation}
\begin{aligned}
    \mathbb{E}\Bigg[\frac{\sum^{T}\limits_{t=1}\|\nabla\Phi(x_t)\|}{T}\Bigg]&\leq\frac{\sqrt{M}(m+T)^{1/6}}{\sqrt{\gamma T}}\leq\mathcal{O}\bigg(\frac{m^{1/6}\sqrt{M}}{\sqrt{\gamma T}}+\frac{\sqrt{M}}{\sqrt{\gamma}T^{1/3}}\bigg)=\mathcal{O}\bigg(\frac{\sqrt{M}}{\sqrt{\gamma}T^{1/3}}\bigg).
\end{aligned}    
\end{equation}
Therefore, we get the result in Theorem~\ref{The:ConvergenceNonConvex}. \hfill~$\Box$

\section{NSTORM under $\mu_y$-PL condition}
In this section, we consider moderating the function $f(g(x),y)$ with respect to $y$ to follow the PL condition, which highlights the extensibility and applicability of our proposed NSTORM method. In particular, we rely on Assumption \ref{Ass:mu-pl-condition} in place of Assumption~\ref{Ass:StrongConcave}.
\begin{assumption}\label{Ass:mu-pl-condition}
($\mu_y$-PL condition) There exists a constant $\mu_y >0$, such that $\|\nabla_yf(a,b)\|^2 \geq 2\mu_y \big(\max_{b'}f(a,b')-f(a,b)\big)$, where $\forall a \in \mathcal{A}$ and $\forall b ,b' \in \mathcal{B}$.
\end{assumption}
Note that Assumption \ref{Ass:mu-pl-condition} moderates Assumption~\ref{Ass:StrongConcave}, which does not require the function $f(g(x),y)$ to be strongly concave with respect to $y$. In fact, Assumption \ref{Ass:mu-pl-condition} holds even if $f(g(x),y)$ is not concave in $y$ at all. If Assumption \ref{Ass:mu-pl-condition} holds, then we also get that quadratic growth condition and error bound condition hold, and we give the definitions of quadratic growth condition and error bound condition below in the subsequent Lemma \ref{EQ_and_QG}.

\begin{lemma}\label{EQ_and_QG}
    (\cite{karimi2016linear}) Function $f(x):\mathbb{R}^d \to \mathbb{R}$ is $L$-smooth and satisfies PL condition with constant $\mu_y$,  then it also satisfies error bound (EB) condition with $\mu_y$, i.e., $\forall x \in \mathbb{R}^d$
    \begin{equation}
        \|\nabla f(x)\| \geq \mu_y\|x^* -x\|,\nonumber
    \end{equation}
    where $x^* \in \operatorname{argmin}_xf(x)$. It also satisfies quadratic growth (QG) condition with $\mu_y$,i.e.,
    \begin{equation}
    h(\mathbf{x}) - h^* \geq \frac{\mu_y}{2}\operatorname{dist}(\mathbf{x})^2, \quad \forall{x},\nonumber
\end{equation}
where $h^*$ is the minimum value of the function, and $\operatorname{dist}(\mathbf{x})$ is the distance of the point $x$ to the optimal solution set.
\end{lemma}
So in summary, Assumption \ref{Ass:mu-pl-condition} is a weaker condition than Assumption \ref{Ass:StrongConcave}, and its validity implies the quadratic growth and error bound conditions hold, even without concavity of $f(g(x),y)$ in $y$.

Before giving specific details of the proof, we would like to use another form of gradient for a better convergence analysis. We replace line 13 in Algorithm \ref{ALG:SCGDA}:

\begin{equation}
\begin{aligned}
    &\tilde{x}_{t+1} = x_t -\gamma v_t,~x_{t+1} = x_t + \eta_t (\tilde{x}_{t+1}-x_t);~~~~\tilde{y}_{t+1} = y_t + w_t,~y_{t+1} = y_t + \lambda\eta_t(\tilde{y}_{t+1}-y_t).
\end{aligned} 
\end{equation}
Note that we add an extra parameter $\lambda$ when updating $y$, but it turns out that in this case Theorem \ref{theorem:mu_pl} still holds as long as $\gamma \leq 1$. Therefore, as long as the theorem holds for this update rule with the additional $\lambda$ parameter, NSTORM also satisfies the conditions of Theorem \ref{theorem:mu_pl}. For simplicity, we refer to the replaced algorithm still as Algorithm \ref{ALG:SCGDA} in the following part.
\begin{theorem}\label{theorem:mu_pl}
Under the Assumption \ref{Ass:Smooth}-\ref{Ass:BoundedVariance} and \ref{Ass:mu-pl-condition}, for Algorithm \ref{ALG:SCGDA}, by setting $\eta_t = \frac{1}{(m+t)^{1/3}}$, $m > \max\{1,8L^3, c_1^3, c_2^3, (24L^2c_1^2+8L^2c_2^2)^3\}$,  $c_1 \geq 2+360L^2\lambda\gamma(C_f^2+C_g^2)$,  $c_2 \geq\frac{2}{3}+40L^2\lambda^2$, $0<\lambda  \leq \frac{m^{1/3}}{2L}$ and $0 < \gamma < \operatorname{min}\{\frac{\lambda\mu_{y}^2}{8C_g^2L^2},\frac{B}{10C_g^2L^2},\frac{\lambda\mu_{y}^2}{9C_g^2L^2}, \frac{4}{BA_1}\}$, where $A_1 = L_f^2+L_g^2+6L^2L_f^2+2L^2C_g^2$, we can obtain the following:
\begin{equation*}
\frac{1}{T}\sum^T\limits_{t=1}\mathbb{E}[\|\nabla\Phi(x_t)\|]\leq \frac{\sqrt{5M}}{\sqrt{T}}(m+T)^{1/6} \leq \frac{\sqrt{5M}m^{1/6}}{\sqrt{T}}+\frac{\sqrt{5M}}{T^{1/3}},
\end{equation*}
 where $M = \frac{4(\Phi(x_1)-\Phi_{*}+\sigma_g^2+\sigma_{g^{'}}^2+2\sigma_f^2}{\gamma}+\frac{2B(c_1^2(\sigma_g^2+\sigma_{g'}^2+\sigma_f^2+6L^2\sigma_g^2)+2c_2^2\sigma_f^2)}{\gamma} \ln(m+T)).$
\end{theorem}

Then we start from the counterpart of Lemma \ref{Lemma:phi}.
\begin{lemma}
Given Assumptions \ref{Ass:Smooth}-\ref{Ass:BoundedVariance} and \ref{Ass:mu-pl-condition} for Algorithm \ref{ALG:SCGDA}, by setting $\gamma \leq \frac{\lambda\mu^2_y}{8C_g^2L^2}$ and $\lambda \leq \frac{1}{2L\eta_t}$, we have:
\begin{equation*}
\begin{aligned}
   \Phi(x_{t+1}) - f(g(x_{t+1}),y_{t+1}) &\leq   (1-\frac{\lambda\mu_y\eta_t}{2})(\Phi(x_{t})-f(g(x_{t}),y_t))-\frac{\eta_t}{4\lambda}\|\Tilde{y}_{t+1}-y_{t}\|^2\\
    &\quad+\frac{\eta_t}{8\gamma}\|\Tilde{x}_{t+1}-x_t\|^2+ \lambda \eta_t \| \nabla_y f(g(x_t),y_t) - w_t\|^2.\\    
\end{aligned}
\end{equation*}
\end{lemma}

\begin{proof}
Due to the smoothness of $f(g(x),\cdot)$, we have:
\begin{equation}
\begin{aligned}
    f(g(x_{t+1}),y_t)  &\leq f(g(x_{t+1}),y_{t+1}) - \langle \nabla_y f(g(x_{t+1}),y_t),y_{t+1}-y_t \rangle + \frac{L}{2}\|y_{t+1}-y_t\|^2\\
    & \leq f(g(x_{t+1}),y_{t+1}) -\underbrace{ \eta_t\langle \nabla_y f(g(x_{t+1}),y_t),\Tilde{y}_{t+1}-y_t \rangle}\limits_{T_1} + \frac{L\eta_t^2}{2}\|\Tilde{y}_{t+1}-y_t\|^2.
\end{aligned}
\end{equation}
Then, we bound the term $T_1$:
\begin{equation}
\begin{aligned}
    &-T_1\\
    &=- \eta_t\langle \nabla_y f(g(x_{t+1}),y_t),\Tilde{y}_{t+1}-y_t \rangle \\
    & = -\lambda\eta_t\nabla_y f(g(x_{t+1}),y_t),w_t \rangle\\
    & = -\frac{\lambda \eta_t}{2} \big(\|\nabla_y f(g(x_{t+1}),y_t)\|^2+\|w_t\|^2 - \|\nabla_y f(g(x_{t+1}),y_t) - \nabla_y f(g(x_t),y_t)+ \nabla_y f(g(x_t),y_t) - w_t\|^2\big)\\
    & \leq -\frac{\lambda \eta_t}{2}\|\nabla_y f(g(x_{t+1}),y_t)\|^2-\frac{ \eta_t}{2\lambda}\|\Tilde{y}_{t+1}-y_t \|^2+\lambda \eta_t L^2\|g(x_{t+1})-g(x_{t})\|^2+\lambda \eta_t \| \nabla_y f(g(x_t),y_t) - w_t\|^2\\ 
    & \leq -\lambda \eta_t \mu_y \big(\Phi(x_{t+1})-f(g(x_{t+1}),y_t)\big)-\frac{ \eta_t}{2\lambda}\|\Tilde{y}_{t+1}-y_t \|^2+\lambda \eta_tL^2C_g^2\|\ x_{t+1}-x_{t}\|^2+\lambda \eta_t \| \nabla_y f(g(x_t),y_t) - w_t\|^2,
\end{aligned}
\end{equation}
where the last inequality is due to the $\mu_y$-PL condition. Similar to \cite{huang2023enhanced}, we have:
\begin{equation}
\begin{aligned}
     f(g(x_{t+1}),y_t)  &\leq f(g(x_{t+1}),y_{t+1}) -\lambda \eta_t \mu_y \big(\Phi(x_{t+1})-f(g(x_{t+1}),y_t)\big)-\frac{ \eta_t}{2\lambda}\|\Tilde{y}_{t+1}-y_t \|^2\\
     &\quad+\lambda \eta_tL^2C_g^2\|\ x_{t+1}-x_{t}\|^2 +\lambda \eta_t \| \nabla_y f(g(x_t),y_t) - w_t\|^2+\frac{L\eta_t^2}{2}\|\Tilde{y}_{t+1}-y_t \|^2,
\end{aligned}
\end{equation}
rearranging the terms:
\begin{equation}\label{eq13}
\begin{aligned}
    \Phi(x_{t+1}) - f(g(x_{t+1}),y_{t+1})&\leq (1 - \lambda \mu_y \eta_t)(\Phi(x_{t+1})-f(g(x_{t+1}),y_t))-\frac{\eta_t}{2\lambda}\|\Tilde{y}_{t+1}-y_t \|^2\\
    &\quad+ \lambda \eta_tL^2C_g^2\|\ x_{t+1}-x_{t}\|^2 +\lambda \eta_t \| \nabla_y f(g(x_t),y_t) - w_t\|^2+\frac{L\eta_t^2}{2}\|\Tilde{y}_{t+1}-y_t \|^2.
\end{aligned}
\end{equation}
Next, due to $\|\nabla_xf(g(x_1),y_1)-\nabla_xf(g(x_2),y_1)\| \leq LC_g\|x_1-x_2\|$, We can observe that the function $f(g(\cdot),y_t)$ exhibits smoothness with respect to $x$.  As a result, we can deduce that:
\begin{equation}
    f(g(x_t),y_t) + \langle \nabla_xf(g(x_t),y_t),x_{t+1}-x_{t}\rangle -\frac{LC_g}{2}\|x_{t+1}-x_{t}\|^2 \leq f(g(x_{t+1}),y_t),
\end{equation}
then we have:
\begin{equation}\label{15}
\begin{aligned}
    &f(g(x_t),y_t) - f(g(x_{t+1}),y_t) \\
    & \leq -\langle \nabla_xf(g(x_t),y_t),x_{t+1}-x_{t}\rangle + \frac{LC_g\eta_t^2}{2}\|\Tilde{x}_{t+1}-x_{t}\|^2 \\
    & = -\eta_t\langle \nabla_xf(g(x_t),y_t) - \nabla \Phi(x_{t}),\Tilde{x}_{t+1}-x_{t}\rangle-\eta_t\langle   \nabla \Phi(x_{t}),\Tilde{x}_{t+1}-x_{t}\rangle +\frac{LC_g\eta_t^2}{2}\|\Tilde{x}_{t+1}-x_{t}\|^2 \\
    & \leq \frac{\eta_t}{8\gamma}\|\Tilde{x}_{t+1}-x_{t}\|^2+2\gamma\eta_t\|\nabla_xf(g(x_t),y_t) - \nabla \Phi(x_{t})\|^2-\langle   \nabla \Phi(x_{t}),x_{t+1}-x_{t}\rangle+\frac{LC_g\eta_t^2}{2}\|\Tilde{x}_{t+1}-x_{t}\|^2\\
    &\leq \frac{\eta_t}{8\gamma}\|\Tilde{x}_{t+1}-x_{t}\|^2+2\gamma C_g^2L^2\eta_t\|y^*(x_{t})-y_{t}\|^2-\langle   \nabla \Phi(x_{t}),x_{t+1}-x_{t}\rangle+\frac{LC_g\eta_t^2}{2}\|\Tilde{x}_{t+1}-x_{t}\|^2,\\
\end{aligned}
\end{equation}
where the second inequality holds by Cauchy-Schwartz inequality. Then according to the smoothness of $\Phi(x_t)$, we can get:
\begin{equation}\label{eq smooth:phi}
    \Phi(x_{t+1})\leq \Phi(x_{t})+\langle \nabla \Phi(x_{t}),x_{t+1}-x_{t}\rangle +\frac{L_{\Phi}}{2}\|x_{t+1}-x_{t}\|^2.
\end{equation}
Combining \eqref{15} and \eqref{eq smooth:phi}, we obtain:
\begin{equation}
\begin{aligned}
    &f(g(x_t),y_t) - f(g(x_{t+1}),y_t) \\
    &\leq \frac{\eta_t}{8\gamma}\|\Tilde{x}_{t+1}-x_{t}\|^2+2\gamma C_g^2L^2\eta_t\|y^*(x_{t})-y_{t}\|^2+\Phi(x_{t})- \Phi(x_{t+1})+\frac{L_{\Phi}}{2}\|x_{t+1}-x_{t}\|^2+\frac{LC_g\eta_t^2}{2}\|\Tilde{x}_{t+1}-x_{t}\|^2\\
    &\leq \frac{\eta_t}{8\gamma}\|\Tilde{x}_{t+1}-x_{t}\|^2 + \frac{4\gamma C_g^2L^2\eta_t}{\mu_y}\big(\Phi(x_{t})-f(g(x_{t}),y_t) \big)+\Phi(x_{t})- \Phi(x_{t+1})+\frac{L_{\Phi}}{2}\|x_{t+1}-x_{t}\|^2\\
    &\quad+\frac{LC_g\eta_t^2}{2}\|\Tilde{x}_{t+1}-x_{t}\|^2\\
    & = \frac{4\gamma C_g^2L^2\eta_t}{\mu_y}\big(\Phi(x_{t})-f(g(x_{t}),y_t) \big)+\Phi(x_{t})- \Phi(x_{t+1})+(\frac{\eta_t}{8\gamma}+\frac{LC_g\eta_t^2}{2}+\frac{L_{\Phi}\eta_t^2}{2})\|\Tilde{x}_{t+1}-x_{t}\|^2,
\end{aligned}
\end{equation}
where the last inequality holds by quadratic growth (QG) condition with $\mu_y$. 
\\
Then we have:
\begin{equation}\label{eq18}
\begin{aligned}
    \Phi(x_{t+1})-f(g(x_{t+1}),y_{t})
    & = \Phi(x_{t+1}) - \Phi(x_{t}) + \Phi(x_{t}) - f(g(x_{t}),y_{t})+  f(g(x_{t}),y_{t})-f(g(x_{t+1}),y_{t})\\
    & \leq (1+\frac{4\gamma C_g^2L^2\eta_t}{\mu_y})\big(\Phi(x_{t})-f(g(x_{t}),y_t) \big)+(\frac{\eta_t}{8\gamma}+\frac{LC_g\eta_t^2}{2}+\frac{L_{\Phi}\eta_t^2}{2})\|\Tilde{x}_{t+1}-x_{t}\|^2.
\end{aligned}
\end{equation}
Substituting \eqref{eq18} in ~\eqref{eq13}, we get:
\begin{equation}
\begin{aligned}
    &\Phi(x_{t+1})-f(g(x_{t+1}),y_{t+1})\\
    & \leq (1-\lambda\mu_y\eta_t)\Big((1+\frac{4\gamma C_g^2L^2\eta_t}{\mu_y})\big(\Phi(x_{t})-f(g(x_{t}),y_t) \big)+(\frac{\eta_t}{8\gamma}+\frac{LC_g\eta_t^2}{2}+\frac{L_{\Phi}\eta_t^2}{2})\|\Tilde{x}_{t+1}-x_{t}\|^2\Big)\\
    &\quad -\frac{\eta_t}{2\lambda}\|\Tilde{y}_{t+1}-y_t \|^2+ \lambda \eta_tL^2C_g^2\|\ x_{t+1}-x_{t}\|^2 +\lambda \eta_t \| \nabla_y f(g(x_t),y_t) - w_t\|^2+\frac{L\eta_t^2}{2}\|\Tilde{y}_{t+1}-y_t \|^2\\
    & = (1-\lambda\mu_y\eta_t)(1+\frac{4\gamma C_g^2L^2\eta_t}{\mu_y})(\Phi(x_{t})-f(g(x_{t}),y_t))+(\frac{L\eta_t^2}{2}-\frac{\eta_t}{2\lambda})\|\Tilde{y_{t+1}}-y_{t}\|^2\\
    &\quad +((1-\lambda\mu_y\eta_t)(\frac{\eta_t}{8\gamma}+\frac{LC_g\eta_t^2}{2}+\frac{L_{\Phi}\eta_t^2}{2})+\lambda\eta_t^3L^2C_g^2)\|\Tilde{x}_{t+1}-x_{t}\|^2+ \lambda \eta_t \| \nabla_y f(g(x_t),y_t) - w_t\|^2\\
    &\leq (1-\lambda\mu_y\eta_t)(1+\frac{4\gamma C_g^2L^2\eta_t}{\mu_y})(\Phi(x_{t})-f(g(x_{t}),y_t))+(\frac{L\eta_t^2}{2}-\frac{\eta_t}{2\lambda})\|\Tilde{y_{t+1}}-y_{t}\|^2\\
    & \quad+((1-\lambda\mu_y\eta_t)(\frac{\eta_t}{8\gamma}+L_{\Phi}\eta_t^2)+\lambda\eta_t^3L^2C_g^2)\|\Tilde{x}_{t+1}-x_{t}\|^2+\lambda \eta_t \| \nabla_y f(g(x_t),y_t) - w_t\|^2\\
    & \leq (1-\frac{\lambda\mu_y\eta_t}{2})(\Phi(x_{t})-f(g(x_{t}),y_t))-\frac{\eta_t}{4\lambda}\|\Tilde{y_{t+1}}-y_{t}\|^2+\frac{\eta_t}{8\gamma}\|\Tilde{x}_{t+1}-x_t\|^2+ \lambda \eta_t \| \nabla_y f(g(x_t),y_t) - w_t\|^2,\\
\end{aligned}
\end{equation}
where the last inequality holds by $C_g > 1$, $\gamma \leq \frac{\lambda\mu^2_y}{8C_g^2L^2}$ and $\lambda \leq \frac{1}{2L\eta_t}$ and the following inequalities:
\begin{align}
    \gamma \leq \frac{\lambda\mu^2_y}{8C_g^2L^2} &\Rightarrow \frac{\eta_t\lambda\mu_y^2}{2} \geq \frac{4\gamma C_g^2L^2\eta_t}{\mu_y}, \nonumber \\
    \lambda \leq \frac{1}{2L\eta_t} &\Rightarrow \frac{1}{2\lambda} \geq L\eta_t,\\
    L_{\Phi} = (C_fL_g+\frac{3C_g^2L^2}{2\mu_y}) &\Rightarrow \lambda\mu_y L_{\Phi}\eta_t^3 \geq \lambda L^2C_g^2\eta_t^3 \Rightarrow L_{\Phi}\eta_t^2 \leq \frac{\lambda\mu_y \eta_t^2}{8\gamma}.\nonumber 
\end{align}
\end{proof}
We provide  the following definitions and  Lemma \ref{lemma:v,g_x} to prove the counterpart of Lemma \ref{lemma:phismooth}. Given a $\rho$-strong function $\psi(x) : \mathcal{A} \to \mathbb{R}$, we define a Bregman distance\cite{censor1981iterative} \cite{censor1992proximal} associated with $\psi(x)$ as follows:
\begin{equation}
    D(z,x)=\psi(z)-[\psi(x)+\langle\nabla\psi(x),z-x\rangle],\quad\forall x,z \in \mathcal{A}.
\end{equation}
Where $\mathcal{A} \subseteq \mathbb{R}^{d} $ is a closed convex set. Assume $h(x): \mathcal{A} \to \mathbb{R}$ is a convex and possibly non-smooth function, we define a generalized
projection problem as \citep{ghadimi2016mini}:
\begin{equation}\label{defx+}
    x^{+}=\arg \min _{z \in \mathcal{A}}\left\{\langle z, v\rangle+h(z)+\frac{1}{\gamma} D(z, x)\right\}, \quad x \in \mathcal{A}.
\end{equation}
Where $v \in \mathbb{R}^d$ and $\gamma > 0$. . Following \cite{ghadimi2016mini}, we define a generalized gradient as follows:
\begin{equation}
    \mathcal{G}(x,v,\gamma) = \frac{1}{\gamma}(x-x^+).
\end{equation}

\begin{lemma}\label{lemma:v,g_x}
(Lemma 1 in \cite{ghadimi2016mini}) Let $x^+$be given in Eq.(\ref{defx+}). Then we have, for any $x \in \mathcal{A}, v \in \mathbb{R}^d$ and $\gamma >0$:
\begin{equation}
    \left\langle v, \mathcal{G}(x, v, \gamma)\right\rangle \geq \rho\left\|\mathcal{G}(x, v, \gamma)\right\|^{2}+\frac{1}{\gamma}\left[h\left(x^{+}\right)-h(x)\right],\nonumber
\end{equation}
where $\rho > 0 $ depends on $\rho$ strongly convex function $\psi(x)$.
\end{lemma}
Based on Lemma \ref{lemma:v,g_x}, let $h(x) = 0$, we have:
\begin{equation}
    \left\langle v, \mathcal{G}(x, v, \gamma)\right\rangle \geq \rho\left\|\mathcal{G}(x, v, \gamma)\right\|^{2}.
\end{equation}

\begin{lemma}\label{lemma:phi recursive}
Given Assumptions \ref{Ass:Smooth} - \ref{Ass:BoundedVariance} and \ref{Ass:mu-pl-condition} for Algorithm \ref{ALG:SCGDA}, by setting $0 < \gamma \leq \frac{1}{2L_{\Phi}\eta_t}$, we can obtain the following:
\begin{equation}
\begin{aligned}
    \Phi(x_{t+1}) &\leq \Phi(x_t) + \frac{4\gamma L^2C_g^2\eta_t}{\mu_y}\Big(\Phi(x_{t})-f(g(x_{t}),y_t)\Big)+4\gamma C_g^2L^2\eta_t\mathbb{E}[\|u_t-g(x_t)\|^2]\\
    &\quad+8\gamma C_f^2\eta_t\mathbb{E}[\|v'_t-\nabla g(x_t)\|^2]+8\gamma C_g^2\eta_t\mathbb{E}[\|v''_t-\nabla_gf(u_t,y_t)\|^2]-\frac{\eta_t}{2\gamma}\|\Tilde{x}_{t+1}-x_{t}\|^2.
\end{aligned}
\end{equation}
\end{lemma}

\begin{proof}
    According to Lemma ~\ref{lemma:phismooth}, i.e., function $\Phi(x)$ is $L_{\Phi}$-smooth, we have:
    \begin{equation}
    \begin{aligned}
        \Phi(x_{t+1}) &\leq \Phi(x_{t}) + \langle \nabla \Phi(x_{t}),x_{t+1}-x_{t}\rangle+\frac{L_{\Phi}}{2}\|x_{t+1}-x_{t}\|^2 \\
        & = \Phi(x_{t})+ \eta_t\langle \nabla \Phi(x_{t}),\Tilde{x}_{t+1}-x_{t}\rangle + \frac{L_{\Phi}\eta_t^2}{2}\|\Tilde{x}_{t+1}-x_{t}\|^2\\
        & = \Phi(x_{t})+\eta_t\underbrace{\langle v_t ,\Tilde{x}_{t+1}-x_{t}\rangle}\limits_{T_2}+\eta_t\underbrace{\langle \Phi(x_{t})-v_t,\Tilde{x}_{t+1}-x_{t}\rangle}\limits_{T_3}+\frac{L_{\Phi}\eta_t^2}{2}\|\Tilde{x}_{t+1}-x_{t}\|^2.\\
    \end{aligned}
    \end{equation}
    Given the mirror function $\psi_t(x) = \frac{1}{2}x^Tx$ is $\rho$-strongly convex, we define Bergman distance as in ~\cite{ghadimi2016mini}:
    \begin{equation}
        D_t(x,x_t) = \psi_t(x) - [\psi_t(x_t)+\langle\nabla\psi_t(x_t),x-x_t\rangle] = \frac{1}{2}\|x-x_t\|^2.
    \end{equation}
    By applying Lemma \ref{lemma:v,g_x} to Algorithm \ref{ALG:SCGDA} ,i.e., the problem $\Tilde{x}_{t+1}=x_t-\gamma v_t = \operatorname{argmin}_{x \in \mathbb{R}^d}\Big \{ \langle v_t , x\rangle +\frac{1}{2\gamma}\|x-x_t\|^2\Big\}$, we obtain:
    \begin{equation}
        \langle v_t , \frac{1}{\gamma}(x_t-\Tilde{x}_{t+1})\rangle \geq \|\frac{1}{\gamma}(x_{t}-\Tilde{x}_{t+1})\|^2,
    \end{equation}
    thus we have:
    \begin{equation}
        T_2 = \langle v_t, \Tilde{x}_{t+1}-x_t\rangle \leq -\frac{1}{\gamma}\|\Tilde{x}_{t+1}-x_t\|^2.
    \end{equation}
    Next, we bound the term $T_3 $:
    \begin{equation}
    \begin{aligned}
        T_3 &= \langle \Phi(x_{t})-v_t,\Tilde{x}_{t+1}-x_{t}\rangle\\
        & \leq \|\Phi(x_{t})-v_t\|\cdot\|\Tilde{x}_{t+1}-x_{t}\|\\
        & \leq\gamma\|\Phi(x_{t})-v_t\|^2+\frac{1}{4\gamma}\|\Tilde{x}_{t+1}-x_{t}\|^2\\
        & = \gamma\|\nabla_xf(g(x_t),y^*(x_t))-\nabla_xf(g(x_t),y_{t})+\nabla_xf(g(x_t),y_{t})-v_t\|^2+\frac{1}{4\gamma}\|\Tilde{x}_{t+1}-x_{t}\|^2\\
        & \leq 2\gamma\|\nabla_xf(g(x_t),y^*(x_t))-\nabla_xf(g(x_t),y_{t})\|^2+2\gamma\|\nabla_xf(g(x_t),y_{t})-v_t\|^2+\frac{1}{4\gamma}\|\Tilde{x}_{t+1}-x_{t}\|^2\\
        & \leq 2\gamma L^2C_g^2 \|y^*(x_t)-y_{t}\|^2+2\gamma\|\nabla_xf(g(x_t),y_{t})-v_t\|^2+\frac{1}{4\gamma}\|\Tilde{x}_{t+1}-x_{t}\|^2,\\
    \end{aligned}
    \end{equation}
    where the last inequality holds by Assumption \ref{Ass:Smooth}. Now  back to the proof, similar to \cite{huang2023enhanced}, we get:
    \begin{equation}
    \begin{aligned}
         \Phi(x_{t+1}) &\leq\Phi(x_{t})+\eta_t\langle v_t ,\Tilde{x}_{t+1}-x_{t}\rangle+\eta_t\langle \Phi(x_{t})-v_t,\Tilde{x}_{t+1}-x_{t}\rangle+\frac{L_{\Phi}\eta_t^2}{2}\|\Tilde{x}_{t+1}-x_{t}\|^2\\
         & \leq \Phi(x_t)-\frac{\eta_t}{\gamma}\|\Tilde{x}_{t+1}-x_t\|^2+2\gamma L^2C_g^2\eta_t\|y^*(x_t)-y_t\|^2+2\gamma\eta_t\|\nabla_xf(g(x_t),y_{t})-v_t\|^2\\
         &\quad+ \frac{\eta_t}{4\gamma}\|\Tilde{x}_{t+1}-x_{t}\|^2+\frac{L_{\Phi}\eta_t^2}{2}\|\Tilde{x}_{t+1}-x_{t}\|^2\\
         & = \Phi(x_t) +2\gamma L^2C_g^2\eta_t\|y^*(x_t)-y_t\|^2+2\gamma\eta_t\|\nabla_xf(g(x_t),y_{t})-v_t\|^2-\frac{\eta_t}{2\gamma}\|\Tilde{x}_{t+1}-x_{t}\|^2\\
         &\quad-(\frac{\eta_t}{4\gamma}-\frac{L_{\Phi}\eta_t}{2})\|\Tilde{x}_{t+1}-x_{t}\|^2\\
         & \leq \Phi(x_t) +2\gamma L^2C_g^2\eta_t\|y^*(x_t)-y_t\|^2+2\gamma\eta_t\|\nabla_xf(g(x_t),y_{t})-v_t\|^2-\frac{\eta_t}{2\gamma}\|\Tilde{x}_{t+1}-x_{t}\|^2\\
         & \leq \Phi(x_t) + \frac{4\gamma L^2C_g^2\eta_t}{\mu_y}\Big(\Phi(x_{t})-f(g(x_{t}),y_t)\Big)+2\gamma\eta_t\|\nabla_xf(g(x_t),y_{t})-v_t\|^2-\frac{\eta_t}{2\gamma}\|\Tilde{x}_{t+1}-x_{t}\|^2,\\
    \end{aligned}
    \end{equation}
    where the last two inequality holds by $0 < \gamma \leq \frac{1}{2L_{\Phi}\eta_t}$. Then, based on Lemma \ref{lemma_vt_variance}, we can obtain the following: 
    \begin{equation}
    \begin{aligned}
        \Phi(x_{t+1}) &\leq \Phi(x_t) + \frac{4\gamma L^2C_g^2\eta_t}{\mu_y}\Big(\Phi(x_{t})-f(g(x_{t}),y_t)\Big)+4\gamma C_g^2L^2\eta_t\mathbb{E}[\|u_t-g(x_t)\|^2]\\
        &\quad+8\gamma C_f^2\eta_t\mathbb{E}[\|v'_t-\nabla g(x_t)\|^2]+8\gamma C_g^2\eta_t\mathbb{E}[\|v''_t-\nabla_gf(u_t,y_t)\|^2]-\frac{\eta_t}{2\gamma}\|\Tilde{x}_{t+1}-x_{t}\|^2.
    \end{aligned}
    \end{equation}
This completes the proof.
\end{proof}
Note that replacing the Assumption does not affect the bounded variance of the estimators. With this in mind, we now present the convergence analysis. Defining the Lyapunov function, for any $t \geq 1$:
\begin{equation}
\begin{aligned}
    P_t &= \mathbb{E}\Big[\Phi(x_t)+(\Phi(x_t)-f(g(x_t),y_t))+\frac{B}{\eta_{t-1}}\big(\|u_t-g(x_t)\|^2+\|v'_t-\nabla g(x_t)\|^2\\
    &\quad+\|\nabla_gf(u_t,y_t)-v''_t\|^2+\|\nabla_yf(g(x_t),y_t)-w_t\|^2\big)\Big] ,   
\end{aligned}
\end{equation}
where $B = \frac{1}{40L^2\lambda}$. And we denote that $\mathcal{T}_{t} = \mathbb{E}[\Phi(x_t)-f(g(x_t),y_t)]$,  $\mathcal{J}_{t} = \mathbb{E}[\|u_t-g(x_t)\|^2]$, $\mathcal{H}_{t} = \mathbb{E}[\|v'_t-\nabla g(x_t)\|^2]$, $\mathcal{X}_{t} = \mathbb{E}[\|v''_t - \nabla_g f(u_t,y_t)\|^2]$, $
\mathcal{K}_{t} = \mathbb{E}[\|w_t - \nabla_y f(g(x_t),y_t)\|^2]$, according to  above lemmas, we have:
\begin{equation}
\begin{aligned}
    P_{t+1}-P_t 
    & = \frac{4\gamma L^2 \eta_t^2}{\mu_{y}}\mathcal{T}_{t}+2\gamma\eta_t(2C_g^2L^2\mathcal{J}_{t}+4C_f^2\mathcal{H}_{t}+4C_g^2\mathcal{X}_{t})-\frac{\eta_t}{2\gamma}\mathbb{E}[\|\Tilde{x}_{t+1}-x_t\|^2]+\mathcal{T}_{t+1}\\
    &\quad-\mathcal{T}_{t}+\frac{B}{\eta_t}(\mathcal{J}_{t+1}+H_{t+1}+\mathcal{X}_{t+1}+\mathcal{K}_{t+1})-\frac{B}{\eta_{t-1}}(\mathcal{J}_{t}+\mathcal{H}_{t}+\mathcal{X}_{t}+\mathcal{K}_{t})\\
    & \leq (\frac{4\gamma L^2\eta_t^2}{\mu_y}-\frac{\lambda\mu_{y}\eta_t}{2})\mathcal{T}_{t}+(\frac{B(1-\beta_{t+1})}{\eta_t}+\frac{B(12L^2\beta_{t+1}+4L^2\alpha_{t+1})}{\eta_t}-\frac{B}{\eta_t})\mathcal{J}_{t}\\
    &\quad+(8\gamma C_f^2\eta_t+\frac{B}{\eta_t}(1-\beta_{t+1})-\frac{B}{\eta_{t-1}})\mathcal{H}_{t}+(8\gamma C_g^2\eta_t+\frac{B}{\eta_t}(1-\beta_{t+1})-\frac{B}{\eta_{t-1}})\mathcal{X}_{t}\\
    &\quad+(\lambda \eta_t+\frac{B}{\eta_t}(1-\beta_{t+1})-\frac{B}{\eta_{t-1}})\mathcal{K}_{t}+(-\frac{\eta_t}{4\lambda}+10L^2B\eta_t)\mathbb{E}[\|\Tilde{y}_{t+1}-y_t\|^2]\\
    &\quad+(\frac{\eta_t}{8\gamma}-\frac{\eta_t}{2\gamma}+B\eta_t(2L_f^2 +2L_g^2 +12L^2 L_f^2+4L^2C_g^2))\mathbb{E}[\|\Tilde{x}_{t+1}-x_t\|^2]\\
    &\quad+\frac{2B}{\eta_t}(\beta_{t+1}^2(\sigma_g^2+\sigma_{g'}^2+\sigma_f^2+6L^2\sigma_g^2)+2\alpha_{t+1}^2\sigma_f^2)
\end{aligned}
\end{equation}
Let $\eta_t = \frac{1}{(m+t)^{1/3}}$, $\beta_{t+1}=c_1\beta_t^2$,  $\alpha_{t+1}= c_2\beta_t^2$ and $m > \operatorname{max}\{1, c_1^3, c_2^3, (24L^2c_1^2+8L^2c_2^2)^3\}$, we have $\eta_t \leq \eta_0 = \frac{1}{m^{1/3}} \leq 1$, $\beta_{t} = c_1\eta_{t-1}^2 < c_1\eta_0 < 1$ and $\alpha_{t} = c_2\eta_{t-1}^2 < c_2\eta_0 < 1$, then following \eqref{Eq:eta_t}, we can get $\frac{1}{\eta_t}-\frac{1}{\eta_{t-1}} \leq \frac{2}{3}\eta_t^2$.
\\
Next,  Let $c_1 \geq 2 + \frac{9\gamma(C_f^2+C_g^2)}{B}$, we have:
\begin{equation}
\begin{aligned}
    &8\gamma C_f^2\eta_t+\frac{B}{\eta_t}(1-\beta_{t+1})-\frac{B}{\eta_{t-1}}\leq 8\gamma C_f^2 \eta_t + \frac{2B\eta_t}{3} - c_1B\eta_t <  -\frac{C_f^2\gamma\eta_t}{4},
\end{aligned}
\end{equation}
and
\begin{equation}
\begin{aligned}
    &8\gamma C_g^2\eta_t+\frac{B}{\eta_t}(1-\beta_{t+1})-\frac{B}{\eta_{t-1}}\leq 8\gamma C_g^2 \eta_t+ \frac{2B\eta_t}{3} - c_1B\eta_t <  -\frac{C_g^2\gamma\eta_t}{4}.
\end{aligned}
\end{equation}
Let $c_2 \geq \frac{2}{3}+\frac{\lambda}{B}$, we get:
\begin{equation}
\begin{aligned}
    &\lambda \eta_t+\frac{B}{\eta_t}(1-\beta_{t+1})-\frac{B}{\eta_{t-1}} \leq \lambda\eta_t + \frac{2\eta_t}{3}-c_2B\eta_t \leq 0.
\end{aligned}
\end{equation}
Then, setting $0 < \lambda \leq \frac{m^{1/3}}{2L}\leq \frac{1}{2L\eta_t}$, $0 < \gamma < \operatorname{min}\{\frac{\lambda\mu^2_y}{8C_g^2L^2},\frac{B}{10C_g^2L^2},\frac{\lambda\mu^2_y}{9C_g^2L^2}, \frac{4}{BA_1}\}$, where $A_1 = L_f^2+L_g^2+6L^2L_f^2+2L^2C_g^2$, we have: 
\begin{equation}\label{eq p_t }
\begin{aligned}
    P_{t+1}-P_t &\leq -\frac{L^2 C_g^2 \gamma\eta_t}{2\mu_y}\mathcal{T}_{t} - \frac{C_g^2L^2\gamma\eta_t}{4}\mathcal{J}_{t} - \frac{C_f^2\gamma\eta_t}{4}\mathcal{H}_{t} - \frac{C_g^2\gamma\eta_t}{4}\mathcal{X}_{t} -\frac{\eta_t}{4\gamma}\mathbb{E}[\|\Tilde{x}_{t+1}-x_t\|^2]\\
    & \quad+\eta_t^3(2B(c_1^2(\sigma_g^2+\sigma_{g'}^2+\sigma_f^2+6L^2\sigma_g^2)+2c_2^2\sigma_f^2)).
\end{aligned}
\end{equation}
Taking average over $t= 1,2,\dots,T$ on both sides of (\ref{eq p_t }), we have:
\begin{equation}
\begin{aligned}
    &\frac{1}{T}\sum^T\limits_{t =1 }\Big(\frac{2L^2 C_g^2 \eta_t}{\mu_y}\mathcal{T}_{t} + C_g^2L^2\eta_t\mathcal{J}_{t} + C_f^2\eta_t\mathcal{H}_{t} + C_g^2\eta_t\mathcal{X}_{t} +\frac{\eta_t}{\gamma^2}\mathbb{E}[\|\Tilde{x}_{t+1}-x_t\|^2]\Big) \\
    & \quad\leq  \frac{4(P_1-P_{T+1})}{\gamma T}+\frac{2B(c_1^2(\sigma_g^2+\sigma_{g'}^2+\sigma_f^2+6L^2\sigma_g^2)+2c_2^2\sigma_f^2)}{\gamma T}\sum^T\limits_{t=1}\eta_{t}^3,
\end{aligned}
\end{equation}
Since $\eta_t$ is decreasing, we have:
\begin{equation}
\begin{aligned}
    &\frac{1}{T}\sum^T\limits_{t =1 }\Big(\frac{2L^2 C_g^2}{\mu_y}\mathcal{T}_{t} + C_g^2L^2\mathcal{J}_{t} + C_f^2\mathcal{H}_{t} + C_g^2 \mathcal{X}_{t} +\frac{1}{\gamma^2}\mathbb{E}[\|\Tilde{x}_{t+1}-x_t\|^2]\Big) \\
    & \quad\leq  \frac{4(P_1-P_{T+1})}{\gamma\eta_T T}+\frac{2B(c_1^2(\sigma_g^2+\sigma_{g'}^2+\sigma_f^2+6L^2\sigma_g^2)+2c_2^2\sigma_f^2)}{\gamma \eta_T T}\sum^T\limits_{t=1}\eta_{t}^3,
\end{aligned}
\end{equation}
Setting $M = \frac{4(\Phi(x_1)-\Phi_{*}+\sigma_g^2+\sigma_{g^{'}}^2+2\sigma_f^2}{\gamma}+\frac{2B(c_1^2(\sigma_g^2+\sigma_{g'}^2+\sigma_f^2+6L^2\sigma_g^2)+2c_2^2\sigma_f^2)}{\gamma} \ln(m+T))$,  from the initialization, we can easily get:
\begin{equation}
\begin{aligned}
    &\frac{1}{T}\sum^T\limits_{t =1 }\Big(\frac{2L^2 C_g^2}{\mu_y}\mathcal{T}_{t} + C_g^2L^2\mathcal{J}_{t} + C_f^2\mathcal{H}_{t} + C_g^2 \mathcal{X}_{t} +\frac{1}{\gamma^2}\mathbb{E}[\|\Tilde{x}_{t+1}-x_t\|^2]\Big) = M\frac{(m+T)^{1/3}}{T}.
\end{aligned}
\end{equation}
We set:
\begin{equation}
\begin{aligned}
    \mathcal{G}_t &= \frac{\sqrt{2}LC_g}{\sqrt{\mu_{y}}}\sqrt{\Phi(x_t)-f(g(x_t),y_t)}+C_gL\|u_t-g(x_t)\|+C_f\|v'_t-\nabla g(x_t)\|\\
    &\quad+C_g\|v''_t-\nabla_g f(u_t,y_t)\|+\frac{1}{\gamma}\|\Tilde{x}_{t+1}-x_t\|  .  
\end{aligned}
\end{equation}
According to Jensen's inequality, we have:
\begin{equation}
\begin{aligned}
    \frac{1}{T}\sum^T\limits_{t=1}\mathbb{E}[\mathcal{G}_t]&\leq \Big(\frac{1}{T}\sum^T\limits_{t=1}\mathbb{E}[\frac{10L^2 C_g^2}{\mu_y}(\Phi(x_t)-f(g(x_t),y_t))+5C_g^2 L^2\|u_t-g(x_t)\|^2\\
    &\quad+5C_f^2\|v'_t-\nabla g(x_t)\|^2+5C_g^2\|v''_t-\nabla_gf(u_t,y_t)\|^2+\frac{5}{\gamma^2}\|\Tilde{x}_{t+1}-x_t\|^2]\Big)^{1/2}\\
    & \leq \frac{\sqrt{5M}}{\sqrt{T}}(m+T)^{1/6}.
\end{aligned}
\end{equation}
According to $\mu_{y}$-PL condition we have:
\begin{equation}
\begin{aligned}
    \Phi(x_t)-f(g(x_t),y_t)&= f(g(x_t),y^*(x_t)) - f(g(x_t),y_t)\\
    &= \max\limits_{y}f(g(x_t),y) - f(g(x_t),y_t) \geq \frac{\mu_y}{2}\|y^*(x_t)-y_t\|^2,
\end{aligned}
\end{equation}
then we have:
\begin{equation}\label{eq:term_Phi-f(g(xt),yt)}
    \frac{\sqrt{2}}{\sqrt{\mu_y}}\sqrt{\Phi(x_t)-f(g(x_t),y_t)} \geq \|y^*(x_t)-y_t\|.
\end{equation}
Besides,
\begin{equation}\label{eq:term_Cf-and-Cg}
\begin{aligned}
    &C_f\|v'_t-\nabla g(x_t)\| +C_g\|v''_t-\nabla_gf(u_t,y_t)\|\\
    & \geq \|\nabla g(x_t)\nabla_gf(u_t,y_t)-v'_t\nabla_gf(u_t,y_t)\|+\|v'_t\nabla_gf(u_t,y_t)-v'_tv''_t\|\\
    & \geq \|\nabla g(x_t)\nabla_gf(u_t,y_t)-v'_t\nabla_gf(u_t,y_t)+v'_t\nabla_gf(u_t,y_t)-v'_tv''_t\|\\
    & = \|\nabla g(x_t)\nabla_gf(u_t,y_t)-v'_tv''_t\|,
\end{aligned}
\end{equation}
and we also have:
\begin{equation}\label{eq:term_CgL-and-}
\begin{aligned}
    &C_gL\|u_t-g(x_t)\|+ \|\nabla g(x_t)\nabla_gf(u_t,y_t)-v'_tv''_t\|\\
    & \geq \|\nabla_xf(g(x_t),y_t)-\nabla g(x_t)\nabla_gf(u_t,y_t)\|+\|\nabla g(x_t)\nabla_gf(u_t,y_t)-v'_tv''_t\|\\
    & \geq \|\nabla_xf(g(x_t),y_t)-\nabla g(x_t)\nabla_gf(u_t,y_t)+\nabla g(x_t)\nabla_gf(u_t,y_t)-v'_tv''_t\|\\
    & = \|\nabla_xf(g(x_t),y_t)-v_t\|.
\end{aligned}
\end{equation}
Combine  \eqref{eq:term_Phi-f(g(xt),yt)}, \eqref{eq:term_Cf-and-Cg} and \eqref{eq:term_CgL-and-}, we have:
\begin{equation}
\begin{aligned}
    \mathcal{G}_t &\geq LC_g\|y^*(x_t)-y_t\|+\|\nabla_xf(g(x_t),y_t)-v_t\|+\frac{1}{\gamma}\|\Tilde{x}_{t+1}-x_t\|\\
    & \geq \|\nabla \Phi(x_t)-\nabla_x f(g(x_t),y_t)\|+\|\nabla_xf(g(x_t),y_t)-v_t\|+\frac{1}{\gamma}\|\Tilde{x}_{t+1}-x_t\|\\
    & \geq \|\nabla \Phi(x_t)-v_t\|+\frac{1}{\gamma}\|\Tilde{x}_{t+1}-x_t\|\\
    & = \|\nabla \Phi(x_t)-v_t\| + \|v_t\| \geq \|\nabla \Phi(x_t)\|.
\end{aligned}
\end{equation}
Then we obtain:
\begin{equation}
    \frac{1}{T}\sum^T\limits_{t=1}\mathbb{E}[\|\nabla\Phi(x_t)\|]\leq \frac{1}{T} \sum^T\limits_{t=1}\mathbb{E}[\mathcal{G}_t] \leq \frac{\sqrt{5M}}{\sqrt{T}}(m+T)^{1/6} \leq \frac{\sqrt{5M}m^{1/6}}{\sqrt{T}}+\frac{\sqrt{5M}}{T^{1/3}}.
\end{equation}
This completes the proof and obtains the results in Theorem~\ref{theorem:mu_pl}.\hfill~$\Box$
\\
Let $m = O(1)$, we have $M = O(1)$. Therefore, NSTORM under $PL$-condition has the same convergence rate, i.e., $O(1/T^{1/3})$. Similarly, let $\frac{1}{T}\sum^T\limits_{t=1}\mathbb{E}[\|\nabla\Phi(x_t)\|] = O(1/T^{1/3}) \leq \epsilon$, we have $T = O(\kappa^3/\epsilon^3)$. Moreover, in each iteration NSOTRM only requires two samples, $\xi$ and $\zeta$, to estimate the gradient and function values. Therefore, the sample complexity for NSOTRM to find an $\epsilon$-stationary point for compositional minimax problems is $O(\kappa^3/\epsilon^3)$.

\section{Proof of Theorem~\ref{theorem ago2}}

\begin{lemma}\label{lemma:Phi_a}
Given Assumptions \ref{Ass:Smooth}-\ref{ass ab bound} for Algorithm \ref{alg:AOA}, by setting $0 \ < \gamma \ < \frac{\rho}{2L_\Phi \eta_t}$, based on Lemma \ref{lemma_vt_variance}, we have:
\begin{equation}
\begin{aligned}
    \Phi(x_{t+1})&\leq \Phi(x_t)+\frac{2\gamma C_g^2 L^2 \eta_t}{\rho}\|y^{*}(x_t)-y_{t}\|^2+\frac{4C_g^2L^2\gamma \eta_t}{\rho}\|u_{t}-g(x_t)\|^2\\
    &\quad+\frac{8C_f^2\gamma\eta_t}{\rho}\|v'_t-\nabla g(x_t)\|^2+\frac{8C_g^2\gamma\eta_t}{\rho}\|v''_t-\nabla_gf(u_t,y_t)\|^2-\frac{\rho \eta_t}{2\gamma}\|\Tilde{x}_{t+1}-x_t\|^2.\nonumber
\end{aligned}
\end{equation}
\end{lemma}

\begin{proof}
Since $\Phi(x_{t})$ is $L_{\Phi}$-Lipschitz continuous gradient, we have:
\begin{equation}
\begin{aligned}
    \Phi(x_{t+1})&\leq\Phi(x_t)+\langle \nabla\Phi(x_t),x_{t+1}-x_t\rangle+\frac{L_{\Phi}}{2}\|x_{t+1}-x_t\|^2\\
    & = \Phi(x_t)+\eta_t\langle \nabla\Phi(x_t),\Tilde{x}_{t+1}-x_t\rangle+\frac{L_{\Phi}\eta_t^2}{2}\|\Tilde{x}_{t+1}-x_t\|^2\\
    & = \Phi(x_t)+\eta_t\langle v_t ,\Tilde{x}_{t+1}-x_t\rangle+\eta_t\langle\nabla\Phi(x_t)-v_t,\Tilde{x}_{t+1}-x_t\rangle+\frac{L_{\Phi}\eta_t^2}{2}\|\Tilde{x}_{t+1}-x_t\|^2\\
    & = \Phi(x_t)+\eta_t\langle v_t ,\Tilde{x}_{t+1}-x_t\rangle + \eta_t\langle \nabla\Phi(x_t)-\nabla_x f(g(x_t),y_t),\Tilde{x}_{t+1}-x_t\rangle\\
    &\quad+\eta_t\langle \nabla_x f(g(x_t),y_t)-v_t,\Tilde{x}_{t+1}-x_t\rangle+\frac{L_{\Phi}\eta_t^2}{2}\|\Tilde{x}_{t+1}-x_t\|^2 \\
    &\leq \Phi(x_t)+\eta_t\underbrace{\langle v_t ,\Tilde{x}_{t+1}-x_t\rangle}\limits_{T_1}+\eta_t\underbrace{\|\nabla\Phi(x_t)-\nabla_x f(g(x_t),y_t)\|\cdot\|\Tilde{x}_{t+1}-x_t\|}\limits_{T_2}\\
    &\quad+\eta_t\underbrace{\|\nabla_x f(g(x_t),y_t)-v_t\|\cdot\|\Tilde {x}_{t+1}-x_t\|}\limits_{T_3}+\frac{L_{\Phi}\eta_t^2}{2}\|\Tilde{x}_{t+1}-x_t\|^2.
\end{aligned}    
\end{equation}

According to Assumption ~\ref{ass ab bound}, we get $A_{t} \succ \rho I_{d}$, then according to Lemma \ref{lemma:v,g_x}, we have:
\begin{equation}
\begin{aligned}
    \langle v_{t}, \frac{1}{\gamma}\left(x_{t}-\tilde{x}_{t+1}\right)\rangle \geq \rho\|\frac{1}{\gamma}\left(x_{t}-\tilde{x}_{t+1}\right)\|^{2}\Rightarrow\left\langle v_{t}, \tilde{x}_{t+1}-x_{t}\right\rangle \leq-\frac{\rho}{\gamma}\left\|\tilde{x}_{t+1}-x_{t}\right\|^{2},
\end{aligned}    
\end{equation}
thus we have:
\begin{equation}
\begin{aligned}
    T_1 = \langle v_t, \Tilde{x}_{t+1}-x_t)\rangle \leq -\frac{\rho}{\gamma}\|\Tilde{x}_{t+1}-x_t\|^2.
\end{aligned}    
\end{equation}
For the term $T_2$, we have:
\begin{equation}
\begin{aligned}
    T_2 &= \|\nabla\Phi(x_t)-\nabla_x f(g(x_t),y_t)\|\cdot\|\Tilde{x}_{t+1}-x_t\| \\
    &\leq \frac{2\gamma}{\rho}\|\nabla\Phi(x_t)-\nabla_x f(g(x_t),y_t)\|^2+\frac{\rho}{8\gamma}\|\Tilde{x}_{t+1}-x_t\|^2 \\
    & \leq \frac{2\gamma C_g^2 L^2}{\rho}\|y^{*}(x_t)-y_t\|^2+\frac{\rho}{8 \gamma }\|\Tilde{x}_{t+1}-x_t\|^2,
\end{aligned}    
\end{equation}
where the second inequality holds by $\langle a,b\rangle \ < \frac{\nu}{2}\|a\|^2+\frac{1}{2\nu}\|b\|^2$, $\nu = \frac{4\gamma}{\rho}$, and the last inequality holds by Assumption~\ref{Ass:Smooth}.
\\
Similarly, for the term $T_3$, we have:
\begin{equation}
\begin{aligned}
    T_3 & = \|\nabla_x f(g(x_t),y_t)-v_t\|\cdot\|\Tilde {x}_{t+1}-x_t\| \leq\frac{2\gamma}{\rho}\|\nabla_xf(g(x_t),y_t)-v_t\|^2+\frac{\rho}{8\gamma}\|\Tilde{x}_{t+1}-x_t\|^2.
\end{aligned}    
\end{equation}
\\
Combine the term $T_1, T_2$ and $T_3$, we have:
\begin{equation}
\begin{aligned}
    \Phi(x_{t+1})&\leq \Phi(x_t)-\frac{\rho \eta_t}{\gamma}\|\Tilde{x}_{t+1}-x_t\|^2+\frac{2\gamma C_g^2 L^2\eta_t}{\rho}\|y^*(x_t)-y_t\|^2\\
    &\quad+\frac{\rho \eta_t}{4\gamma}\|\Tilde{x}_{t+1}-x_t\|^2+\frac{2\gamma \eta_t}{\rho}\|\nabla_xf(g(x_t),y_t)-v_t\|^2+\frac{L_{\Phi}\eta_t^2}{2}\|\Tilde{x}_{t+1}-x_t\|^2\\
    &\leq \Phi(x_t)+\frac{2\gamma C_g^2 L^2 \eta_t}{\rho}\|y^*(x_t)-y_t\|^2+\frac{2\gamma \eta_t}{\rho}\|\nabla_xf(g(x_t),y_t)-v_t\|^2 \\
    &\quad +(-\frac{\rho\eta_t}{\gamma}+\frac{\rho \eta_t}{4\gamma}+\frac{L_{\Phi}\eta_t^2}{2})\|\Tilde{x}_{t+1}-x_t\|^2\\
    &\leq\ \Phi(x_t)+\frac{2\gamma C_g^2 L^2 \eta_t}{\rho}\|y^*(x_t)-y_t\|^2+\frac{2\gamma \eta_t}{\rho}\|\nabla_xf(g(x_t),y_t)-v_t\|^2-\frac{\rho \eta_t}{2\gamma}\|\Tilde{x}_{t+1}-x_t\|^2,
\end{aligned}    
\end{equation}
where the last inequality follows by  $0 \ < \gamma \ < \frac{\rho}{2L_\Phi \eta_t}$.\\
Then according to Lemma \ref{lemma_vt_variance}, we have:
\begin{equation}
\begin{aligned}
    \Phi(x_{t+1})&\leq \Phi(x_t)+\frac{2\gamma C_g^2 L^2 \eta_t}{\rho}\|y^{*}(x_t)-y_{t}\|^2+\frac{4C_g^2L^2\gamma \eta_t}{\rho}\|u_{t}-g(x_t)\|^2\\
    &\quad+\frac{8C_f^2\gamma\eta_t}{\rho}\|v'_t-\nabla g(x_t)\|^2+\frac{8C_g^2\gamma\eta_t}{\rho}\|v''_t-\nabla_gf(u_t,y_t)\|^2-\frac{\rho \eta_t}{2\gamma}\|\Tilde{x}_{t+1}-x_t\|^2.
\end{aligned}
\end{equation}
This completes the proof.
\end{proof}

\begin{lemma}
Given Assumptions \ref{Ass:Smooth} - \ref{Ass:StrongConcave} for Algorithm \ref{alg:AOA}, by setting $0 < \lambda \leq \frac{b}{6L} \leq \frac{b_t}{6L}$, we have:
\begin{equation} 
\begin{aligned}
     \|y^*(x_{t+1})-y_{t+1}\|^2&\leq(1-\frac{\mu \lambda \eta_t}{4b_t})\|y^*(x_{t})-y_{t}\|^2-\frac{3\eta_t}{4}\|\Tilde{y}_{t+1}-y_t\|^2  \\
     &\quad+\frac{25\lambda\eta_t}{6\mu b_t}\|\nabla_yf(g(x_{t}),y_t)-w_t\|^2+\frac{25\kappa^2b_t\eta_t}{6\mu \lambda}\|\Tilde{x}_{t+1}-x_t\|^2.\nonumber
\end{aligned}
\end{equation}    
\end{lemma}

\begin{proof}
Similar to the proof of Lemma 6 in \cite{huang2023adagda}, according to $\mu$-strongly concave of function $f(g(x_t),y)$ in $y$, we have:
\begin{equation}\label{eqmuconcave}
\begin{aligned}
    f(g(x_{t}),y^*(x_t))&\leq f(g(x_{t}),y_t)+\langle \nabla_yf(g(x_{t}),y_t),y^*(x_t)-y_t\rangle+\frac{\mu}{2}\|y^*(x_t)-y_t\|^2\\
    & = f(g(x_{t}),y_t)+\langle w_t,y^*(x_t)-\Tilde{y}_{t+1}\rangle +\langle \nabla_y f(g(x_{t}),y_t)-w_t,y^*(x_t)-\Tilde{y}_{t+1}\rangle\\
    &\quad+\langle \nabla_y f(g(x_{t}),y_t),\Tilde{y}_{t+1}-y_t\rangle - \frac{\mu}{2}\|y^*(x_t)-y_t\|^2.
\end{aligned}
\end{equation}    
Due to the smooth of function $f(g(x_t),y)$ in $y$, we have:
\begin{equation}\label{eqsmooth}
\begin{aligned}
    -\frac{L}{2}\|\Tilde{y}_{t+1}-y_t\|^2\leq f(g(x_{t}),\Tilde{y}_{t+1})-f(g(x_{t}),y_t)-\langle \nabla_yf(g(x_{t}),y_t),\Tilde{y}_{t+1}-y_t\rangle.
\end{aligned}
\end{equation}
Summing up \eqref{eqmuconcave}) and \eqref{eqsmooth}), we have:
\begin{equation}\label{eqsm+mu}
\begin{aligned}
    f(g(x_{t}),y^*(x_t))&\leq f(g(x_{t}),\Tilde{y}_{t+1})+\langle w_t,y^*(x_t)-\Tilde{y}_{t+1}\rangle+\frac{L}{2}\|\Tilde{y}_{t+1}-y_t\|^2-\frac{\mu}{2}\|y^*(x_t)-y_t\|^2\\
    &\quad+\langle \nabla_{y} f(g(x_{t}),y_{t})-w_{t}, y^*(x_t)-\Tilde{y}_{t+1} \rangle.
\end{aligned}
\end{equation}
By the optimal of lines 14-15 in Algorithm \ref{alg:AOA} and the definition of $B_t$, we have:
\begin{equation}
    \langle -w_t+\frac{b_t}{\lambda}(\Tilde{y}_{t+1}-y_t),y^*(x_t)-\Tilde{y}_{t+1} \rangle \geq 0,
\end{equation}
then,
\begin{equation}\label{eqsm+mu part1}
\begin{aligned}
    \langle w_t,y^*(x_t)-\Tilde{y}_{t+1}\rangle &\leq \frac{1}{\lambda}\langle b_t(\Tilde{y}_{t+1}-y_t),y^*(x_t)-\Tilde{y}_{t+1} \rangle\\
    & = \frac{1}{\lambda}\langle b_t(\Tilde{y}_{t+1}-y_t),y_t-\Tilde{y}_{t+1}\rangle+\frac{1}{\lambda}\langle b_t(\Tilde{y}_{t+1}-y_t),y^*(x_t)-y_t\rangle\\
    & = \frac{-b_t}{\lambda}\|\Tilde{y}_{t+1}-y_t\|^2+\frac{b_t}{\lambda}\langle \Tilde{y}_{t+1}-y_t,y^*(x_t)-y_t\rangle.
\end{aligned}
\end{equation}
Summing up \eqref{eqsm+mu} and \eqref{eqsm+mu part1}, we have:
\begin{equation}
\begin{aligned}
    f(g(x_{t}),y^*(x_t))&\leq f(g(x_{t}),\Tilde{y}_{t+1})+\frac{b_t}{\lambda}\langle \Tilde{y}_{t+1}-y_t,y^*-y_t\rangle - \frac{b_t}{\lambda}\|\Tilde{y}_{t+1}-y_t\|^2\\
    &\quad + \langle \nabla_yf(g(x_{t}),y_t)-w_t,y^*-\Tilde{y}_{t+1}\rangle - \frac{\mu}{2}\|y^*-y_t\|^2+\frac{L}{2}\|\Tilde{y}_{t+1}-y_t\|^2.
\end{aligned}
\end{equation}
Due to the definition of $y^{*}(x_{t})$, then:
\begin{equation}\label{eq9}
\begin{aligned}
0 \leq & \frac{b_{t}}{\lambda}\left\langle\tilde{y}_{t+1}-y_{t}, y^{*}\left(x_{t}\right)-y_{t}\right\rangle+\left\langle\nabla_{y} f\left(g(x_{t}), y_{t}\right)-w_{t}, y^{*}\left(x_{t}\right)-\tilde{y}_{t+1}\right\rangle \\
& -\frac{b_{t}}{\lambda}\left\|\tilde{y}_{t+1}-y_{t}\right\|^{2}-\frac{\mu}{2}\left\|y^{*}\left(x_{t}\right)-y_{t}\right\|^{2}+\frac{L}{2}\left\|\tilde{y}_{t+1}-y_{t}\right\|^{2}.
\end{aligned}    
\end{equation}
By $y_{t+1} = y_{t}+\eta_{t}(\Tilde{y}_{t+1}-y_{t})$, we have:
\begin{equation}
\begin{aligned}
    \|y_{t+1}-y^*(x_t)\|^2&=\|y_t+\eta_t(\tilde{y}_{t+1}-y_t)-y^*(x_{t})\|^2\\ 
    &=\|y_t-y^*(x_t)\|^2+2\eta_t\langle\tilde{y}_{t+1}-y_t,y_t-y^*(x_t)\rangle+\eta_t^2\|\tilde{y}_{t+1}-y_t\|^2,
\end{aligned}
\end{equation}
then,
\begin{equation}\label{eq11}
    \langle\tilde{y}_{t+1}-y_t,y^*(x_t)-y_t\rangle\leq\dfrac{1}{2\eta_t}\|y_t-y^*(x_t)\|^2+\dfrac{\eta_t}{2}\|\tilde{y_{t+1}}-y_t\|^2-\dfrac{1}{2{\eta_t}}\|y_{t+1}-y^*(x_{t})\|^2.
\end{equation}
Considering the term $\left\langle\nabla_{y} f\left(g(x_{t}), y_{t}\right)-w_{t}, y^{*}\left(x_{t}\right)-\tilde{y}_{t+1}\right\rangle$, we have:
\begin{equation}\label{eq12}
\begin{aligned}
    &\left\langle\nabla_{y} f\left(g(x_{t}), y_{t}\right)-w_{t}, y^{*}\left(x_{t}\right)-\tilde{y}_{t+1}\right\rangle \\
    &=\left\langle\nabla_{y} f\left(g(x_{t}), y_{t}\right)-w_{t}, y^{*}\left(x_{t}\right)-y_{t}\right\rangle+\left\langle\nabla_{y} f\left(g(x_{t}), y_{t}\right)-w_{t}, y_{t}-\tilde{y}_{t+1}\right\rangle \\
    &\leq \frac{1}{\mu}\left\|\nabla_{y} f\left(g(x_{t}), y_{t}\right)-w_{t}\right\|^{2}+\frac{\mu}{4}\left\|y^{*}\left(x_{t}\right)-y_{t}\right\|^{2}+\frac{1}{\mu}\left\|\nabla_{y} f\left(g(x_{t}), y_{t}\right)-w_{t}\right\|^{2}+\frac{\mu}{4}\left\|y_{t}-\tilde{y}_{t+1}\right\|^{2} \\
    &=\frac{2}{\mu}\left\|\nabla_{y} f\left(g(x_{t}), y_{t}\right)-w_{t}\right\|^{2}+\frac{\mu}{4}\left\|y^{*}\left(x_{t}\right)-y_{t}\right\|^{2}+\frac{\mu}{4}\left\|y_{t}-\tilde{y}_{t+1}\right\|^{2}.
\end{aligned}
\end{equation}
By plugging the inequalities \eqref{eq12} and \eqref{eq11} into \eqref{eq9}, we have:
\begin{equation}
\begin{aligned}
    &\frac{b_{t}}{2 \eta_{t} \lambda}\left\|y_{t+1}-y^{*}\left(x_{t}\right)\right\|^{2} \\
    &\leq \left(\frac{b_{t}}{2 \eta_{t} \lambda}-\frac{\mu}{4}\right)\left\|y_{t}-y^{*}\left(x_{t}\right)\right\|^{2}+\left(\frac{\eta_{t} b_{t}}{2 \lambda}-\frac{b_{t}}{\lambda}+\frac{\mu}{4}+\frac{L}{2}\right)\left\|\tilde{y}_{t+1}-y_{t}\right\|^{2} +\frac{2}{\mu}\left\|\nabla_{y} f\left(g(x_{t}), y_{t}\right)-w_{t}\right\|^{2} \\
    & \leq  \left(\frac{b_{t}}{2 \eta_{t} \lambda}-\frac{\mu}{4}\right)\left\|y_{t}-y^{*}\left(x_{t}\right)\right\|^{2}+\left(\frac{3 L}{4}-\frac{b_{t}}{2 \lambda}\right)\left\|\tilde{y}_{t+1}-y_{t}\right\|^{2}+\frac{2}{\mu}\left\|\nabla_{y} f\left(g(x_{t}), y_{t}\right)-w_{t}\right\|^{2} \\
    &=  \left(\frac{b_{t}}{2 \eta_{t} \lambda}-\frac{\mu}{4}\right)\left\|y_{t}-y^{*}\left(x_{t}\right)\right\|^{2}-\left(\frac{3 b_{t}}{8 \lambda}+\frac{b_{t}}{8 \lambda}-\frac{3 L_{f}}{4}\right)\left\|\tilde{y}_{t+1}-y_{t}\right\|^{2} +\frac{2}{\mu}\left\|\nabla_{y} f\left(g(x_{t}), y_{t}\right)-w_{t}\right\|^{2} \\
    &\leq  \left(\frac{b_{t}}{2 \eta_{t} \lambda}-\frac{\mu}{4}\right)\left\|y_{t}-y^{*}\left(x_{t}\right)\right\|^{2}-\frac{3 b_{t}}{8 \lambda}\left\|\tilde{y}_{t+1}-y_{t}\right\|^{2}+\frac{2}{\mu}\left\|\nabla_{y} f\left(g(x_{t}), y_{t}\right)-w_{t}\right\|^{2},
\end{aligned}    
\end{equation}
where the second inequality holds by $L \geq \mu$ and $0  < \eta_t \leq 1$, and the last inequality is due to $0 < \lambda \leq \frac{b}{6L} \leq \frac{b_t}{6L}$, it implies that:
\begin{equation}\label{eq7}
    \|y_{t+1}-y^*(x_{t})\|^2\leq(1-\dfrac{\eta_t\mu\lambda}{2b_t})\|y_t-y^*(x_t)\|^2-\frac{3\eta_t}{4}\|\Tilde{y}_{t+1}-y_t\|^2+\frac{4\eta_t\lambda}{\mu b_t}\|\nabla_y f(g(x_{t}),y_t)-w_t\|^2.
\end{equation}
Next, we considering the term $\|y_{t+1}-y^{*}(x_{t+1})\|^2$:
\begin{equation}\label{eq8}
\begin{aligned}
    &\left\|y_{t+1}-y^{*}\left(x_{t+1}\right)\right\|^{2}\\
    & =\left\|y_{t+1}-y^{*}\left(x_{t}\right)+y^{*}\left(x_{t}\right)-y^{*}\left(x_{t+1}\right)\right\|^{2} \\
    & =\left\|y_{t+1}-y^{*}\left(x_{t}\right)\right\|^{2}+2\left\langle y_{t+1}-y^{*}\left(x_{t}\right), y^{*}\left(x_{t}\right)-y^{*}\left(x_{t+1}\right)\right\rangle+\left\|y^{*}\left(x_{t}\right)-y^{*}\left(x_{t+1}\right)\right\|^{2} \\
    & \leq\left(1+\frac{\eta_{t} \mu \lambda}{4 b_{t}}\right)\left\|y_{t+1}-y^{*}\left(x_{t}\right)\right\|^{2}+\left(1+\frac{4 b_{t}}{\eta_{t} \mu \lambda}\right)\left\|y^{*}\left(x_{t}\right)-y^{*}\left(x_{t+1}\right)\right\|^{2} \\
    & \leq\left(1+\frac{\eta_{t} \mu \lambda}{4 b_{t}}\right)\left\|y_{t+1}-y^{*}\left(x_{t}\right)\right\|^{2}+\left(1+\frac{4 b_{t}}{\eta_{t} \mu \lambda}\right) \kappa^{2}\left\|x_{t}-x_{t+1}\right\|^{2},
\end{aligned}
\end{equation}
where the first inequality holds by Cauchy-Schwarz inequality and Young’s inequality.
\\
Combine the Eq.(\ref{eq7}) and Eq.(\ref{eq8}), since $0 < \eta_t \leq 1$, $0 < \lambda \leq \frac{b_t}{6L}$, we have $\lambda \leq \frac{b_t}{6L} \leq \frac{1}{6 \mu}$, and $\eta_t \leq 1 \leq \frac{b_t}{6\mu \lambda}$, we have:
\begin{equation}
\begin{aligned}
    &\left\|y_{t+1}-y^{*}\left(x_{t+1}\right)\right\|^{2}\\
    &\leq  \left(1+\frac{\eta_{t} \mu \lambda}{4 b_{t}}\right)\left(1-\frac{\eta_{t} \mu \lambda}{2 b_{t}}\right)\left\|y_{t}-y^{*}\left(x_{t}\right)\right\|^{2}-\left(1+\frac{\eta_{t} \mu \lambda}{4 b_{t}}\right) \frac{3 \eta_{t}}{4}\left\|\tilde{y}_{t+1}-y_{t}\right\|^{2} \\
    & +\left(1+\frac{\eta_{t} \mu \lambda}{4 b_{t}}\right) \frac{4 \eta_{t} \lambda}{\mu b_{t}}\left\|\nabla_{y} f\left(g(x_{t}), y_{t}\right)-w_{t}\right\|^{2}+\left(1+\frac{4 b_{t}}{\eta_{t} \mu \lambda}\right) \kappa^{2}\left\|x_{t}-x_{t+1}\right\|^{2}\\
    & \leq (1-\frac{\eta_t \mu \lambda}{4b_t})\left\|y_{t}-y^{*}\left(x_{t}\right)\right\|^{2}-\frac{3\eta_t}{4}\left\|\tilde{y}_{t+1}-y_{t}\right\|^{2}+\frac{25\eta_t\lambda}{6\mu b_t}\|\nabla_{y} f\left(g(x_{t}), y_{t}\right)-w_{t}\|^{2}+\frac{25\kappa^2b_t}{6\eta_t \mu \lambda}\|x_{t}-x_{t+1}\|^{2}\\
    &\leq(1-\frac{\eta_t \mu \lambda}{4b_t})\left\|y_{t}-y^{*}\left(x_{t}\right)\right\|^{2}-\frac{3\eta_t}{4}\left\|\tilde{y}_{t+1}-y_{t}\right\|^{2}+\frac{25\eta_t\lambda}{6\mu b_t}\|\nabla_{y} f\left(g(x_{t}), y_{t}\right)-w_{t}\|^{2}+\frac{25\eta_t\kappa^2b_t}{6 \mu \lambda}\|\Tilde{x}_{t+1}-x_{t}\|^{2}.
\end{aligned}
\end{equation}
This completes the proof.
\end{proof}

\begin{lemma}\label{lemmau12a}
Given Assumptions ~\ref{Ass:Smooth} -\ref{Ass:BoundedVariance}, for Algorithm~\ref{alg:AOA}, we have:
\begin{equation}
    \mathbb{E}[\|u_{t+1}-u_{t}\|^2]\leq2\beta_{t+1}^2\mathbb{E}[\|g(x_{t})-u_{t}\|^2]+2\beta_{t+1}^2\sigma_{g}^{2}+2L_{f}^{2}\eta_{t}^2\mathbb{E}[\|\Tilde{x}_{t+1}-x_t\|^2].\nonumber
\end{equation}    
\end{lemma}

\begin{proof}
\begin{equation}
\begin{aligned}
    &\mathbb{E}[\|u_{t+1}-u_{t}\|^2]\\
    &=\mathbb{E}[\|(1-\beta_{t+1})(u_{t}-g(x_{t};\xi_{t+1}))+g(x_{t+1};\xi_{t+1})-u_{t}\|^2] \\
    &=\mathbb{E}[\|\beta_{t+1}(g(x_{t})-u_{t})+(g(x_{t+1};\xi_{t+1})-g(x_{t};\xi_{t+1}))+\beta_{t+1}(g(x_{t};\xi_{t+1})-g(x_{t}))\|^2]\\
    &\leq\mathbb{E}[\|\beta_{t+1}(g(x_{t})-u_{t})+\beta_{t+1}(g(x_{t};\xi_{t+1})-g(x_{t}))\|^2]+2L_{f}^{2}\mathbb{E}[\|x_{t+1}-x_{t}\|^2]\\
    &\leq2\beta_{t+1}^2\mathbb{E}[\|(g(x_{t})-u_{t})\|^2]+2\beta_{t+1}^2\sigma_{g}^{2}+2L_{f}^{2}\eta_{t}^{2}\mathbb{E}[\|\Tilde{x}_{t+1}-x_t\|^2],
\end{aligned}
\end{equation}    
where the first inequality holds by Assumption \ref{Ass:Smooth}.
\end{proof}

\begin{lemma}
Given Assumptions~\ref{Ass:Smooth}-\ref{Ass:BoundedVariance}, based on the Algorithm~\ref{alg:AOA}, we have:
\begin{equation}
    \mathbb{E}[\|u_{t+1}-g(x_{t+1})\|^2] \leq(1-\beta_{t+1})\mathbb{E}[\|u_t-g(x_t)\|^2]+2\beta_{t+1}^2\sigma_{g}^{2}+2L_{f}^{2}\eta_t^2\mathbb{E}[\|\Tilde{x}_{t+1}-x_t\|^2].\nonumber
\end{equation}    
\end{lemma}

\begin{proof}
\begin{equation}
\begin{aligned}
    \mathbb{E}[\|u_{t+1}-g(x_{t+1})\|^2]
    &=\mathbb{E}[\|(1-\beta_{t+1})(u_{t}-g(x_{t};\xi_{t+1}))+g(x_{t+1};\xi_{t+1})-g(x_t)\|^2]\\
    &=\mathbb{E}[\|(1-\beta_{t+1})(u_{t}-g(x_{t}))+\beta_{t+1}(g(x_{t};\xi_{t+1})-g(x_{t}))\\
    &\quad+(g(x_{t})-g(x_{t+1})-(g(x_{t};\xi_{t+1})-g(x_{t+1};\xi_{t+1})))\|^2]\\
    &\leq(1-\beta_{t+1})\mathbb{E}[\|u_{t}-g(x_{t})\|^2]+2\beta_{t+1}^{2}\mathbb{E}[\|g(x_{t};\xi_{t+1}))-g(x_{t})\|^2]\\
    &\quad+2\mathbb{E}[\|g(x_{t+1};\xi_{t+1}))-g(x_{t};\xi_{t+1})\|^2]\\
    &\leq(1-\beta_{t+1})\mathbb{E}[\|u_{t}-g(x_{t})\|^2]+2\beta_{t+1}^2\sigma_{g}^{2}+2L_{f}^{2}\eta_{t}^{2}\mathbb{E}[\|\Tilde{x}_{t+1}-x_t\|^2].
\end{aligned}
\end{equation}    
This completes the proof.
\end{proof}

\begin{lemma}
Given Assumptions~\ref{Ass:Smooth}-\ref{Ass:BoundedVariance}, in Algorithm~\ref{alg:AOA}, based on lemma~\ref{lemmau12a}, we have:
\begin{equation}
\begin{aligned}
    &\mathbb{E}[\|v''_{t+1}-\nabla_gf(u_{t+1},y_{t+1})\|^2]\\
    &\leq(1-\beta_{t+1}^{1})\mathbb{E}[\|v''_{t}-\nabla_gf(u_{t},y_t)\|^2]+2\beta_{t+1}^{2}(\sigma_{f}^{2}+2L^2\sigma_g^2)+2L^2\eta_{t}^{2}\mathbb{E}[\|\Tilde{y}_{t+1}-y_{t}\|^2]\\
    &\quad+4L^2L_{f}^{2}\eta_{t}^{2}\mathbb{E}[\|\Tilde{x}_{t+1}-x_t\|^2]+4L^2\beta_{t+1}^{2}\mathbb{E}[\|u_{t}-g(x_{t})\|^2].\nonumber
\end{aligned}
\end{equation}    
\end{lemma}

\begin{proof}
\begin{equation}
\begin{aligned}
    &\mathbb{E}[\|v''_{t+1}-\nabla_gf(u_{t+1},y_{t+1})\|^2]\\
    &=\mathbb{E}[\|(1-\beta_{t+1})(v''_{t}-\nabla_gf(u_{t},y_t;\zeta_{t+1}))+\nabla_gf(u_{t+1},y_{t+1};\zeta_{t+1})-\nabla_gf(u_{t+1},y_{t+1})\|^2]\\
    &=\mathbb{E}[\|(1-\beta_{t+1})(v''_{t}-\nabla_gf(u_{t},y_t))+\beta_{t+1}(\nabla_gf(u_{t},y_t;\zeta_{t+1})-\nabla_gf(u_{t},y_t))\\
    &\quad+\nabla_gf(u_{t},y_t)-\nabla_gf(u_{t+1},y_{t+1})-(\nabla_gf(u_{t},y_t;\zeta_{t+1})-\nabla_gf(u_{t+1},y_{t+1};\zeta_{t+1}))\\
    &\leq(1-\beta_{t+1})\mathbb{E}[\|v''_{t}-\nabla_gf(u_{t},y_t)\|^2]+2L^2(\mathbb{E}[\|y_{t+1}-y_t\|^2]+\mathbb{E}[\|u_{t+1}-u_{t}\|^2])+2\beta_{t+1}^{2}\sigma_{f}^{2}\\
    &\leq(1-\beta_{t+1}^{1})\mathbb{E}[\|v''_{t}-\nabla_gf(u_{t},y_t)\|^2]+2\beta_{t+1}^{2}(\sigma_{f}^{2}+2L^2\sigma_g^2)+2L^2\eta_{t}^{2}\mathbb{E}[\|\Tilde{y}_{t+1}-y_{t}\|^2]\\
    &\quad+4L^2L_{f}^{2}\eta_{t}^{2}\mathbb{E}[\|\Tilde{x}_{t+1}-x_t\|^2]+4L^2\beta_{t+1}^{2}\mathbb{E}[\|u_{t}-g(x_{t})\|^2],
\end{aligned}
\end{equation}    
where the first inequality holds by Assumption \ref{Ass:Smooth} and the last inequality holds by Lemma \ref{lemmau12a}.
\end{proof}

\begin{lemma}\label{lemma:w_t-variance-a}
Given Assumptions \ref{Ass:Smooth} - \ref{Ass:StrongConcave}, for Algorithm \ref{alg:AOA}, based on lemma \ref{lemmau12}, we can obtain the following:
\begin{equation}
\begin{aligned}
    &\mathbb{E}[\|w_{t+1}-\nabla_yf(g(x_{t+1}),y_{t+1})\|^2]\\
    &\leq(1-\alpha_{t+1})\mathbb{E}[\|w_{t}-\nabla_yf(g(x_{t}),y_{t})\|^2]+4\alpha_{t+1}^{2}\sigma_{f}^{2}+8L^{2}\eta_{t}^{2}\mathbb{E}[\|\tilde{y}_{t+1}-y_t\|^2]\\
    &\quad+(4L^2\alpha_{t+1}^2+8L^2\beta_{t+1}^2)\mathbb{E}[\|u_{t}-g(x_{t})\|^2]+8L^2\beta_{t+1}^{2}\sigma_{g}^{2}+(8L_{f}^{2}L^2\eta_{t}^{2}+4L^2C_g^2\eta_t^2)\mathbb{E}[\|\tilde{x}_{t+1}-x_t\|^2].\nonumber
\end{aligned}    
\end{equation}
\end{lemma}

\begin{proof}
\begin{equation}
\begin{aligned}
    &\mathbb{E}[\|w_{t+1}-\nabla_yf(g(x_{t+1}),y_{t+1})\|^2]\\
    & = \mathbb{E}[\|(1-\alpha_{t+1})(w_t-\nabla_yf(u_t,y_t;\zeta_{t+1}))+\nabla_yf(u_{t+1},y_{t+1};\zeta_{t+1})-\nabla_yf(g(x_{t+1}),y_{t+1})\|^2] \\
    & = \mathbb{E}[\|(1-\alpha_{t+1})(w_t-\nabla_yf(g(x_t),y_t))+\alpha_{t+1}(\nabla_yf(u_t,y_t;\zeta_{t+1})-\nabla_yf(g(x_t),y_t))\\
    &\quad+\nabla_yf(g(x_t),y_t)-\nabla_yf(g(x_{t+1}),y_{t+1})-(\nabla_yf(u_t,y_t;\zeta_{t+1})-\nabla_yf(u_{t+1},y_{t+1};\zeta_{t+1}))]\\
    & \leq (1-\alpha_{t+1})\mathbb{E}[\|w_t-\nabla_yf(g(x_t),y_t)\|^2]+2\alpha_{t+1}^2\underbrace{\mathbb{E}[\|\nabla_yf(u_t,y_t;\zeta_{t+1})-\nabla_yf(g(x_t),y_t)\|^2]}\limits_{Q_1}\\
    &\quad+4L^2C_g^2\eta_t^2\mathbb{E}[\|\tilde{x}_{t+1}-x_t\|^2]+8L^2\eta_t^2\mathbb{E}[\|\tilde{y}_{t+1}-y_t\|^2]+4L^2\mathbb{E}[\|u_{t+1}-u_t\|^2].
\end{aligned}
\end{equation}
Next, we bound the term $Q_1$:
\begin{equation}
\begin{aligned}
    Q_1 &= \mathbb{E}[\|\nabla_yf(u_t,y_t;\zeta_{t+1})-\nabla_yf(g(x_t),y_t)\|^2]\\
    & =\mathbb{E}[\|\nabla_yf(u_t,y_t;\zeta_{t+1})-\nabla_yf(u_t,y_t)+\nabla_yf(u_t,y_t)-\nabla_yf(g(x_t),y_t)\|^2]\\
    & \leq 2\sigma_f^2+2L^2\mathbb{E}[\|u_t-g(x_t)\|^2].
\end{aligned}
\end{equation}
Then, according to Lemma \ref{lemmau12}, we can conclude that:
\begin{equation}
\begin{aligned}
    &\mathbb{E}[\|w_{t+1}-\nabla_yf(g(x_{t+1}),y_{t+1})\|^2]\\
    &\leq(1-\alpha_{t+1})\mathbb{E}[\|w_{t}-\nabla_yf(g(x_{t}),y_{t})\|^2]+4\alpha_{t+1}^{2}\sigma_{f}^{2}+8L^{2}\eta_{t}^{2}\mathbb{E}[\|\tilde{y}_{t+1}-y_t\|^2]\\
    &\quad+(4L^2\alpha_{t+1}^2+8L^2\beta_{t+1}^2)\mathbb{E}[\|u_{t}-g(x_{t})\|^2]+8L^2\beta_{t+1}^{2}\sigma_{g}^{2}+(8L_{f}^{2}L^2\eta_{t}^{2}+4L^2C_g^2\eta_t^2)\mathbb{E}[\|\tilde{x}_{t+1}-x_t\|^2].
\end{aligned}
\end{equation}
This completes the proof.
\end{proof}
Similar to the proof of Lemma \ref{lemma:u_t_variance}, Lemma \ref{Lemma:v'_t_variance}, Lemma \ref{Lemma:v''_t-variance} and Lemma \ref{lemma:w_t-variance}, we can conclude that the variance between the estimators and their true values can be controlled by simple strategy, i.e., proper step size of $x$ and the proper value of $\beta, \alpha$.
\\
Now, we come into the proof of Theorem \ref{theorem ago2}. Define the potential function, for any $t \geq 1$:
\begin{equation}
    P_t = \mathbb{E}[\Phi(x_t)] + A \mathcal{T}_t +\frac{1}{\eta_{t-1}}(\mathcal{J}_t+\mathcal{H}_t+\mathcal{X}_t+\mathcal{K}_t),
\end{equation}
where $A = 15L^2+\frac{10C_g^2L^2\hat{b}}{\rho\lambda\mu}$. Denote that $\mathcal{J}_t: =\mathbb{E}[\|u_{t}-g(x_{t})\|^2]$, $\mathcal{H}_t: =\mathbb{E}[\|v'_{t}-\nabla g(x_{t})\|^2]$, $\mathcal{X}_t: =\mathbb{E}[\|w_{t}-\nabla_yf(g(x_{t}),y_{t})\|^2]$, $\mathcal{K}_t: =\mathbb{E}[\|v''_{t}-\nabla_gf(u_{t},y_{t})\|^2]$ and $\mathcal{T}_t : =\mathbb{E}[\|y^*(x_{t})-y_{t}\|^2]$. Then according to the above lemmas, we have:
\begin{equation}
\begin{aligned}
    P_{t+1}-P_{t}
    & \leq \frac{2\gamma C_g^2 L^2\eta_t}{\rho}\mathcal{T}_t+\frac{2\gamma \eta_t}{\rho}(2C_g^2L^2\mathcal{J}_t+4C_f^2\mathcal{H}_t+4C_g^2\mathcal{K}_t)-\frac{\rho \eta_t}{2\gamma}\mathbb{E}[\|\Tilde{x}_{t+1}-x_t\|^2]\\
    & \quad+ \frac{1}{\eta_t}(\mathcal{J}_{t+1}+\mathcal{H}_{t+1}+\mathcal{X}_{t+1}+\mathcal{K}_{t+1})-\frac{1}{\eta_{t_1}}(\mathcal{J}_{t}+\mathcal{H}_{t}+\mathcal{X}_{t}+\mathcal{K}_{t})\\
    &\quad+A(-\frac{\lambda \mu \eta_t}{4b_t}\mathcal{T}_t-\frac{3\eta_t}{4}\mathbb{E}[\|\Tilde{y}_{t+1}-y_t\|^2]+\frac{25\lambda \eta_t}{6\mu b_t}\mathcal{X}_t+\frac{25\kappa^2b_t\eta_t}{6\mu\lambda}\mathbb{E}[\|\Tilde{x}_{t+1}-x_t\|^2])\\
    &\leq (\frac{2\gamma C_g^2 L^2\eta_t}{\rho}-\frac{\mu\lambda A\eta_t}{4b_t})\mathcal{T}_t+(\frac{4\gamma C_g^2L^2\eta_t}{\rho}+\frac{(1-\beta_{t+1})}{\eta_t}+\frac{10L^2\beta_{t+1}^2+4L^2\alpha_{t+1}^2}{\eta_t}-\frac{1}{\eta_{t-1}})\mathcal{J}_t\\
    &\quad+(\frac{8C_f^2 \gamma \eta_t}{\rho}+\frac{1-\beta_{t+1}}{\eta_t}-\frac{1}{\eta_{t-1}})\mathcal{H}_t+(\frac{8C_g^2 \gamma \eta_t}{\rho}+\frac{1-\beta_{t+1}}{\eta_t}-\frac{1}{\eta_{t-1}})\mathcal{K}_t\\
    &\quad+(\frac{25\lambda A \eta_t }{6\mu b_t}+\frac{1-\beta_{t+1}}{\eta_t}-\frac{1}{\eta_{t-1}})\mathcal{X}_t+(-\frac{3A\eta_t}{4}+10L^2\eta_t)\mathbb{E}[\|\Tilde{y}_{t+1}-y_t\|^2]\\
    &\quad+(\frac{25\kappa^2 A b_t\eta_t}{6\mu \lambda}-\frac{\rho \eta_t}{2\gamma}+(2L_f^2+2L_g^2+12L^2L_f^2+4L^2C_g^2)\eta_t)\mathbb{E}[\|\Tilde{x}_{t+1}-x_t\|^2]\\
    &\quad+\frac{2\beta_{t+1}^2}{\eta_t}(\sigma^2_g+\sigma^{2}_{g^{'}}+\sigma^2_f+6L^2\sigma^2_g) +\frac{4\alpha_{t+1}^2\sigma_f^2}{\eta_t}.
\end{aligned}
\end{equation}
\\
By setting $m > \max \{\frac{8L_{\Phi}^3\gamma^3}{\rho^3},(10L^2c_1^2+4L^2c_2^2)^3, c_1^3, c_2^3\}$,  
 $\beta_{t+1}=c_1\eta_t^2\leq c_1\eta_t < 1$, $\alpha_{t+1}=c_2\eta_t^2\leq c_2\eta_t< 1 $ and 
 $\eta_t=\frac{1}{(t+m)^{1/3}}$ , according to \eqref{Eq:eta_t} we have:
\begin{equation}
\begin{aligned}
    \frac{1}{\eta_t}-\frac{1}{\eta_{t-1}}\leq\frac{2^{2/3}}{3}\eta_t^2\leq\frac{2}{3}\eta_t^2.
\end{aligned}    
\end{equation}
\\
Let $c_1 \geq 2+\frac{5\gamma(2C_f^2+2C_g^2+C_g^2L^2)}{\rho}$ and $c_2 \geq \frac{2}{3} + \frac{25\lambda A}{6\mu b}$ we can get:
\begin{equation}
\begin{aligned}
    \frac{8C_f^2 \gamma \eta_t}{\rho}+\frac{1-\beta_{t+1}}{\eta_t}-\frac{1}{\eta_{t-1}}& \leq \frac{8C_f^2 \gamma \eta_t}{\rho}+ \frac{2\eta_t}{3} -c_1\eta_t \leq  -\frac{2C_f^2\gamma\eta_t}{\rho},
\end{aligned}
\end{equation}
and
\begin{equation}
\begin{aligned}
    \frac{8C_g^2 \gamma \eta_t}{\rho}+\frac{1-\beta_{t+1}}{\eta_t}-\frac{1}{\eta_{t-1}}& \leq \frac{8C_g^2 \gamma \eta_t}{\rho}+ \frac{2\eta_t}{3} -c_1\eta_t \leq  -\frac{2C_g^2\gamma\eta_t}{\rho}.
\end{aligned}
\end{equation}
Then setting $\gamma  \leq \frac{\rho}{4\sqrt{B_1^2+\rho B_2}}$, where $B_1 = \frac{50C_g^2\kappa^4\hat{b}}{\lambda^2}$, $B_2 = \frac{70\kappa^3L}{\lambda}+2L_f^2+2L_g^2+12L^2L_f^2+4L^2C_g^2$, we can get:
\begin{equation}
\begin{aligned}
P_{t+1}-P_t&\leq -\frac{C_g^2L^2 \gamma}{2\rho}\eta_t\mathcal{T}_t-\frac{C_g^2L^2\gamma}{\rho}\eta_t\mathcal{J}_t-\frac{2C_f^2\gamma}{\rho}\eta_t\mathcal{H}_t-\frac{2C_g^2\gamma}{\rho}\eta_t\mathcal{K}_t-\frac{\rho}{4\gamma}\eta_t\mathbb{E}[\|\Tilde{x}_{t+1}-x_{t}\|^2]\\
&\quad+2c_1^2\eta_t^3(\sigma^2_g+\sigma^{2}_{g^{'}}+\sigma^2_f+6L^2\sigma^2_g) +4c_2^2\eta_t^3\sigma_f^2.
\end{aligned}    
\end{equation}
\\
Taking average over $t = 1,2,\cdots, T$ on both sides of above inequality, we have:
\begin{equation}
\begin{aligned}
    &\frac{1}{T}\sum^T\limits_{t=1}(\frac{C_g^2L^2 \gamma}{2\rho}\eta_t\mathcal{T}_t+\frac{C_g^2L^2\gamma}{\rho}\eta_t\mathcal{J}_t
    +\frac{2C_f^2\gamma}{\rho}\eta_t\mathcal{H}_t+\frac{2C_g^2\gamma}{\rho}\eta_t\mathcal{K}_t+\frac{\rho}{4\gamma}\eta_t\mathbb{E}[\|\Tilde{x}_{t+1}-x_{t}\|^2])\\
    &\leq \frac{P_1-P_{t+1}}{T}+\frac{(2c_1^2(\sigma^2_g+\sigma^{2}_{g^{'}}+\sigma^2_f+6L^2\sigma^2_g) +4c_2^2\sigma_f^2)\ln(m+T)}{T}.
\end{aligned}    
\end{equation}
\\
Since $\eta_t$ is decreasing, we have:
\begin{equation}
\begin{aligned}
    &\frac{1}{T}\sum^T\limits_{t=1}(\frac{C_g^2L^2 \gamma}{2\rho}\mathcal{T}_t+\frac{C_g^2L^2\gamma}{\rho}\mathcal{J}_t
    +\frac{2C_f^2\gamma}{\rho}\mathcal{H}_t+\frac{2C_g^2\gamma}{\rho}\mathcal{K}_t+\frac{\rho}{4\gamma}\mathbb{E}[\|\Tilde{x}_{t+1}-x_{t}\|^2])\\
    &\leq \frac{P_1-P_{t+1}}{T}(m+T)^{1/3}+\frac{(2c_1^2(\sigma^2_g+\sigma^{2}_{g^{'}}+\sigma^2_f+6L^2\sigma^2_g) +4c_2^2\sigma_f^2)\ln(m+T)}{T}(m+T)^{1/3}.
\end{aligned}    
\end{equation}
\\
Denote that $M = \frac{\Phi(x_1)-\Phi_*+ \sigma_g^2+\sigma_{g^{'}}^2+\sigma_f^2 +L^2\sigma_g^2}{\rho } + \frac{(2c_1^2(\sigma^2_g+\sigma^{2}_{g^{'}}+\sigma^2_f+6L^2\sigma^2_g) +4c_2^2\sigma_f^2)\ln(m+T)}{\rho}$, then from the initialization, we can easily get:
\begin{equation}
\begin{aligned}
     &\frac{1}{T}\sum^T\limits_{t=1}(\frac{C_g^2L^2 }{2\rho^2}\mathcal{T}_t+\frac{C_g^2L^2}{\rho^2}\mathcal{J}_t
    +\frac{2C_f^2}{\rho^2}\mathcal{H}_t+\frac{2C_g^2}{\rho^2}\mathcal{K}_t+\frac{1}{4\gamma^2}\mathbb{E}[\|\Tilde{x}_{t+1}-x_{t}\|^2])\\
    &\leq \frac{P_1-P_{t+1}}{\rho \gamma T}(m+T)^{1/3}+\frac{2c_1^21(\sigma_g^2+\sigma_{g^{'}}^{2}+2\sigma_f^2+4L^2\sigma_g^2)\ln(m+T)}{\rho \gamma T}(m+T)^{1/3} \leq M\frac{(m+T)^{1/3}}{\gamma T}.
\end{aligned}    
\end{equation}
According to Jensen’s inequality, we have:
\begin{equation}\label{jensen}
\begin{aligned}
    &\frac{2}{T}\sum^T\limits_{t=1}\mathbb{E}\Big[\frac{C_g L}{\sqrt{2}\rho}\|y^*(x_t)-y_t\|+\frac{C_gL}{\rho}\|u_{t}-g(x_t)\| \\
    &\quad+\|\frac{\sqrt{2}C_f}{\rho}\|v'_{t}-\nabla g(x_t)\|+\frac{\sqrt{2}C_g}{\rho}\|v''_{t}-\nabla_gf(u_{t},y_t)\|+\frac{1}{2\gamma}\|\Tilde{x}_{t+1}-x_t\|\Big] \\
    & = \mathbb{E}\Big[\frac{2C_gL}{\sqrt{2}\rho}\|y^*(x_t)-y_t\|+\frac{2C_gL}{\rho}\|u_{t}-g(x_t)\| \\
    &\quad+\|\frac{2\sqrt{2}C_f}{\rho}\|v'_{t}-\nabla g(x_t)\|+\frac{2\sqrt{2}C_g}{\rho}\|v''_{t}-\nabla_gf(u_{t},y_t)\|+\frac{1}{\gamma}\|\Tilde{x}_{t+1}-x_t\|\Big] \\
    & \leq 2\Big(\frac{5}{T}\sum^T\limits_{t=1}(\frac{C_g^2L^2 }{2\rho^2}\mathcal{T}_t+\frac{C_g^2L^2}{\rho^2}\mathcal{J}_t
    +\frac{2C_f^2}{\rho^2}\mathcal{H}_t+\frac{2C_g^2}{\rho^2}\mathcal{K}_t+\frac{1}{4\gamma^2}\mathbb{E}[\|\Tilde{x}_{t+1}-x_{t}\|^2])\Big)^{1/2} \leq 2\sqrt{\frac{5M(m+T)^{1/3}}{\gamma T}}.
\end{aligned}    
\end{equation}
Besides, $\|A_t\| \geq \rho$, we have:
\begin{equation}
\begin{aligned}
    &\frac{2C_gL}{\sqrt{2}\rho}\|y^*(x_t)-y_t\|+\frac{2C_gL}{\rho}\|u_{t}-g(x_t)\|+\|\frac{2\sqrt{2}C_f}{\rho}\|v'_{t}-\nabla g(x_t)\|\\
    &\quad+\frac{2\sqrt{2}C_g}{\rho}\|v''_{t}-\nabla_gf(u_{t},y_t)\|+\frac{1}{\gamma}\|\Tilde{x}_{t+1}-x_t\| \\
    & \geq \underbrace{\frac{C_gL}{\rho}\|y^*(x_t)-y_t\|}\limits_{T_5}+\underbrace{\frac{C_gL}{\rho}\|u_{t}-g(x_t)\|+\|\frac{C_f}{\rho}\|v'_{t}-\nabla g(x_t)\|+\frac{C_g}{\rho}\|v''_{t}-\nabla_gf(u_{t},y_t)\|} \limits_{T_6}+\underbrace{\frac{1}{\gamma}\|\Tilde{x}_{t+1}-x_t\|}\limits_{T_4}.
\end{aligned}    
\end{equation}
Looking at the term $T_4$, we have:
\begin{equation}
    T_4  = \frac{1}{\gamma}\|\Tilde{x}_{t+1}-x_t\| = \frac{1}{\|A_t\|}\|v_t\|.
\end{equation}
Then for the term $T_5$, we have:
\begin{equation}
\begin{aligned}
    T_5 = \frac{C_gL}{\rho}\|y^*(x_t)-y_t\| &\geq \frac{1}{\rho}\|\nabla_x f(g(x_t),y^*(x_t)) - \nabla_x f(g(x_t),y_t)\|\\
    &= \frac{1}{\rho}\|\nabla \Phi(x_t) - \nabla_x f(g(x_t),y_t)\|.
\end{aligned}    
\end{equation}
Next, for the term $T_6$, we have:
\begin{equation}
\begin{aligned}
    T_6 
    &= \frac{C_gL}{\rho}\|u_{t}-g(x_t)\|+\|\frac{C_f}{\rho}\|v'_{t}-\nabla g(x_t)\|+\frac{C_g}{\rho}\|v''_{t}-\nabla_gf(u_{t},y_t)\| \\
    & \geq \frac{1}{\rho}(\|\nabla g(x_t)\nabla_gf(u_{t},y_t)-\nabla_x f(g(x_t),y_t)\|+\|\nabla g(x_t)\nabla_gf(u_{t},y_t)-v'_{t}\nabla_gf(u_{t},y_t)\|  +\|v'_{t}\nabla_gg(u_{t},y_t)-v'_{t}v''_{t}\|)\\
    &\geq \frac{1}{\rho}(\|\nabla g(x_t)\nabla_gf(u_{t},y_t)-\nabla_x f(g(x_t),y_t)\|+\|\nabla g(x_t)\nabla_gf(u_{t},y_t)-v'_{t}v''_{t}\|) \geq \frac{1}{\rho}\|\nabla_xf(g(x_t),y_t)-v_t\|.
\end{aligned}    
\end{equation}
Then,
\begin{equation}
\begin{aligned}
    &\mathcal{M}_t = \frac{2C_gL}{\sqrt{2}\rho}\|y^*(x_t)-y_t\|+\frac{2C_gL}{\rho}\|u_{t}-g(x_t)\|+\|\frac{2\sqrt{2}C_f}{\rho}\|v'_{t}-\nabla g(x_t)\|\\
    &\quad+\frac{2\sqrt{2}C_g}{\rho}\|v''_{t}-\nabla_gf(u_{t},y_t)\|+\frac{1}{\gamma}\|\Tilde{x}_{t+1}-x_t\| \\
    &\geq \frac{1}{\|A_t\|}\|v_t\|+\frac{1}{\rho}\|\nabla\Phi(x_t) - \nabla_x f(g(x_t),y_t)\|+\frac{1}{\rho}\|\nabla_xf(g(x_t),y_t)-v_t\|\\
    &\geq \frac{1}{\|A_t\|}\|v_t\| + \frac{1}{\rho}\|\nabla\Phi(x_t)-v_t\|\\
    &\geq \frac{1}{\|A_t\|}(\|v_t\|+\|\nabla\Phi(x_t)-v_t\|)\\
    &\geq \frac{1}{\|A_t\|}\|\nabla\Phi(x_t)\|.
\end{aligned}    
\end{equation}
We can get $\|\nabla\Phi(x_t)\| \leq \mathcal{M}_t\|A_t\|$. By using Cauchy-Schwarz inequality, we have:
\begin{equation}
\begin{aligned}
    \frac{1}{T}\sum^T\limits_{t=1}\mathbb{E}[\|\nabla\Phi(x_t)\|] \leq \frac{1}{T}\sum^T\limits_{t=1}\mathbb{E}[\mathcal{M}_t\|A_t\|]\leq\sqrt{\frac{1}{T}\sum^T\limits_{t=1}\mathbb{E}[\mathcal{M}_t^2]}\cdot\sqrt{\frac{1}{T}\sum^T\limits_{t=1}\mathbb{E}[\|A_t\|^2]}.
\end{aligned}    
\end{equation}
Based on \eqref{jensen}, we have:
\begin{equation}
\begin{aligned}
    \frac{1}{T}\sum^T\limits_{t=1}\mathbb{E}[\mathcal{M}_t^2]& \leq  \frac{20M(m+T)^{1/3}}{\gamma T}.
\end{aligned}
\end{equation}
Thus, we have:
\begin{equation}
\begin{aligned}
    \frac{1}{T}\sum^T\limits_{t=1}\mathbb{E}[\|\nabla\Phi(x_t)\|] &\leq \sqrt{\frac{1}{T}\sum^T\limits_{t=1}\mathbb{E}[\|A_t\|^2]} \cdot \frac{2\sqrt{5M(m+T)^{1/3}}}{\sqrt{\gamma T}}\\
    &\leq \sqrt{\frac{1}{T}\sum^T\limits_{t=1}\mathbb{E}[\|A_t\|^2]} \cdot (\frac{2\sqrt{5Mm^{1/3}}}{\sqrt{\gamma T}}+\frac{2\sqrt{5M}}{\sqrt{\gamma}T^{1/3}}).
\end{aligned}    
\end{equation}
Therefore, we get the result in Theorem~\ref{theorem ago2}.\hfill~$\Box$

\section{ADA-NSTORM with the Different Adam-Type Generator}
Adaptive learning rates have been widely used in stochastic optimization problems, with many successful methods proposed such as Adam \cite{kingma2014adam}, AdaBelief \cite{zhuang2020adabelief}, AMSGrad \cite{j.2018on}, and AdaBound \cite{luo2018adaptive}. However, their application in stochastic compositional problems remains less explored. To enable adaptive learning rates, we propose generating adaptive learning rate matrices in different ways. We give ADA-NSTORM with different Adam-type as Algorithm \ref{alg:Ada-all-type}.$\rho$ is an arbitrary positive constant that is greater than 0. It is introduced to prevent the matrices $A_t$ or $B_t$ from containing zeros, which would cause issues with the scoring factor calculations.
\begin{algorithm}[h]
    \caption{Illustration of ADA-NSTORM method with different Adam-type.}
    \label{alg:Ada-all-type}
    \renewcommand{\algorithmicrequire}{\textbf{Initialization:}}
    \renewcommand{\algorithmicensure}{\textbf{Output:}}
    \begin{algorithmic}[1]
        \REQUIRE $\rho > 0.$
        \STATE replace line 13 in Algorithm \ref{alg:AOA} by the following cases:
        \STATE Case 1. AMSGrad \cite{j.2018on}:
            \begin{equation}
            \begin{aligned}
                & a'_t=\tau a'_{t-1}+(1-\tau) v_t^2,~~a_t = \max(a_{t-1}, a'_t),~~A_{t}=\operatorname{diag}\left(\sqrt{a_{t}}+\rho\right);\\
                & b'_t=\tau b'_{t-1}+(1-\tau) w_t^2,~~b_t = \max(b_{t-1}, b'_t),~~B_{t}=\operatorname{diag}(\sqrt{b_{t}}+\rho).
            \end{aligned}
            \end{equation}
        \STATE Case 2. AdaBelief \cite{zhuang2020adabelief}:
            \begin{equation}
            \begin{aligned}
                &a_{t} =\tau a_{t-1}  +(1-\tau)( \nabla_{g}f\left(u_{t}, y_{t} ; \zeta_{t}\right)\cdot v'_{t} - v_t)^{2},~~A_{t}=\operatorname{diag}\left(\sqrt{a_{t}}+\rho\right);\\
                &b_{t} =\tau b_{t-1} +(1-\tau) (\nabla_{y} f\left(u_{t}, y_{t} ; \zeta_{t}\right)-w_t)^2 ,~~B_{t} =\operatorname{diag}(\sqrt{b_t}+\rho).
            \end{aligned}       
            \end{equation}
        \STATE Case 3. AdaBound \cite{luo2018adaptive}:
            \begin{equation}\label{eq:AdaBound}
            \begin{aligned}
                & a'_t=\tau a'_{t-1}+(1-\tau) v_t^2,~~ a_t = \Pi_{[C_l, C_u]}[a'_t], ~~A_{t}=\operatorname{diag}\left(\sqrt{a_{t}}+\rho\right);\\
                & b'_t =\tau b'_{t-1}+(1-\tau) w_t^2,~~b_t =  \Pi_{[C_l, C_u]}[a'_t],~~B_{t}=\operatorname{diag}(\sqrt{b_{t}}+\rho).
            \end{aligned}
            \end{equation}

    \end{algorithmic}
\end{algorithm}

In case 1, we consider using  AMSGrad. AMSGrad enhances Adam optimization by retaining the maximum of all past learning rates $v_t$, denote as $v_{max}$. This maximal learning rate replaces $v_t$ when calculating the current learning rate $\eta_t$. As a result, the learning rate decays more slowly over training compared to Adam, where $v_t$ continually decreases. By preserving larger historical learning rates in $v_{max}$, AMSGrad stabilizes the learning rate at a higher value, avoiding premature convergence. On certain tasks, AMSGrad achieves superior performance to Adam.

In case 2, we consider using AdaBelief. In conventional non-convex optimization, parameter updates for $x$ typically rely solely on noisy gradient value or gradient estimator value. AdaBelief incorporates both the noisy gradients and estimator values when updating x. It increases the update size when the estimator and noisy gradient are in close agreement. But when there is a large gap between the estimator and noisy gradient, AdaBelief slows down the updates. Thus, AdaBelief adapts the update pace based on the alignment between these two information sources. Nonetheless, this method cannot be directly applied to our problem, i.e., we can only obtain the biased estimation of the full gradient. As such, combining with NSTORM, we consider the gap between the inner estimator and the gradient of $g(x)$, i.e., $(\nabla_{g} f\left(u_{t}, y_{t} ; \zeta_{t}\right)\cdot v'_{t} - v_t)$. 

In Case 1 and Case 2 of Algorithm \ref{alg:Ada-all-type}, we can see that making the same assumptions as in Theorem \ref{theorem ago2}, specifically Assumptions \ref{Ass:Smooth}-\ref{Ass:StrongConcave} and \ref{ass ab bound}, leaves the recursion inequalities for the estimation error of the inner and outer functions' values and gradients unchanged. That is, Lemmas \ref{lemma:Phi_a}-\ref{lemma:w_t-variance-a} still hold. Now under Assumptions \ref{Ass:Smooth}-\ref{ass ab bound} for Algorithm \ref{alg:Ada-all-type}, by setting the same parameters as in Theorem \ref{theorem ago2}, we can obtain:
\begin{equation}
    \begin{split}
    &\frac{1}{T}\sum^T\limits_{t=1}\mathbb{E}[\|\nabla\Phi(x_t)\|]\leq \sqrt{\frac{1}{T}\sum^T\limits_{t=1}\mathbb{E}[\|A_t\|^2]} \cdot \left(\frac{2\sqrt{5Mm^{1/3}}}{\sqrt{\gamma T}}+\frac{2\sqrt{5M}}{\sqrt{\gamma}T^{1/3}}\right),
    \end{split}    
\end{equation}
where $M =  (\Phi(x_1)-\Phi_*+ \sigma_g^2+\sigma_{g^{'}}^2+\sigma_f^2 +L^2\sigma_g^2 )/\rho + ((2c_1^2(\sigma^2_g + \sigma^{2}_{g^{'}} + \sigma^2_f + 6L^2\sigma^2_g ) +4c_2^2\sigma_f^2)\ln(m+T))/\rho$ and $\Phi_*$ represents the minimum value of $\Phi(x)$. Therefore, the sample complexity for Algorithm \ref{alg:Ada-all-type} with Case 1 and Case 2 remains $O(\kappa^3/\epsilon^3)$.

Note that this is the first work that introduces AdaBelief into the compositional minimax optimization problem without using a large batch size.

In Case 3, we consider using AdaBound which constrains learning rates within predefined minimum and maximum bounds. In \ref{eq:AdaBound}, where $C_l \leq C_u$, the projection $\Pi$ restricts $a'_t$ to the range $[C_l, C_u]$. We can use an analysis similar to Theorem \ref{theorem ago2} to derive the sample complexity. Furthermore, Theorem \ref{theorem ago2} still applies to Case 3 of Algorithm \ref{alg:Ada-all-type} even without needing Assumption \ref{ass ab bound}. We only require a minor modification - the projection threshold $C_u$  is set equal to $\hat{b}$ from Assumption \ref{ass ab bound}.Specifically, Under Assumption \ref{Ass:Smooth}-\ref{Ass:StrongConcave} for Algorithm \ref{alg:Ada-all-type},  setting $C_u = \hat{b}$, $\eta_t=\frac{1}{(m+t)^{1/3}}$,  $m > \max \{\frac{8L_{\Phi}^3\gamma^3}{\rho^3},(10L^2c_1^2+4L^2c_2^2)^3, c_1^3, c_2^3\}$,  $c_1 \geq 2+\frac{5\gamma(2C_f^2+2C_g^2+C_g^2L^2)}{\rho}$ , $c_2 \geq \frac{2}{3} + \frac{125\lambda L^2}{2\mu b}+\frac{125\gamma C_g^2\kappa^2\hat{b}}{3b}$, $\gamma  \leq \frac{\rho}{4\sqrt{B_1^2+\rho B_2}}$,   where $B_1 = \frac{50C_g^2\kappa^4\hat{b}}{\lambda^2}$, $B_2 = \frac{70\kappa^3L}{\lambda}+2L_f^2+2L_g^2+12L^2L_f^2+4L^2C_g^2$, $\beta_{t+1}=c_1\eta_t^2\leq c_1\eta_t < 1$, $\alpha_{t+1}=c_2\eta_t^2\leq c_2\eta_t < 1$, $0 < \lambda \leq \frac{b}{6L}$, we can obtain 
\begin{equation}
    \begin{split}
    &\frac{1}{T}\sum^T\limits_{t=1}\mathbb{E}[\|\nabla\Phi(x_t)\|]\leq \sqrt{\frac{1}{T}\sum^T\limits_{t=1}\mathbb{E}[\|A_t\|^2]} \cdot \left(\frac{2\sqrt{5Mm^{1/3}}}{\sqrt{\gamma T}}+\frac{2\sqrt{5M}}{\sqrt{\gamma}T^{1/3}}\right),
    \end{split}    
\end{equation}
where $M =  (\Phi(x_1)-\Phi_*+ \sigma_g^2+\sigma_{g^{'}}^2+\sigma_f^2 +L^2\sigma_g^2 )/\rho + ((2c_1^2(\sigma^2_g + \sigma^{2}_{g^{'}} + \sigma^2_f + 6L^2\sigma^2_g ) +4c_2^2\sigma_f^2)\ln(m+T))/\rho$ and $\Phi_*$ represents the minimum value of $\Phi(x)$. Consequently, the sample complexity of Algorithm \ref{alg:Ada-all-type} using Case 3 remains $O(\kappa^3/\epsilon^3)$.

\section{Additional Experiments}
\subsection{Experimental Setup of deep AUC}
\noindent{\bf Dataset description.} Table \ref{tab:dataset} reports the detailed statistics for the different datasets. Note that "Number of image" refers to the number of samples in the original training set. The "Imbalance Ratio" represents the ratio of the number of positive examples to the total number of examples.
\begin{table*}[htb]
\centering
\begin{tabular}{cccc}
\hline
\textbf{Dataset} & \textbf{Number of images} & \textbf{Imbalance Ratio} & \textbf{Number of labels} \\ \hline
CAT\_VS\_DOG \cite{krizhevsky2009learning} & 20,000 & 1\%,5\%,10\%,30\% & 2 \\
CIFAR10 \cite{krizhevsky2009learning}      & 50,000 & 1\%,5\%,10\%,30\% & 2 \\
CIFAR100 \cite{krizhevsky2009learning}    & 50,000 & 1\%,5\%,10\%,30\% & 2 \\
STL10 \cite{coates2011analysis}       & 5,000  & 1\%,5\%,10\%,30\% & 2 \\ \hline
\end{tabular}
\caption{Dataset Description for Classification Tasks}
\label{tab:dataset}
\end{table*}

\noindent{\bf Training configurations.} All benchmark datasets were evaluated using the NVIDIA GTX-3090. The same dataloaders from \cite{yuan2020robust}\footnote{https://libauc.org/}were used for all datasets. Specifically, for the benchmark datasets, a 19k/1k, 45k/5k, 45k/5k, and 4k/1k training/validation split was applied for CatvsDog, CIFAR10, CIFAR100, and STL10, respectively.

\noindent{\bf Loss function.}We use $(\omega,\theta)$ to denote an example. where $\omega \in \mathbb{R}^d$ denotes the input and $\theta \in \mathcal{Y}$ denotes its corresponding label. 
\begin{equation}\label{AUC_loss}
\begin{aligned}
     &\min_{x, a, b} \max_{y \in \Omega} \Theta\left(x-\alpha \nabla L_{\mathrm{AVG}}(x), a, b, y\right)=\frac{1}{n} \sum_{i=1}^{n} \Upsilon\left(x-\alpha \nabla L_{\mathrm{AVG}}(x), a, b, y ; \omega_{i}, \theta_{i}\right).
\end{aligned}
\end{equation}
Note that we define the vector $\bar{x} = (x;a,b)$, $n_+(n_{-})$ are the number of positive (negative) examples, $p = n_+/n$, we rewrite   $\phi(x,a,b,y;\omega_{i}, \theta_{i}) = g_{1}\left(\bar{x} ;\omega_i,  \theta_{i}\right)+\theta g_{2}\left(\bar{x} ;\omega_{i}, \theta_{i}\right)-g_{3}(\theta)$, where the first part $g_{1}\left(\overline{\mathbf{x}} ; \omega_{i}, \theta_{i}\right)=(1-p)\left(f\left(\mathbf{x} ; \omega_{i}\right)-a\right)^{2} \mathbb{I}_{\left[\theta_{i}=1\right]}+p\left(f\left(\mathbf{x} ; \omega_{i}\right)-b\right)^{2} \mathbb{I}_{\left[\theta_{i}=-1\right]} +2 p f\left(\mathbf{x} ; \omega_{i}\right) \mathbb{I}_{\left[\theta_{i}=1\right]}-2(1-p) f\left(\mathbf{x} ; \omega_{i}\right) \mathbb{I}_{\left[\theta_{i}=-1\right]}$, the second part
$g_{2}\left(\overline{\mathbf{x}} ; \omega_{i}, \theta_{i}\right)=2\left(p f\left(\mathbf{x} ; \omega_{i}\right) \mathbb{I}_{\left[\theta_{i}=-1\right]}-(1-p) f\left(\mathbf{x} ; \omega_{i}\right) \mathbb{I}_{\left[\theta_{i}=1\right]}\right)$ and the last part $g_3(\theta) = p(1-p)\theta^2$. It's notable that the inner function $g(\mathbf{x})$ is equal to $\mathbf{x}-\alpha\nabla L_{\mathrm{AVG}}(\mathbf{x})$, and $\nabla g(\mathbf{x})$ involves the Hessian matrix $\nabla^2 L_(\mathbf{x})$, we use the same strategy as \cite{yuan2022compositional}  that simply ignore the second-order term. The loss function $\Theta$ is transformed from a special surrogate loss \cite{gao2015consistency} $\min_{x} \frac{1}{n_{+} n_{-}} \sum_{\theta_{i}=1} \sum_{\theta_{j}=-1}\left(c-\left(f\left(x ; \omega_{i}\right)-f\left(x ; \omega_{j}\right)\right)\right)^{2}$, due to the solution of this surrogate loss is computationally expensive, where $c$ is a margin parameter, e.g., $c= 1$.
\subsection{Additional Experimental Results of Deep AUC}
As shown in Figure \ref{fig:lr_m_0.1_all_dataset}, we tested the effect of changing the hyperparameter $m$ from 50 to 5000 for NSOTRM on different datasets. We found that varying the value of $m$ had a slight effect on the results. Similarly, we tested the effect of changing the step size ratio of $x$ and $y$, specifically the hyperparameter $\gamma$ as shown in Figure \ref{fig:gamma_0.1_all_dataset}, on different test sets. We found that the best results consistently occurred when the step size for $x$ was less than the step size for $y$. In Figure \ref{fig:upper_all}, we can see that changing the upper bound from 50 to 5000 had little effect on the results across the three datasets. This provides stronger evidence for the moderation implied by Assumption \ref{ass ab bound}. Additionally, in Figure \ref{fig:tau_0.1_all_dataset}, changing the value of $\tau$ from 0.1 to 0.9 leads to only small changes in the results. This demonstrates the robustness of our algorithm.
\subsection{Additional Ablation Studies of Deep AUC}
\begin{figure*}[t!]
    \centering
    \includegraphics[width=\textwidth]{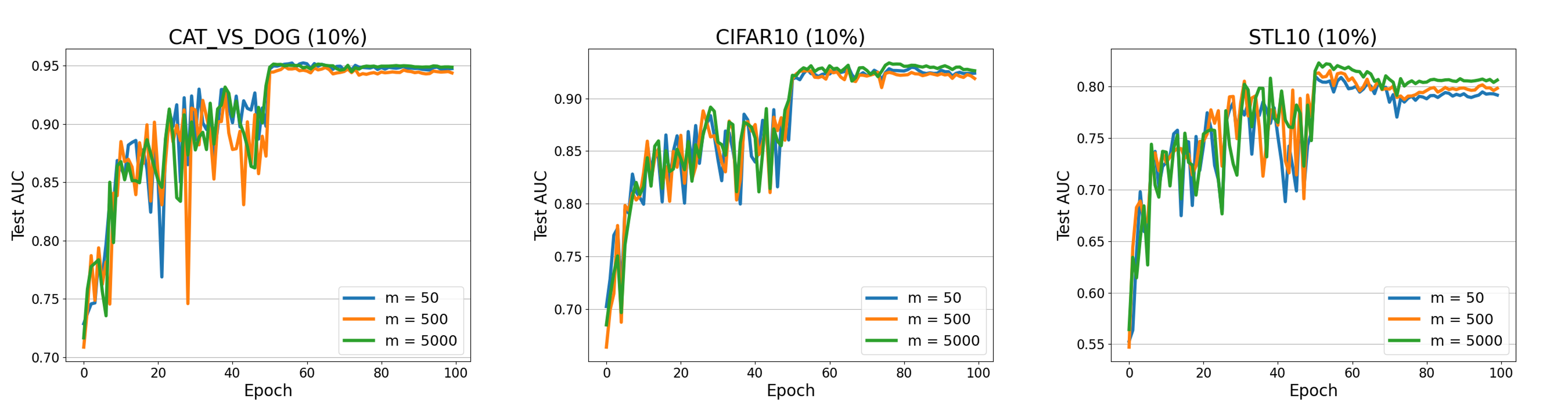}
    \caption{Different $m$ on four Datasets}
    \label{fig:lr_m_0.1_all_dataset}
\end{figure*}

\begin{figure*}[t!]
    \centering
    \includegraphics[width=\textwidth]{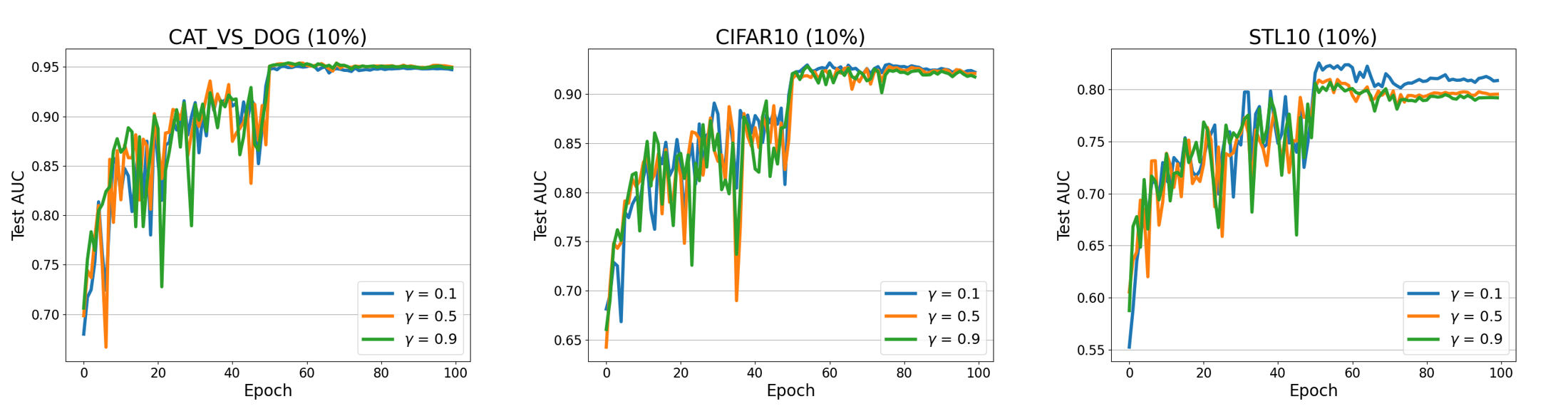}
    \caption{Different $\gamma$ on four Datasets}
    \label{fig:gamma_0.1_all_dataset}
\end{figure*}

\begin{figure*}[t!]
    \centering
    \includegraphics[width=\textwidth]{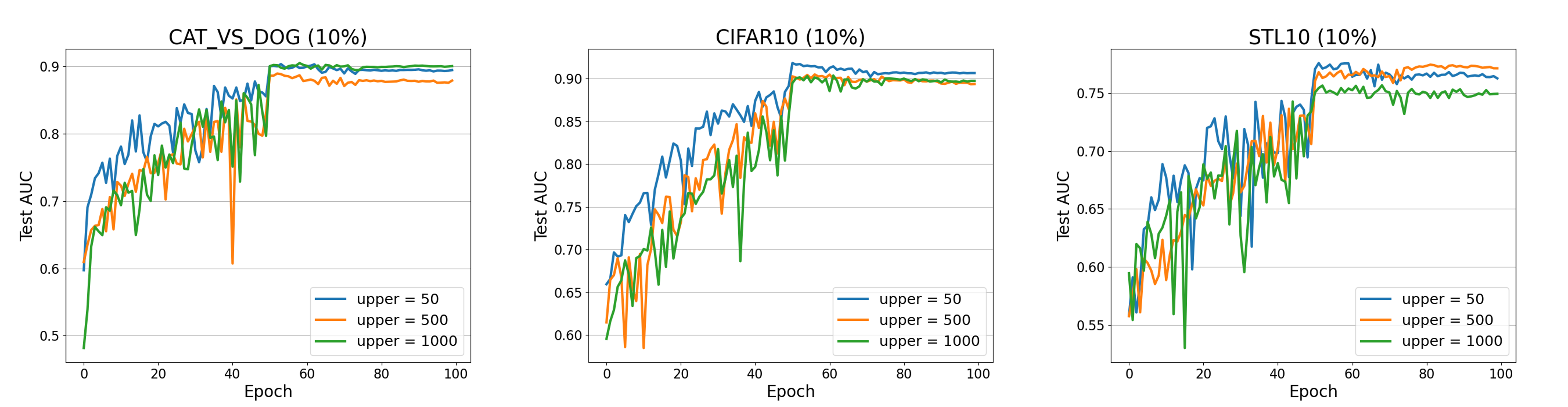}
    \caption{Different upper bound on four Datasets}
    \label{fig:upper_all}
\end{figure*}

\begin{figure*}[t!]
    \centering
    \includegraphics[width=\textwidth]{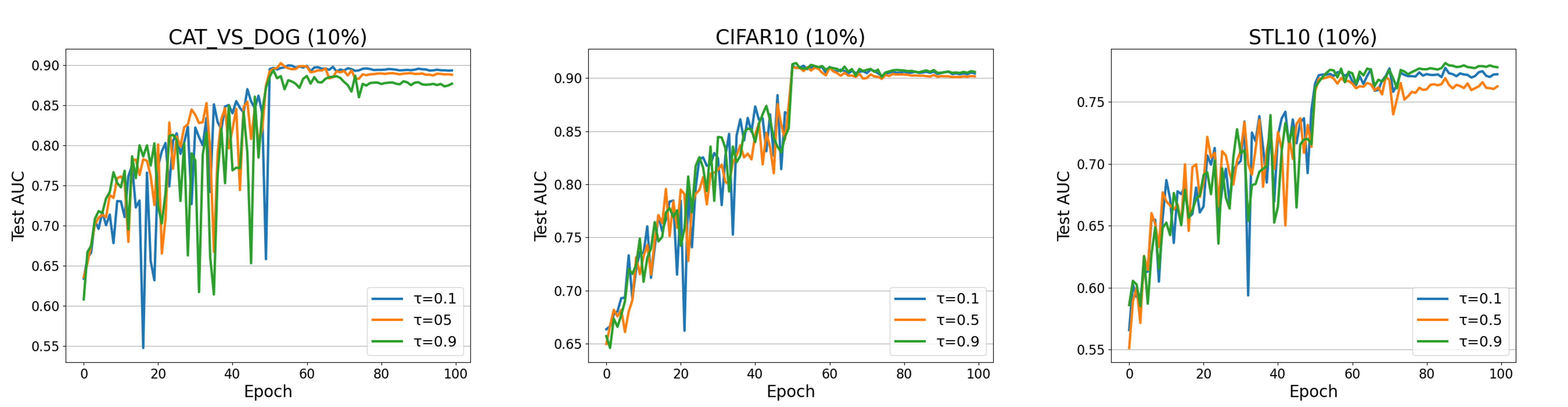}
    \caption{Different $\tau$ on four Datasets}
    \label{fig:tau_0.1_all_dataset}
\end{figure*}

\subsection{Risk-Averse Portfolio Optimization}
We then consider the risk-averse portfolio optimization problem. In this problem, we have $D$ assets to invest during each iteration $\{1,\dots, T\}$, and $r_t \in \mathbb{R}^D$ represents the payoff of $D$ assets in iteration $t$. Our objective is to simultaneously maximize the return on the investment and minimize the associated risk. One commonly used formulation for this problem is the mean-deviation risk-averse optimization \cite{shapiro2021lectures}, where the risk is measured using the standard deviation. This mean-deviation model is widely employed in practical settings and serves as a common choice for conducting experiments in compositional optimization \cite{zhang2021multilevel}. The problem can be formulated as follows:
\begin{equation}
    \min_{x\in\mathcal{X}}\max_{y \in \mathcal{Y}}\frac{1}{D}\sum_{d=1}^{D}y_d \left(-\sum_{t=1}^{T}\langle r_t ,x\rangle + \lambda \sqrt{\frac{1}{T}\sum_{t=1}^{T}(\langle r_t ,x \rangle - \langle \bar{r}, x\rangle)^2} - \left\|y_d - \frac{1}{D}\right\|^2\right),
\end{equation}
where $\bar{r} = \sum_{t=1}^{T}r_t$ and decision variable $x$ denotes the investment quantity vector in $d$ assets and $\mathcal{Y} = \{y = [y_d ]\in\mathbb{R}^D | \sum_{d=1}^{D} y_d =1, y_d \geq 0 ,\forall d\}$. In the experiment, we test different methods on real-world datasets from Keneth R. French Data Library\footnote{https://mba.tuck.dartmouth.edu/pages/faculty/ken.french/data\_library.html}. 

\begin{figure}[t!]
  \centering
    % \caption{Testing accuracy of different numbers of Decoupled BNs on seen and unseen domain.}
  \begin{minipage}{0.24\columnwidth}
  \centering
      \includegraphics[width=\textwidth]{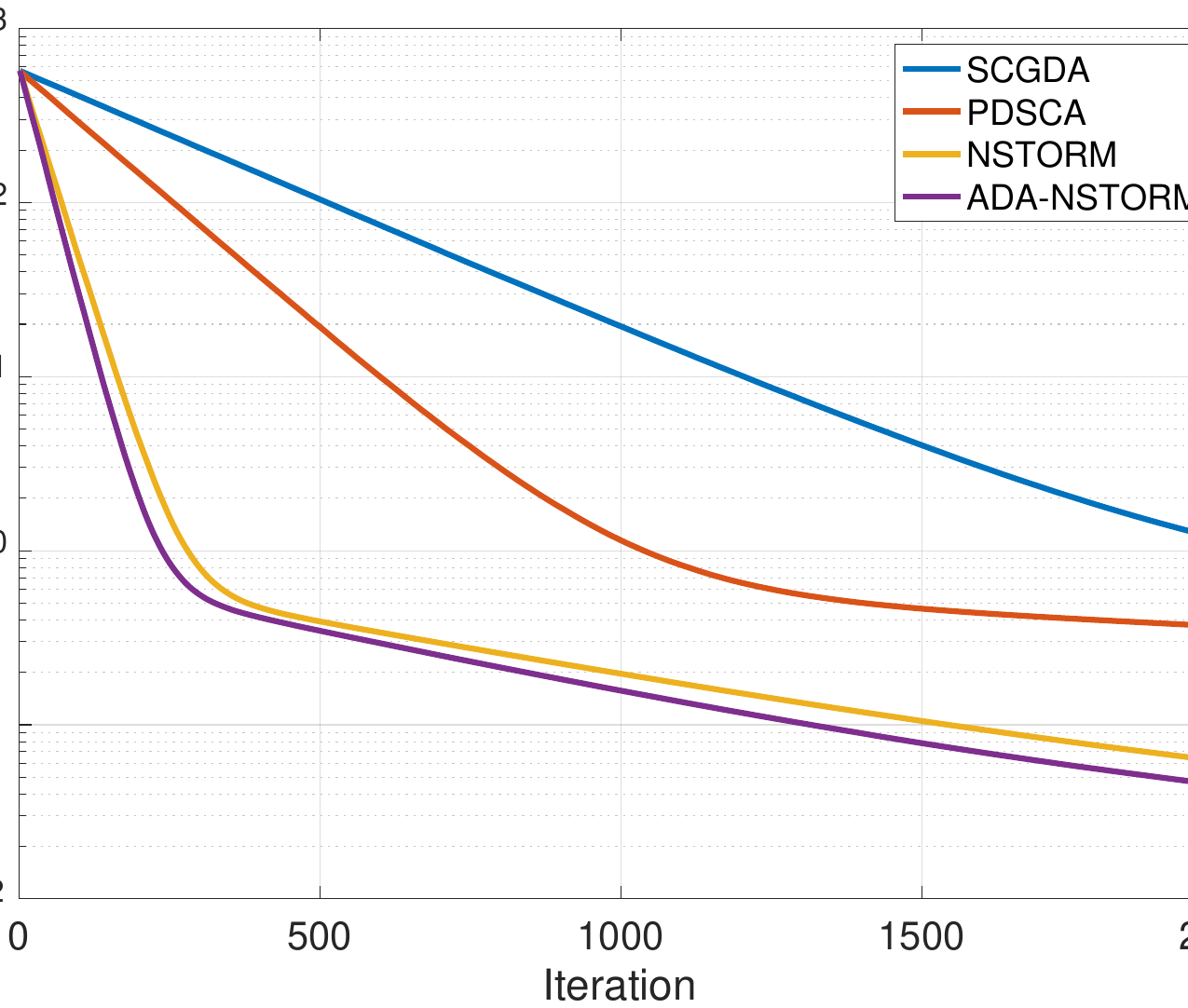}
      \vspace{-10mm}
      \subcaption[first]{Europe.}
      \label{Fig:europe}
    \end{minipage}
    \hfill
  \begin{minipage}{0.24\columnwidth}
  \centering
    \includegraphics[width=\textwidth]{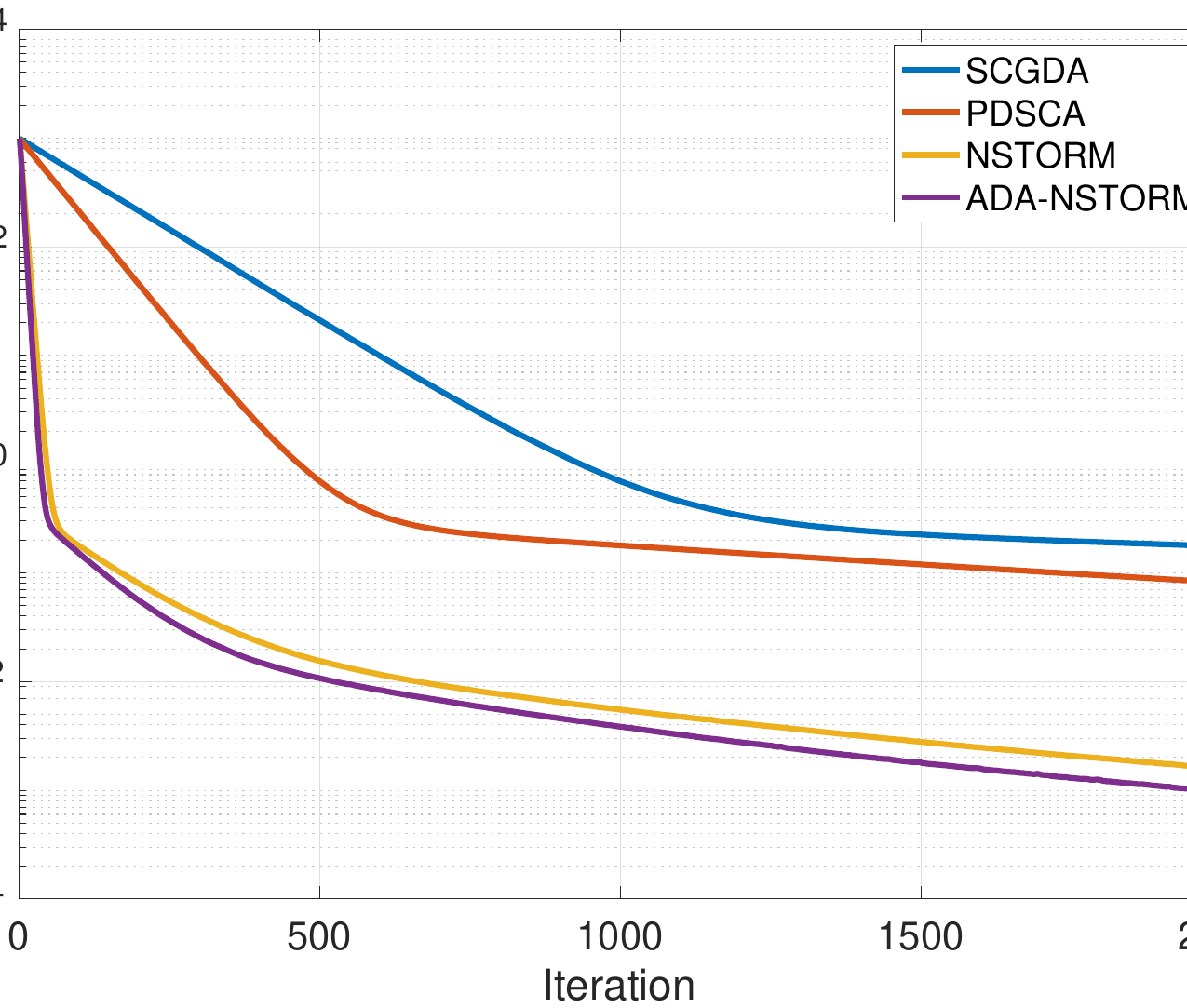}
    \vspace{-10mm}
    \subcaption[second]{Japan.}
     \label{Fig:japan}
    \end{minipage}
    \hfill
  \begin{minipage}{0.24\columnwidth}
  \centering
    \includegraphics[width=\textwidth]{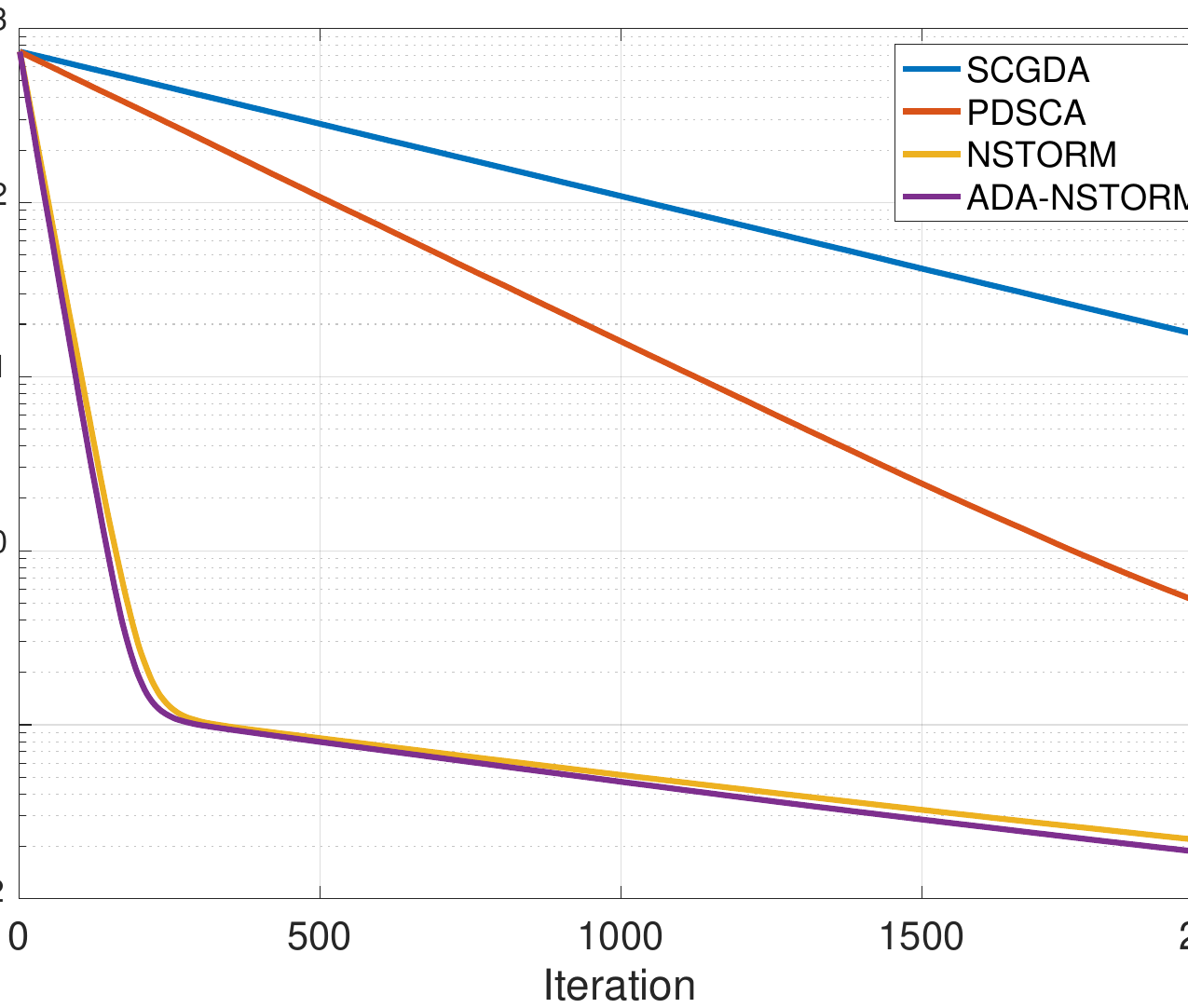}
    \vspace{-10mm}
    \subcaption[third]{North America.}
     \label{Fig:North_America}
    \end{minipage}
        \hfill
  \begin{minipage}{0.24\columnwidth}
  \centering
    \includegraphics[width=\textwidth]{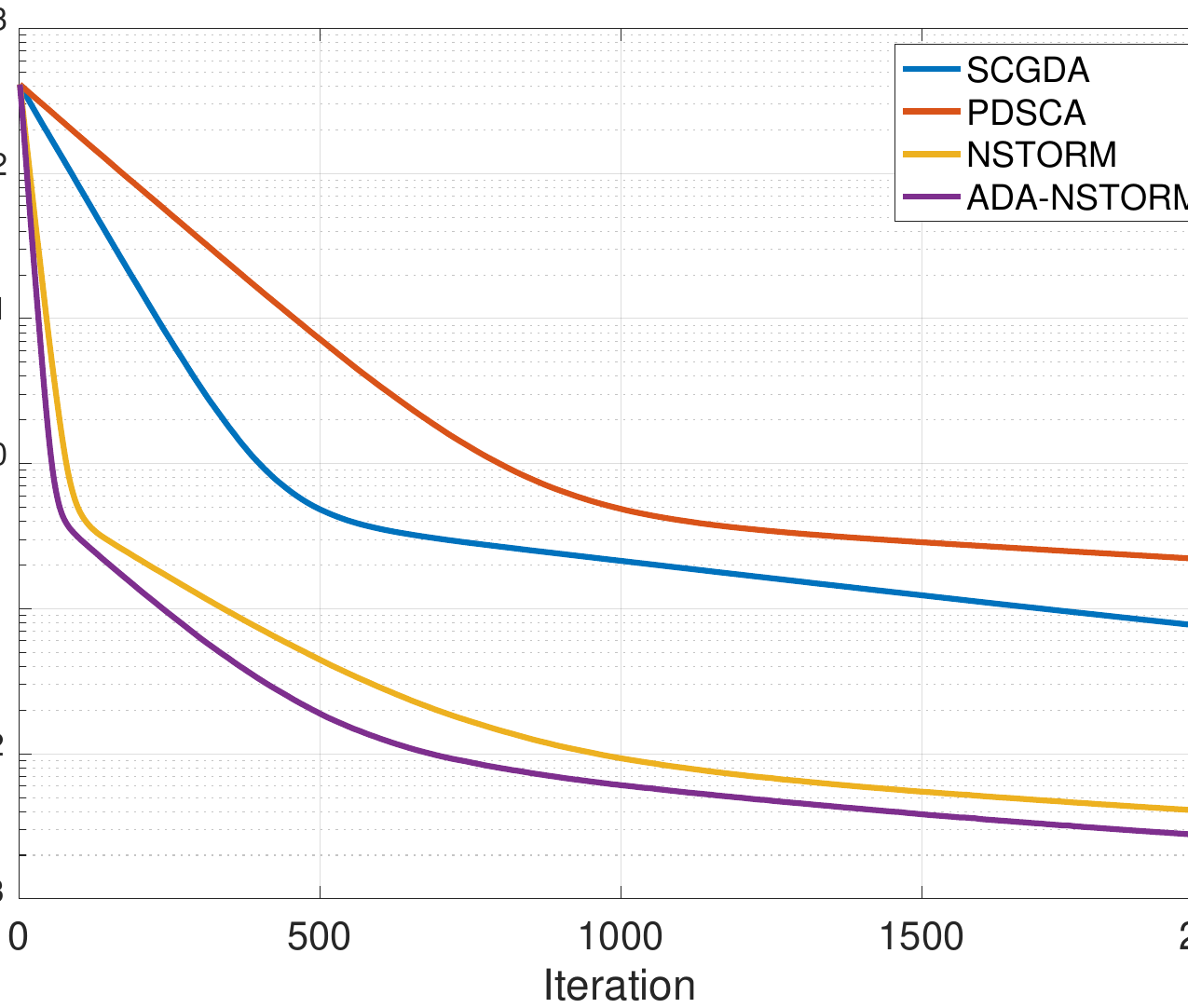}
    \vspace{-10mm}
    \subcaption[fourth]{Global.}
     \label{Fig:global}
    \end{minipage}
    \caption{Objective Gap.}
    \label{Fig:RPGap}
    \vspace{-5mm}
\end{figure}

\begin{figure}[t!]
  \centering
    % \caption{Testing accuracy of different numbers of Decoupled BNs on seen and unseen domain.}
  \begin{minipage}{0.24\columnwidth}
  \centering
      \includegraphics[width=\textwidth]{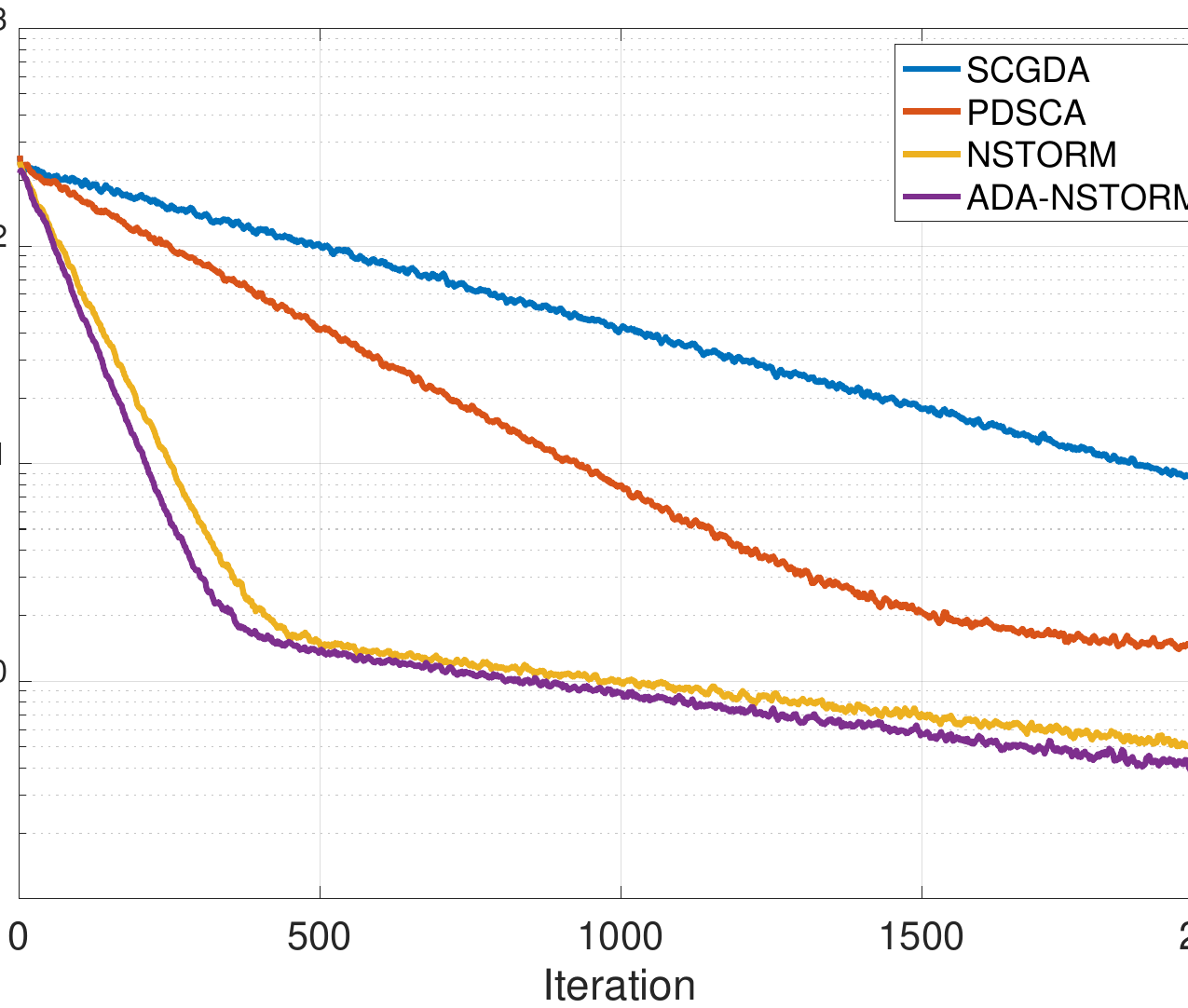}
      \vspace{-12mm}
      \subcaption[first]{Europe.}
      \label{Fig:gradeurope}
    \end{minipage}
    \hfill
  \begin{minipage}{0.24\columnwidth}
  \centering
    \includegraphics[width=\textwidth]{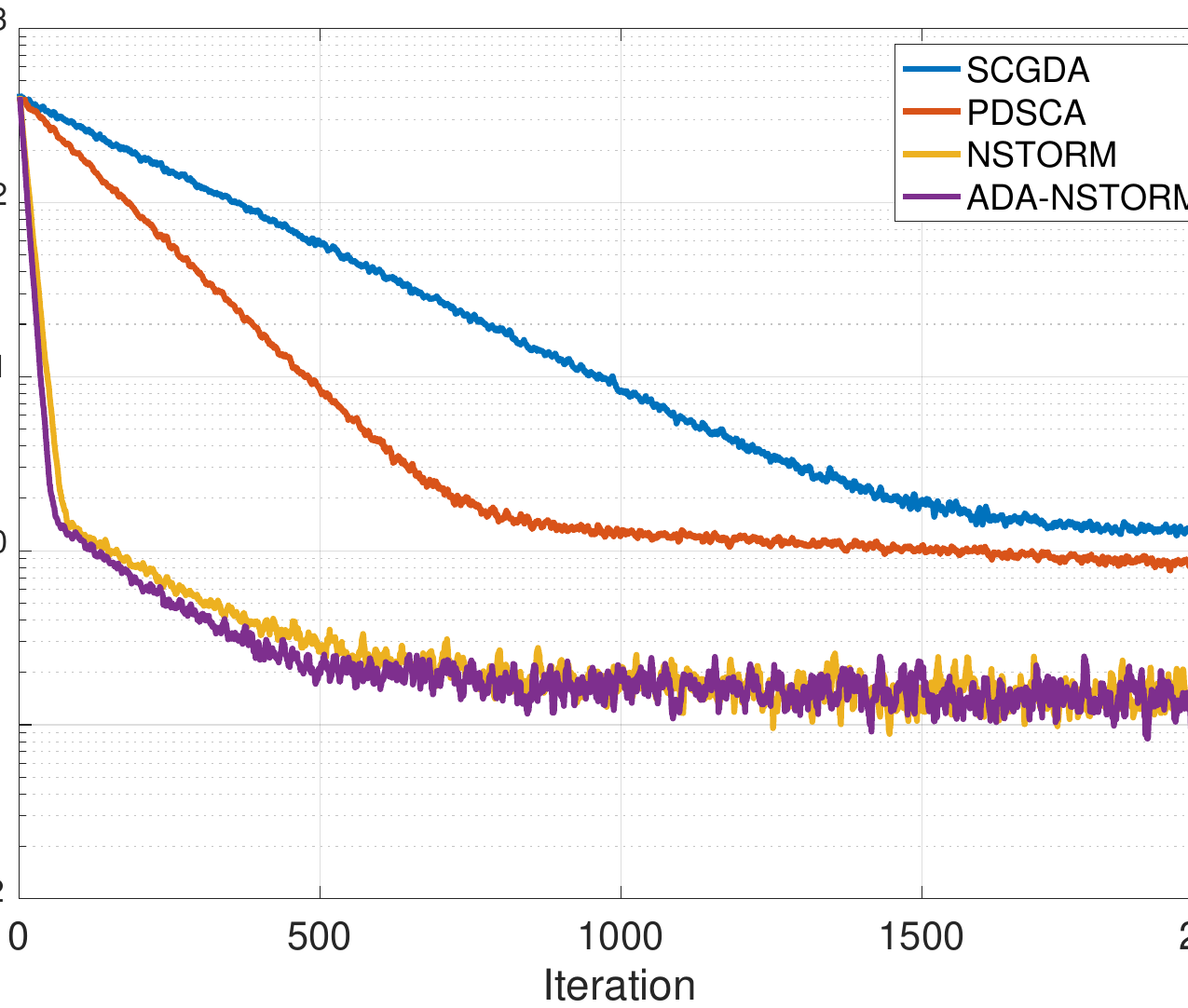}
    \vspace{-12mm}
    \subcaption[second]{Japan.}
     \label{Fig:gradjapan}
    \end{minipage}
    \hfill
  \begin{minipage}{0.24\columnwidth}
  \centering
    \includegraphics[width=\textwidth]{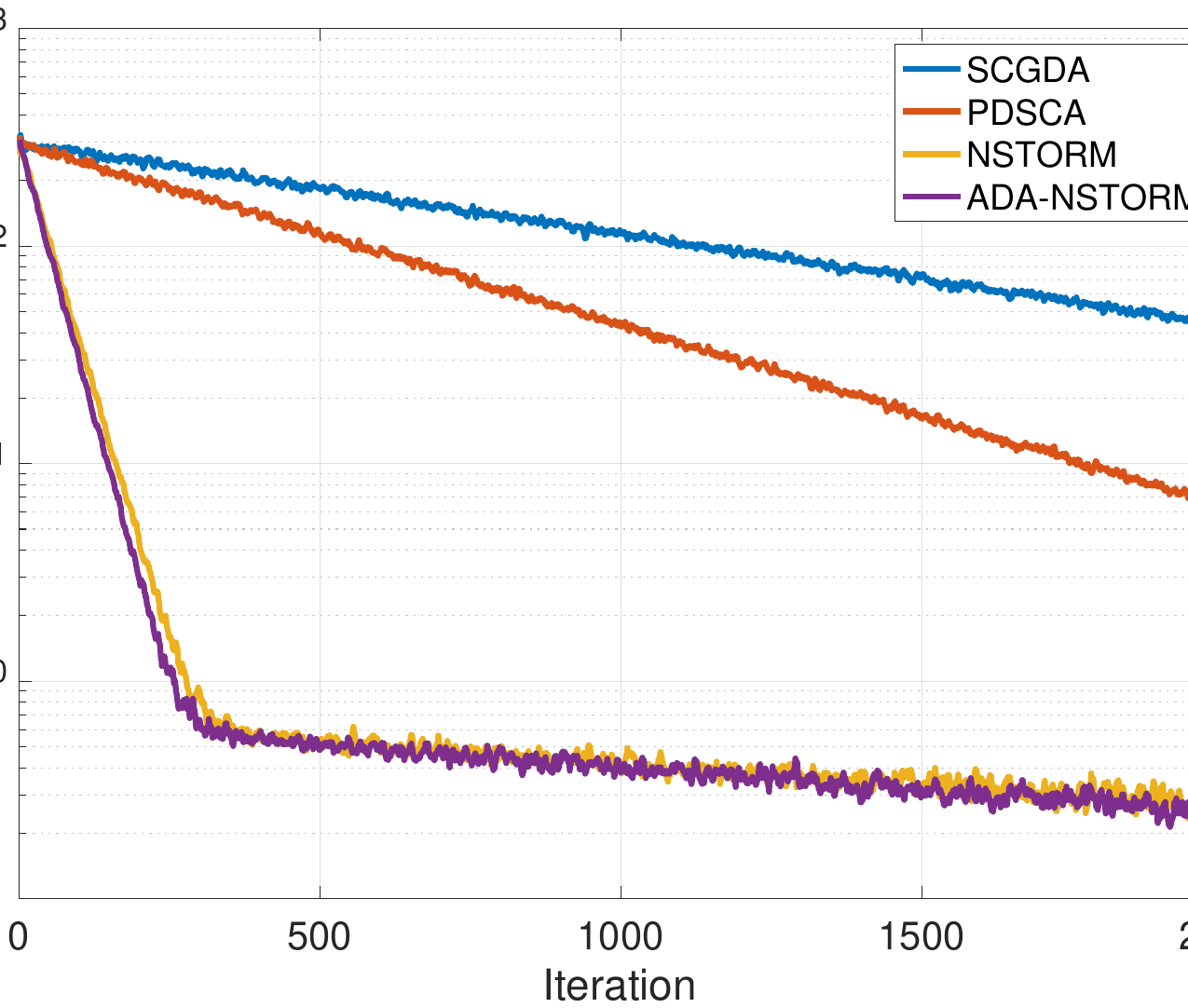}
    \vspace{-12mm}
    \subcaption[third]{North America.}
     \label{Fig:gradNorth_America}
    \end{minipage}
        \hfill
  \begin{minipage}{0.24\columnwidth}
  \centering
    \includegraphics[width=\textwidth]{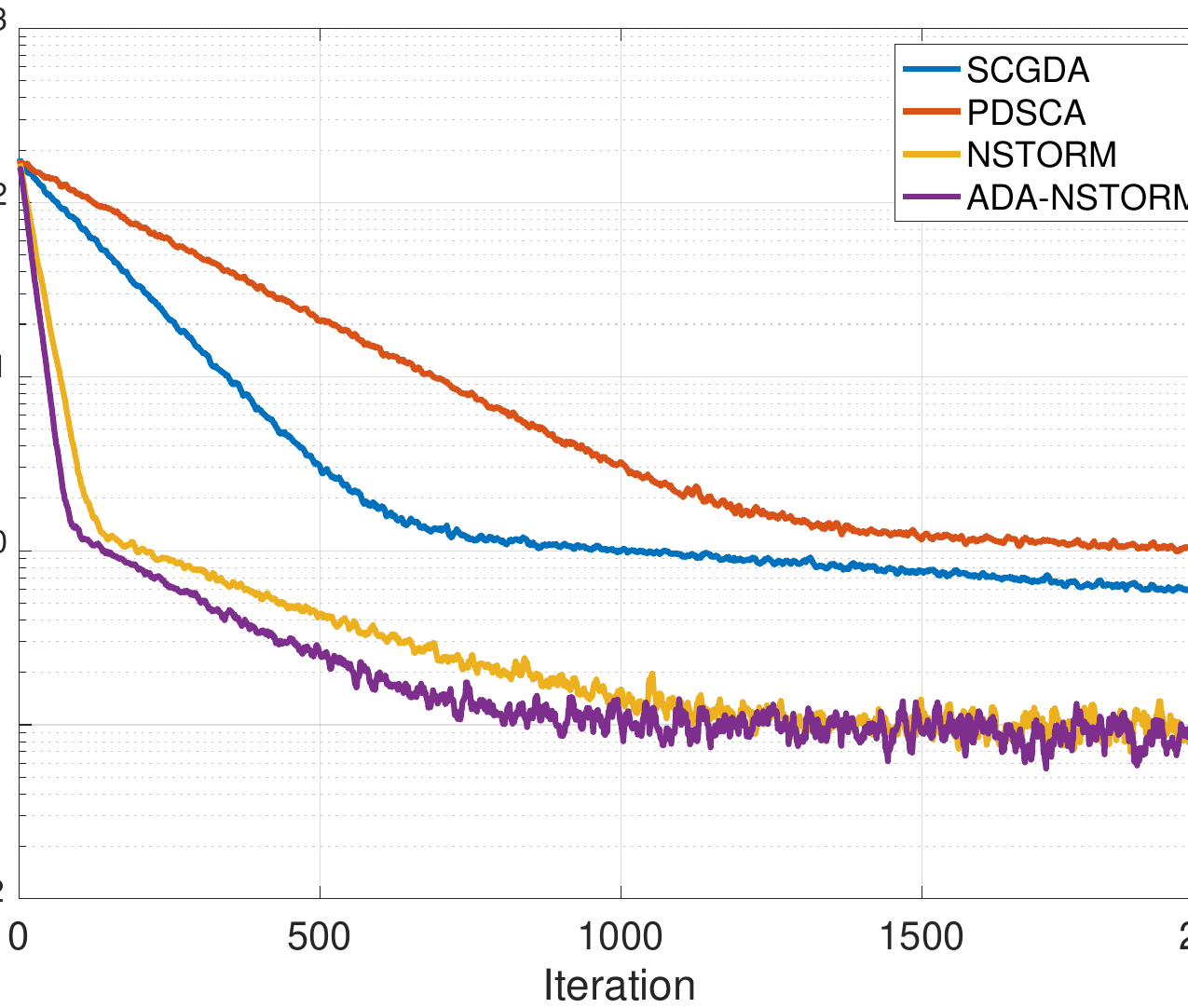}
    \vspace{-12mm}
    \subcaption[fourth]{Global.}
     \label{Fig:gradglobal}
    \end{minipage}
    \caption{Normalized Objective Gap.}
    \label{Fig:RPGapgrad}
\end{figure}

Figures~\ref{Fig:RPGap}-\ref{Fig:RPGapgrad} show the loss value and the norm of the gradient gaps against the number of samples drawn by each method, and all curves are averaged over 20 runs. We can see that our proposed two methods converge much faster than other methods in all datasets. More specifically, both the loss and the gradient of NSTORM and ADA-STORM decrease more quickly, demonstrating the low sample complexity of the proposed methods. In addition, although NSTORM and ADA-STORM obtain the sample complexity theoretically, the latter converges faster in practice due to the adaptive learning rate used in the training procedure.

\subsection{Policy Evaluation in Reinforcement Learning}
In this subsection, we aim to use NSTORM and ADA-NSTORM to optimize the policy evaluation of distributionally robust linear value function approximation in reinforcement learning \cite{zhang2019stochastic}. The value function in reinforcement learning is an important component to compute the reward. More specifically, given a Markov decision process (MDP) $\{\mathcal{S}, P^\pi ,R, r\}$, where $\mathcal{S} = \{1,2,\dots, S\}$ represents the state space, $p_{s,s'}^\pi$ denotes the transition probability from state $s$ to state $s'$ under a given policy $\pi$, $R_{s,s'}$ is the reward when state $s$ goes to state $s'$, and $r$ is the discount factor, then the value function at state $s$ is defined as $V(s) = \sum_{s'=1}^{S}P_{s,s'}^\pi (R_{s,s'} + rV(s'))$. To estimate the value function, a typical choice is to parameterize it with a linear function: $\tilde{V}_x (s)= z_s^\top x$, where $z_s \in \mathbb{R}^L$ is fixed and $x \in \mathbb{R}^L$ is the model parameter which needs to be optimized. We are interested in the distributionally robust variant, i.e., the compositional loss function, \cite{yuan2019stochastic, zhang2019stochastic}, and hence the optimization is modified as follows: 
\begin{equation}\label{Eq:PolicyEvaluation}
    \min_{x}\max_{y} \frac{1}{S}\sum_{s=1}^{S}y_s \left(\tilde{V}_x (s) - \sum_{s'=1}^{S} P_{s,s'}^\pi (R_{s,s'} + r\tilde{V}_x (s'))\right)^2 + \sum_{l=1}^{L}\frac{\beta x_l^2}{1+x_i^2} - \left\|y-\frac{1}{S}\right\|^2 ,
\end{equation}
where $w = [w_l ]\in \mathbb{R}^L$ is the model parameter, $\mathcal{Y} = \{y = [y_s ]\in\mathbb{R}^S | \sum_{s=1}^{S} y_s =1, y_s \geq 0 ,\forall s\}$, $\beta >0$. Following \cite{yuan2019stochastic}, we generate an MDP that has 400 states and each state is associated with 10 actions. Regarding the transition probability, $P_{s,s'}^\pi$ is drawn from $[0,1]$ uniformly. Additionally, to guarantee the ergodicity, we add $10^{-5}$ to $P_{s,s'}^\pi$. Then, we use the four methods to optimize \eqref{Eq:PolicyEvaluation} on this dataset.

\begin{figure}[]
  \centering
  \includegraphics[width=0.4\columnwidth]{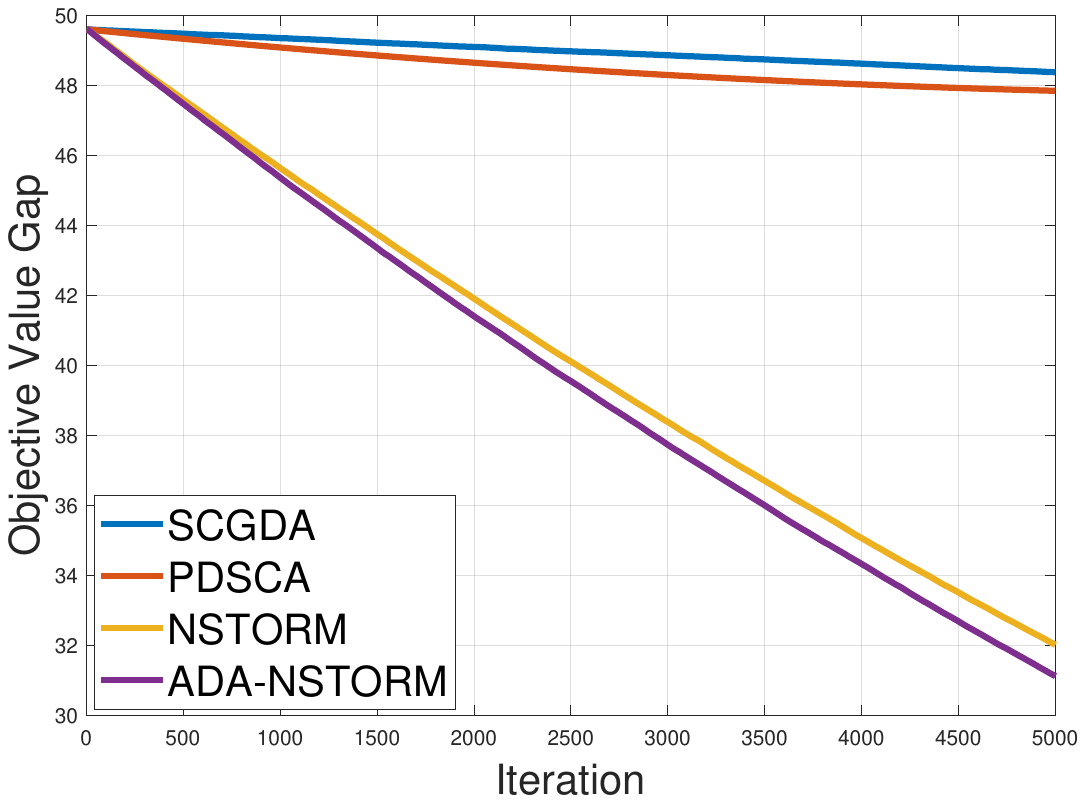}
  \vspace{-3cm}
  \caption{Objective gap of policy evaluation in reinforcement learning.}\label{Fig:RL}
\end{figure}

In Figure~\ref{Fig:RL}, it can be seen that the value function gap decreases with the training going on for all compositional minimax optimization methods. Among all methods, we can see that our proposed NSTORM and ADA-NSTORM methods outperform existing studies, i.e., SCGDA and PDSCA, which confirms the effectiveness of the proposed methods. More specifically, ADA-NSTORM incrementally decreases the gap of the function value, and it is more stable in changing learning rates.

\end{document}